\documentclass{article} % For LaTeX2e
\usepackage[accepted]{icml2024}

\usepackage[T1]{fontenc}
\usepackage{hyperref}
\usepackage{amsmath}
\usepackage{amssymb}
\usepackage{enumerate}
\usepackage{dsfont}
\usepackage{amsthm}
\usepackage{graphics,epsfig,psfrag}
\usepackage{epsf,pstricks,subfloat}
\usepackage{wrapfig}
\usepackage{caption}
\usepackage{subcaption}
\usepackage{multirow}
\renewcommand{\arraystretch}{1.2}
\usepackage{xcolor}
\definecolor{mygray}{gray}{0.95}
\usepackage{bm}
\usepackage{url}
\usepackage{psfrag}
\usepackage{amsmath}
\usepackage{amsfonts}
\usepackage{amsthm}

\usepackage{graphicx}
\usepackage{subcaption}

\def\bfzero{{\boldsymbol{0}}}
\def\bfone{{{\bf1}}}
\def\bfa{{\boldsymbol a}}
\def\bfb{{\boldsymbol b}}
\def\bfc{{\boldsymbol c}}

\def\bff{{\boldsymbol f}}

\def\bfh{{\boldsymbol h}}

\def\bfs{{\boldsymbol s}}

\def\bfw{{\boldsymbol w}}
\def\bfx{{\boldsymbol x}}

\def\bfz{{\boldsymbol z}}

\def\bfA{{\boldsymbol A}}
\def\bfB{{\boldsymbol B}}

\def\bfE{{\boldsymbol E}}

\def\bfI{{\boldsymbol I}}
\def\bfJ{{\boldsymbol J}}

\def\bfP{{\boldsymbol P}}

\def\bfR{{\boldsymbol R}}

\def\bfU{{\boldsymbol U}}

\def\bfW{{\boldsymbol W}}
\def\bfX{{\boldsymbol X}}

\def\bfZ{{\boldsymbol Z}}

\def\hkt{b^{(t)}_{\ell,k}}
\def\cJ{\mathcal{J}}

\newcommand{\W[1]}{\ensuremath{\theta}^{(#1)}}

\newcommand{\bfal}{\boldsymbol{\alpha}}

\newtheorem{theorem}{Theorem}

\newtheorem{lemma}{Lemma}
\newtheorem{definition}{Definition}
\newtheorem{ass}{Assumption}

\usepackage{wrapfig}

\newcommand{\Revise}[1]{{\color{black}#1}}

\begin{document}
\twocolumn[
\icmltitle{SF-DQN: Provable Knowledge Transfer using Successor Feature for Deep Reinforcement Learning}

% It is OKAY to include author information, even for blind
% submissions: the style file will automatically remove it for you
% unless you've provided the [accepted] option to the icml2024
% package.

% List of affiliations: The first argument should be a (short)
% identifier you will use later to specify author affiliations
% Academic affiliations should list Department, University, City, Region, Country
% Industry affiliations should list Company, City, Region, Country

% You can specify symbols, otherwise they are numbered in order.
% Ideally, you should not use this facility. Affiliations will be numbered
% in order of appearance and this is the preferred way.
\icmlsetsymbol{equal}{*}

\begin{icmlauthorlist}
\icmlauthor{Shuai Zhang}{yyy}
\icmlauthor{Heshan Devaka Fernando}{comp}
\icmlauthor{Miao Liu}{sch}
\icmlauthor{Keerthiram Murugesan}{sch}
\icmlauthor{Songtao Lu}{sch}
\icmlauthor{Pin-Yu Chen}{sch}
\icmlauthor{Tianyi Chen}{comp}
%\icmlauthor{}{sch}
\icmlauthor{Meng Wang}{comp}
%\icmlauthor{}{sch}
%\icmlauthor{}{sch}
\end{icmlauthorlist}

\icmlaffiliation{yyy}{Department of Data Science, New Jersey Institute of Technology, Newark, NJ}
\icmlaffiliation{comp}{Rensselaer Polytechnic Institute, Troy, NY.}
\icmlaffiliation{sch}{IBM Research}

\icmlcorrespondingauthor{Shuai Zhang}{sz457@njit.edu}
\icmlcorrespondingauthor{Meng Wang}{wangm7@rpi.edu}

% You may provide any keywords that you
% find helpful for describing your paper; these are used to populate
% the "keywords" metadata in the PDF but will not be shown in the document
\icmlkeywords{Machine Learning, ICML}

\vskip 0.3in
]

% this must go after the closing bracket ] following \twocolumn[ ...

% This command actually creates the footnote in the first column
% listing the affiliations and the copyright notice.
% The command takes one argument, which is text to display at the start of the footnote.
% The \icmlEqualContribution command is standard text for equal contribution.
% Remove it (just {}) if you do not need this facility.

%\printAffiliationsAndNotice{}  % leave blank if no need to mention equal contribution
\printAffiliationsAndNotice{\icmlEqualContribution} % otherwise use the standard text.

\begin{abstract}
This paper studies the transfer reinforcement learning (RL) problem where multiple RL problems have different reward functions but share the same underlying transition dynamics. In this setting, the Q-function of each RL problem (task) can be decomposed into a successor feature (SF) and a reward mapping: the former characterizes the transition dynamics, and the latter characterizes the task-specific reward function.
This Q-function decomposition, coupled with a policy improvement operator known as generalized policy improvement (GPI), reduces the sample complexity of finding the optimal Q-function, and thus the SF \& GPI framework exhibits promising empirical performance compared to traditional RL methods like Q-learning.
However, its theoretical foundations remain largely unestablished, especially when learning the successor features using deep neural networks (SF-DQN). This paper studies the provable knowledge transfer using SFs-DQN in transfer RL problems.  
We establish the first convergence analysis with provable generalization guarantees for SF-DQN with GPI. The theory reveals that SF-DQN with GPI outperforms conventional RL approaches, such as deep Q-network, in terms of both faster convergence rate and better generalization. Numerical experiments on real and synthetic RL tasks support the superior performance of SF-DQN \& GPI, aligning with our theoretical findings.

\end{abstract}

\section{Introduction}
In reinforcement learning (RL), the goal is to train an agent to perform a task within an environment in a desirable manner by allowing the agent to interact with the environment. Here, the agent is guided towards the desirable behavior by the rewards, and the optimal policy is derived from a learned value function (Q-function) in selecting the best actions to maximize the immediate and future rewards. This framework can effectively capture a wide array of real-world applications, such as gaming \cite{mnih2013playing,silver2017mastering}, robotics \cite{KIPI18}, autonomous vehicles \cite{SSS16,SAR18}, healthcare \cite{CNDP20}, and natural language processing \cite{TXCW18}.
However, RL agents often need numerous interactions with the environment to manage complex tasks, especially when RL is equipped with deep neural networks (DNNs). For example, AlphaGo \cite{silver2017mastering} required 29 million matches and 5000 TPUs at a cost exceeding \$35 million, which is time-consuming and memory-intensive. 
Nevertheless, many complex real-world problems can naturally {decompose into} multiple interrelated sub-problems, {all} sharing the same environmental dynamics \cite{sutton1999between, bacon2017option, kulkarni2016hierarchical}. In such scenarios, it becomes highly advantageous for an agent to harness knowledge acquired from previous tasks to enhance its performance in tackling new but related challenges. This practice of leveraging knowledge from one task to improve performance in others is known as transfer learning \cite{lazaric2012transfer,taylor2009transfer,barreto2017successor}.

% This paper studies transfer reinforcement learning, specifically focusing on sub-tasks associated with different reward functions within the same environment. 

This paper focuses on {an RL setting with learning multiple tasks}, where each task is associated with a different reward function but shares the same environment. This setting naturally arises in many real-world applications such as robotics \cite{yu2020meta}. 
We consider exploring the knowledge transfer among multiple tasks via the successor feature (SF) framework \cite{barreto2017successor} which disentangles the environment dynamic from the reward function at an incremental computational cost.
The SF framework is derived from successor representation (SR) \cite{dayan1993improving} by introducing the value function approximation. Specifically, SR \cite{dayan1993improving} decouples the value function into a future state occupancy measure and a reward mapping. Here, the future state occupancy characterizes the transition dynamics of the environment, and the reward mapping characterizes the reward function of the task. SF is a natural application of SR in solving value function approximation. 
Furthermore, \cite{barreto2017successor} propose a generalization of the classic policy improvement, termed generalized policy improvement (GPI), enabling smooth knowledge transfer across learned policies. 
In contrast to traditional policy improvement, which typically considers only a single policy, Generalized Policy Improvement (GPI) operates by maintaining a set of policies, each associated with a distinct skill the agent has acquired. This approach enables the agent to switch among these policies based on the current state or task requirements, providing a flexible and adaptive framework for decision-making.
Empirical findings presented in \cite{barreto2017successor} highlight the superior transfer performance of SF \& GPI in deep RL when compared to conventional methods like Deep Q-Networks (DQNs). Subsequent works further justified the improved performance of SF in subgoal identification \cite{kulkarni2016deep} and real-world robot navigation \cite{zhang2017deep}.

% We focus on this setting because many real world problems can be naturally decomposed into sub-problems, while each subproblem sharing the same environment dynamics.
 % SR was later extended to the value function approximation setting via successor features (SF) \cite{barreto2017successor}
% To tackle this problem, \cite{dayan1993improving} introduced successor representation (SR) which decouples the value function of a task into future state occupancy measure and a reward mapping. Here, the future state occupancy characterizes the transition dynamics of the environment, and the reward mapping characterizes the reward function if the task. SR was later extended to the value function approximation setting via successor features (SF) \cite{barreto2017successor}. SF framework provide a representation of the  Q-function that naturally disentangles the environment dynamics from the reward function, rendering them exceptionally well-suited for facilitating knowledge transfer in our problems.
% \cite{barreto2017successor} introduced the concept of successor features (SFs), namely, the reward function can be decomposed as the transition feature and reward mapping, which is an extension of the successor representation first proposed in \cite{dayan1993improving}. 
% SFs provide a representation of the  Q-function that naturally disentangles the environment dynamics from the reward function, rendering them exceptionally well-suited for facilitating knowledge transfer in our problems.

% \SZ{Add citations for some applications for SF if possible.}

\textbf{Focus of this paper.} While performance guarantees of SF-based learning are provided in the simple tabular setting \cite{barreto2017successor, barreto2018transfer}, less is known for such approaches in the widely used function approximation setting, especially for non-linear models including DNNs. In this context, 
% often referred to as deep reinforcement learning (DRL), states are represented using feature vectors, and the corresponding Q-function is implemented using a DNN \cite{mnih2013playing}. 
this paper aims to close this gap by providing theoretical guarantees for SF learning in the context of DNNs. Our objective is to explore the convergence and generalization analysis of SF when paired with DNN approximation. We also seek to delineate the conditions under which SF learning can offer more effective knowledge transfer among tasks when contrasted with classical deep reinforcement learning (DRL) approaches.

% and investigate under which setting SF learning can aid better transfer among tasks, compared to classical RL approaches such as deep Q-learning \cite{mnih2013playing}.

\textbf{Contributions.} 
This paper presents the first convergence analysis with generalization guarantees for successor feature learning with deep neural network approximation (SF-DQN).
This paper focuses on estimating the optimal Q-value function through the successor feature decomposition, where the successor feature decomposition component is approximated through a deep neural network. The paper offers a comprehensive analysis of the convergence of deep Q-networks with successor feature decomposition and provides insights into the improved performance of the learned Q-value function derived from successor feature decomposition. The key contributions are as follows:

\textbf{(C1) The convergence analysis of the proposed SF-DQN  to the optimal Q-function with generalization guarantees.} 
By decomposing the reward into a linear combination of the transition feature and reward mapping, we demonstrate that the optimal Q-function can be learned by alternately updating the reward mapping and the successor feature using the collected data in online RL, where the corresponding successor feature is parameterized by a deep neural network. The learned Q-function converges to the optimal Q-function with generalization guarantees at a rate of $1/T$, where $T$ is the number of iterations in updating transition features and reward mappings.

\textbf{(C2) The theoretical characterization of enhanced performance by leveraging knowledge from previous tasks through GPI.}
This paper characterizes the convergence rate with generalization guarantees in transfer RL utilizing GPI. The convergence rate accelerates following the degree of correlation between the source and target tasks.

\textbf{(C3) The theoretical characterization of the superior transfer learning performance with SF-DQN over non-representation learning approach DQNs.} 
This paper quantifies the transfer learning ability of SF-DQN and DQN algorithms by evaluating their generalization error when transferring knowledge from one task to another. Our results indicate that SF-DQN achieves improved generalization compared to DQN, demonstrating the superiority of SF-DQN in transfer RL. 

\subsection{Related Works}
\textbf{Successor features in RL.} In pioneering works, \cite{dayan1993improving} introduced the concept of SR, demonstrating that the value function can be decomposed into a reward mapping and a state representation that measures the future state occupancy from a given state, with learning feasibility proof in tabular settings.
Subsequently,
\cite{barreto2017successor} extended SR from three perspectives: (1) the feature domain of SR is extended from states to state-action pairs, known as SF; (2) DNNs are deployed as function approximators to represent the SF and reward mappings; (3) GPI algorithm is introduced to accelerate policy transfer for multi-tasks.  Furthermore, 
\cite{kulkarni2016deep, zhang2017deep} apply SF learning with DNN-based schemes to subgoal identification \cite{kulkarni2016deep} and robot navigation \cite{zhang2017deep}. However, only \cite{barreto2017successor, barreto2018transfer} provided transfer guarantees for Q-learning with SF and GPI for the tabular case under the assumption that the Q-function from the source task is well-estimated. However, to the best of our knowledge, none of works have provided any theoretical guarantees of SF in the function approximation with neural networks. In addition, instead of assuming that the Q-function of the source task is well estimated, our paper offers both convergence analysis and sample complexity for successor feature learning in both the source task training stage and transfer learning stages.
% On the other hand, \cite{barreto2018transfer} shows that the linearity assumption of the reward can be relaxed, and for a given environment with different tasks that share the same action space and state space but have different functions, reward functions themselves can act as state transition features. Furthermore, under the aforementioned relaxed assumptions, \cite{barreto2018transfer}  provided theoretical guarantees for transfer in tabular RL setting. 
We refer readers to a comprehensive comparison of reinforcement learning transfer using Successor Features, as detailed in \cite{zhu2023transfer}. 
% As a generalization of SF methods, \cite{blier2021learning} propose Successor measures, which does not require 

\textbf{RL with neural networks.}
Recent advancements of theoretical analysis in RL with neural network approximation mainly include the Bellman Eluder dimension \cite{JKALS17, RV13}, Neural Tangent Kernel (NTK) \cite{YJWWM20, CYLW19, XG20, DLMW20}, and Besov regularity \cite{Su19, JCWZ22, NGV22}. However, each of these frameworks has its own limitations.
The Eluder dimension exhibits exponential growth even for shallow neural networks \cite{DYM21}, making it challenging to characterize sample complexity in real-world applications of DRL.
The NTK framework linearizes DNNs to bypass the non-convexity derived from the non-linear activation function in neural networks. Nevertheless, it requires using computationally inefficient, extremely wide neural networks \cite{YJWWM20}. Moreover, the NTK approach falls short in explaining the advantages of utilizing non-linear neural networks over linear function approximation \cite{LVC22, FWXY20}.
The Besov space framework \cite{JCWZ22, NGV22, LVC22, FWXY20} requires sparsity on neural networks and makes the impractical assumption that the algorithm can effectively identify the global optimum, which is unfeasible for non-convex optimization involving NNs.
% Besov space requires the neural networks to be sparse and makes an unpractical assumption that the algorithm can find the global optimal of the non-convex objective function \cite{JCWZ22,NGV22,LVC22,FWXY20}. 

\textbf{Theory of generalization in deep learning.}
The theory of generalization in deep learning has been extensively developed in supervised learning, where labeled data is available throughout training. 
% Unlike RL, which relies on the estimated Q function and presents analytical challenges, supervised learning benefits from fixed and known labeled data. 
Generalization in learned models necessitates low training error and small generalization gap. However, in DNNs, training errors and generalization gaps are analyzed separately due to their non-convex nature. 
To ensure bounded generalization, it is common to focus on \textit{one-hidden-layer} neural networks \cite{SS18} in convergence analysis.
% Given the inherent non-convex nature of deep neural networks \cite{SS18}, it is still a common practice to focus on \textit{one-hidden-layer} neural networks when conducting convergence analysis with bounded generalization guarantees. 
Existing theoretical analysis tools in supervised learning with generalization guarantees draw heavily from various frameworks, including the Neural Tangent Kernel (NTK) framework \cite{JGH18, DZPS18, LBNSSPS18}, model recovery techniques \cite{ZSJB17, GLM17, BJW19, SJL18, ZWXL20}, and the analysis of structured data \cite{LL18, SWL22, BG21, AL22, KWLS21, WL21, ZWCLL23,LWLC23,chowdhury2023patch}.

% In addition,
% such as Gaussian data distributions \cite{ZLJ16, BG17}, or linear separability of data \cite{WGC19, BGM18}

\section{Preliminaries}
{In this paper, we address the learning problem involving multiple tasks $\{\mathcal{T}_i\}_{i=1}^n$ and aim to find the optimal policy $\pi_i^\star$ for each task $\mathcal{T}_i$. We begin by presenting the preliminaries for a single task and then elaborate on our algorithm for learning with multiple tasks in the following section.} 

\textbf{Markov decision process and Q-learning.} The Markov decision process (MDP) is defined as a tuple $(\mathcal{S}, \mathcal{A}, \mathcal{P}, r, \gamma)$, where $\mathcal{S}$ is the state space and $\mathcal{A}$ is the set of possible actions. The transition operator $\mathcal{P}: \mathcal{S} \times \mathcal{A} \rightarrow \Delta(\mathcal{S})$ gives the probability of transitioning from the current state $\bfs$ and action $a$ to the next state $\bfs'$. The function $r: \mathcal{S} \times \mathcal{A}\times \mathcal{S}   \rightarrow [-R_{\max}, R_{\max}]$ measures the reward for a given state-action pair. The discount factor $\gamma \in [0, 1)$ determines the significance of future rewards.
% For notational simplicity, we denote  $\mathcal{X}:=\mathcal{S}\times \mathcal{A}\in \mathbb{R}^d$ and $\bfx =(s,a)$ with $s\in\mathcal{S}$ and $a\in\mathcal{A}$. 

{For the $i$-th task,}
the goal of the agent is to find the optimal policy $\pi_i^\star$ with $a_t=\pi_i^\star(\bfs_t)$ at each time step $t$. The aim is to {maximize} the expected discounted sum of reward as $\Revise{\sum_{t=0}^\infty \gamma^t\cdot r_i(\bfs_t,a_t,\bfs_{t+1})}$, {where $r_i$ denotes the reward function for the $i$-th task.} 
For any state-action pair $(\bfs,a)$, we define the action-value  function $Q_i^\pi$  given a policy $\pi$ as 
\begin{equation*}\label{eqn: Q}
\footnotesize
    \begin{split}
        Q_i^\pi(\bfs,a) =&  \mathbb{E}_{\pi,\mathcal{P}}\big[\textstyle\sum_{t=0}^\infty \gamma^t r_i(\bfs_t,a_t,\bfs_{t+1})\mid \bfs_0=\bfs,a_0=a\big].
    \end{split}
\end{equation*}
The optimal $Q$-function, denoted as $Q_i^{\pi^\star}$ or $Q_i^\star$, satisfies
\begin{equation}\label{eqn:b2}
\footnotesize
\begin{split}
Q_i^{\star}(\bfs,a)  :=& \max_{\pi} Q_i^{\pi}(\bfs,a)\\
    =& \mathbb{E}_{\bfs'|\bfs,a}~ r_i(\bfs,a,\bfs^\prime) + \gamma \max_{a'}Q_i^{\pi^\star}(\bfs',a'),
\end{split}
\end{equation}
where \eqref{eqn:b2} is also known as the Bellman equation.
Through the optimal action-value function $Q_i^{\star}$, the agent can derive the optimal policy \cite{Qlearning,su18} following  
\begin{equation}\label{eqn: policy}
{\pi_i^\star(\bfs) =\arg\max_a Q_i^\star(\bfs,a).}    
\end{equation} 
\textbf{Deep Q-networks (DQNs).}
% Similar to \cite{SB18,Sze10} for RL with function approximation, 
The DQN utilizes a DNN parameterized with weights $\omega$, i.e.,
$Q_i(\bfs,a;\omega): \mathbb{R}^d\rightarrow \mathbb{R}$ for the $i$-th task, to approximate the optimal Q-value function $Q_i^\star$ in \eqref{eqn:b2}. Specifically, given  input feature $\bfx:= \bfx(\bfs,a)$, the output of the $L$-hidden-layer DNN is defined as 
\begin{equation}\label{eqn: DQN}
    Q_i(\bfs,a;\omega) := \omega_{L+1}^\top/K\cdot  \sigma\big(\omega_L^\top\cdots \sigma(\omega_1^\top\bfx )\big),
\end{equation} 
where $\sigma(\cdot)$ is the ReLU activation function, i.e., $\sigma(z) = \max\{0, z\}$.

\textbf{Successor feature.}
    {For $i$-the task,} suppose the expected one-step reward associated with the transition $(\bfs,a,\bfs^\prime)$ can be computed as 
    % \vspace{-2mm}
    \begin{equation}\label{eqn:ass_SF}
        {r_i(\bfs,a,\bfs^\prime) = \boldsymbol{\phi}(\bfs,a,\bfs^\prime)^\top \bfw_i^\star, \textit{  with  } \boldsymbol{\phi}, \bfw_i^\star  \in \mathbb{R}^d,}
    \end{equation}
    {where $\phi$ remains the same for all the tasks.}
% \MW{No dimension specification of phi and w?}
With the reward function in \eqref{eqn:ass_SF}, the Q-value function in \eqref{eqn: Q} is reformulated
\begin{equation}\label{eqn: SF}
\begin{split}
    &Q_i^\pi(\bfs,a)\\
    % =~\mathbb{E}_{\pi,\mathcal{P}}\big[\textstyle\sum_{t=0}^\infty \gamma^t r_t\mid(\bfs_0,a_0)\big]
    =&~\mathbb{E}_{\pi,\mathcal{P}}\big[\textstyle\sum_{t=0}^\infty \gamma^t \boldsymbol{\phi}(\bfs_t,a_t,\bfs_{t+1})\mid\bfs_0,a_0\big]^\top \bfw_i^\star\\
    :=&~\psi_i^\pi(\bfs,a)^\top\bfw_i^\star.
\end{split}
\end{equation}
\vspace{-1mm}
Then, the optimal Q function satisfies
\begin{equation}\label{eqn: SF_Q}
\begin{split}
        &Q_i^{\star}(\bfs,a) \\
    =& \mathbb{E}_{\pi_i^\star,\mathcal{P}}\big[\textstyle\sum_{i=0}^\infty \gamma^i \boldsymbol{\phi}(\bfs_i,a_i,\bfs_{i+1})\mid\bfs_0,a_0\big]^\top \bfw_i^\star\\
    :=&~\psi_i^\star(\bfs,a)^\top\bfw_i^\star.
\end{split}
\end{equation}

\section{Problem Formulation and Algorithm} 

\textbf{Problem formulation.}
 Without loss of generality, the data is assumed to be collected from the tasks in the order of $\mathcal{T}_1$ to $\mathcal{T}_n$  during the learning process.  The goal is to utilize collected data for the current task, e.g., $\mathcal{T}_j$, and the learned knowledge from previous tasks $\{\mathcal{T}_i\}_{i=1}^{j-1}$ to derive the optimal policy $\pi_j^\star$ for the current $\mathcal{T}_j$. These tasks share the same environment dynamic but the reward function changes across the task as shown in \eqref{eqn:ass_SF}. For each task $\mathcal{T}_i$, we denote its reward as
\begin{equation}
    r_i = \boldsymbol{\phi} \cdot \bfw^\star_{i}, \quad \textit{with} \quad\|\boldsymbol{\phi}\|_2\le \phi_{\max},
\end{equation}
where $\boldsymbol{\phi}$ is the transition feature across all the tasks and $\bfw^\star_i$ is the reward mapping.

% From \eqref{eqn: policy}, the optimal policy can be derived from the optimal Q-function.
{From \eqref{eqn: SF_Q}, the learning of optimal Q-function for the $i$-th task is decomposed as two sub-tasks: learning SF $\psi_i^\star(\bfs,a)$ and learning reward  $\bfw_i^\star$.}

\textbf{Reward mapping.} To find the optimal $\bfw_i^\star$, we utilize the information from $\boldsymbol{\phi}(\bfs,a,\bfs^\prime)$ and $r_i(\bfs,a,\bfs^\prime)$.
The value of $\bfw^\star_i$ can be obtained by solving the optimization problem 
\begin{equation}\label{eqn: optimization_w}
    \textstyle \min_{\bfw_i}:~~\|r_i - \boldsymbol{\phi}\cdot \bfw_i\|_2.
\end{equation}
\textbf{Successor features.}  {We use $\psi_i^{\pi}$ to denote the successor feature for the $i$-th task, and $\psi_i^{\pi}$ satisfies}
\begin{equation}\label{eqn: sfdqn}
    \psi_i^{\pi}(\bfs,a) =  \mathbb{E}_{\bfs'|\bfs,a} ~ \boldsymbol{\phi}(\bfs,a,\bfs^\prime) +\gamma\cdot \psi^{\pi}_i\big(\bfs',\pi(\bfs^\prime)\big).
\end{equation}
The expression given by \eqref{eqn: sfdqn} aligns perfectly with the Bellman equation in \eqref{eqn:b2}, where $\boldsymbol{\phi}$ acts as the reward. Therefore, following DQNs, we utilize
 a function $\psi(\bfs,a)$ parameterized using the DNN as
\begin{equation}
    \psi_i(\Theta_i;\bfs,a) = H\big(\Theta_i;\bfx(\bfs,a)\big),
\end{equation}
where $\bfx:\mathcal{S}\times\mathcal{A}\longrightarrow \mathbb{R}^d$ is the feature mapping of the state-action pair. Without loss of generality, we assume $|\bfx(\bfs,a)|\le 1$. Then, finding $\psi^{\star}$ is to minimize the mean squared Bellman error (MSBE)
\begin{equation}\label{eqn:MSBE}
\begin{split}
    \min_{\Theta_i}: &f(\Theta_i) := \mathbb{E}_{(\bfs,a)\sim\pi^\star} \Big[ \mathbb{E}_{\bfs'\mid \bfs,a}~ \psi_i(\Theta_i;\bfs,a) \\
    &-\boldsymbol{\phi}(\bfs,a,\bfs^\prime) - \gamma \cdot \psi_i\big(\Theta_i;\bfs',\pi^{\star}(s')\big) \Big]^2.
\end{split}
\end{equation}
It is worth mentioning that although \eqref{eqn:MSBE} and \eqref{eqn: optimization_w} appear to be independent of each other, the update of $\bfw_i$ does affect the update of $\psi_i$ through the shift in data distribution. The collected data is estimated based on the policy depending on the current estimated values of $\psi_i$ and $\bfw_i$, which shifts the distribution of the collected data away from $\pi^\star_i$. This, in turn, leads to a bias depending on the value of $\bfw_i$ in the calculation of the gradient of $\Theta_i$ in minimizing \eqref{eqn:MSBE}.

\textbf{Generalized policy improvement (GPI).} 
Suppose we have acquired knowledge about the optimal successor features for the previous $n$ tasks, and we use $\hat{\psi}_i$ to denote the estimated successor feature function for the $i$-th task with $i\in[n]$. Now, let's consider a new task $\mathcal{T}_{n+1}$ with the reward function defined as $r_{n+1} = \boldsymbol{\phi} \bfw^\star_{n+1}$. Instead of training from scratch, we can leverage the knowledge acquired from previous tasks to improve our approach. We achieve this by deriving the policy as follows
\begin{equation}
\pi(a|\bfs) = \arg\max_{a}\max_{1\le i\le n+1} \hat{\psi}_i(\bfs,a)^\top\bfw^\star_{n+1}.
\end{equation}
This strategy tends to yield better performance than relying solely on $\hat{\psi}_{n+1}(\bfs,a)^\top\bfw^\star_{n+1}$, especially when $\hat{\psi}_{n+1}$ has not yet converged to the optimal successor feature $\psi_{n+1}^\star$ during the early learning stage, while some task is closely related to the new tasks, i.e., some $\bfw_{i}^\star$ is close to $\bfw_{n+1}^\star$.
This policy improvement operator is derived from Bellman's policy improvement theorem \cite{bertsekas1996neuro} and \eqref{eqn:b2}. 
{When the reward is fixed across different policies, e.g., $\{\pi_i\}_{i=1}^n$, and given that the optimal Q-function represents the maximum across the entire policy space, the maximum of multiple Q-functions corresponding to different policies, $\max_{1\le i \le n} Q^{\pi_n}$, is expected to be closer to $Q^\star$ than any individual Q-function, $Q^{\pi_i}$. In this paper, the parameter $\phi$ in learning the successor feature is analogous to the reward in learning the Q-function. As $\phi$ remains the same for different tasks, this analogy has inspired the utilization of GPI in our setting, even where the rewards change.}

\subsection{Successor feature Deep Q-Network}
The goal is to find $\bfw_i$ and $\Theta_i$ by solving the optimization problems in \eqref{eqn: optimization_w} and \eqref{eqn:MSBE} for each task sequentially, and the optimization problems are solved by mini-batch stochastic gradient descent (mini-batch SGD). 
Algorithm \ref{Alg} contains two loops, and the outer loop number $n$ is the number of tasks and inner loop number $T$ is the maximum number {of} iterations in solving \eqref{eqn: optimization_w} and \eqref{eqn:MSBE} for each task. At the beginning, we initialize the parameters as $\Theta^{(0)}$ and $\bfw_i^{(0)}$ for task $i$ with $1\le i \le n$.
In  $t$-th inner loop for the $i$-th task, let $\bfs_t$ be the current state, and $\theta_c$ be the learned weights for task $c$. The agent selects and executes actions according to
\begin{equation}\label{eqn: policy_update}
    a = \pi_\beta(\textstyle\max_{c\in[i]}\psi(\Theta_c;\bfs_t,a)^\top \bfw_i^{(t)}),
\end{equation}
 where $\pi_\beta(Q(\bfs_t,a))$ is the policy operator based on the function $Q(\bfs_t,a)$, e.g., greedy, $\varepsilon$-greedy, and softmax. For example, if $\pi_\beta(\cdot)$ stands for greedy policy, then $ a  = \arg\max_a \textstyle\max_{c\in[i]}\psi(\Theta_c;\bfs_t,a)^\top \bfw_i^{(t)}$. 
 The collected data are stored in a replay buffer with size $N$. 
 Then, we sample a mini-batch of samples from the replay buffer and denote the samples as $\mathcal{D}_t$. 

\begin{algorithm}[h]
\caption{Successor Feature Deep Q-Network (SF-DQN)}\label{Alg}
    \begin{algorithmic}
        \STATE \textbf{Input}: Number of iterations $T$, and experience replay buffer size $N$, step size $\{\eta_t,\kappa_t\}_{t=1}^T$. 
        \STATE  Initialize $\{\Theta_i^{(0)}\}_{i=1}^n$ and $\{\bfw_i^{(0)}\}_{i=1}^{n}$.
        % \STATE Initialize state $\bfS_0$.
        \FOR{Task $i = 1 , 2, \cdots ,n$}
        \FOR{$t = 0, 1, 2, \cdots, T-1$}
        % \State Select the initial weights $\bfW^{(t,0)}$.
        
        % With probability $\epsilon$, select a random action $A_t$, otherwise select $A_t = \argmax_{A} Q(\bfW^{(t,0)};\bfS_t,A)$;
        % \STATE 
        % Execute action $A_t$ in the emulator, and observe reward $R_t$ and the next state $\bfS_{t+1}$.
        % \STATE Store $(\bfS_t,A_t, R_t, \bfS_{t+1})$ in $\mathcal{D}_t$, and denote $\mathcal{D}_t = \{(\bfs_n,a_n,\bfr_n,\bfs'_n)\}_{n=1}^N$.
        % \For{$m = 0, 1, 2, \cdots, M-1$}
        \STATE Collect data and store in the experience replay buffer $\mathcal{D}_t$ following a behavior policy $\pi_t$ in \eqref{eqn: policy_update}.
        % \STATE Denoted the data in replay buffer as $\mathcal{D}_t = \{(\bfs_n,a_n,\bfr_n,\bfs'_n)\}_{n=1}^N$.
        % \State Sample random mini-batch of transition $\mathcal{D}_{t}^{(m)}$  from the replay buffer $\mathcal{D}_t$.
        % \State Set $\bfy_m = \boldsymbol{\phi}(\bfs_m,a_m,\bfs_{m}^\prime) + \gamma$ for $n \in 1, 2, \cdots, |\mathcal{D}_t|$.
        \STATE Perform gradient descent steps on $\Theta_i^{(t)}$ and $\bfw^{(t)}$ following \eqref{eqn: gradient_descent_both}.
        \ENDFOR
        \STATE Return $Q_i = \psi_i(\Theta_i^{(T)})^\top\bfw_i^{(T)}$ for $i=1,2,\cdots, n$.
        \ENDFOR
    \end{algorithmic}
\end{algorithm}

Next, denote the gradient as 
        $g_\bfw(\bfs_m,\bfa_m,
            \bfs_m';\bfw^{(t)})= \big( \boldsymbol{\phi}(\bfs_m,a_m,\bfs^\prime_m)^\top\bfw^{(t)} 
            -r(\bfs_m,a_m,\bfs_m^\prime) \big)\cdot \boldsymbol{\phi}(\bfs_m,a_m,\bfs_m^\prime)$
and $g_\Theta(\bfs_m,\bfa_m,
            \bfs_m';\Theta^{(t)})=
             \big(\psi(\Theta_i^{(t)};\bfs_m,a_m) - 
             \boldsymbol{\phi}(\bfs_m,a_m,\bfs_m^\prime) 
             -\gamma \cdot \psi(\Theta_i^{(t)};\bfs_m',a') \big)
             \cdot \nabla_{\Theta_i}\psi(\Theta_i^{(t)};\bfs_m,a_m)$, 
we update the current weights using a mini-batch gradient descent algorithm following
 \begin{equation}\label{eqn: gradient_descent_both}
        \begin{split}
        \bfw^{(t+1)}  
            &= \bfw^{(t)} -\kappa_t\cdot\sum_{m\in\mathcal{D}_t} g_\bfw(\bfs_m,\bfa_m,
            \bfs_m';\bfw^{(t)})
            \\
             \Theta_i^{(t+1)}&= \Theta_i^{(t)} - \eta_t\cdot\sum_{m\in\mathcal{D}_t} 
             g_\Theta(\bfs_m,\bfa_m,
            \bfs_m';\Theta^{(t)}),
        \end{split}
        \end{equation}
        where $\eta_t$ and $\kappa_t$ are the step sizes, and $a^\prime=\arg\max_a\max_{c\in[i]}\psi(\Theta_c;\bfs_m^\prime,a)^\top \bfw_i^{(t)}$. The gradient for $\Theta_i^{(t)}$, as $g_\Theta(\bfs_m,\bfa_m,
            \bfs_m';\Theta^{(t)})$ in \eqref{eqn: gradient_descent_both},  can be viewed as  the gradient of 
    \begin{equation}\label{eqn:MSBE2}
    \begin{split}
               \sum_{m\in \mathcal{D}_t} \big[\psi_i(\Theta_i;&\bfs_m,a_m) -\boldsymbol{\phi}(\bfs_m,a_m,\bfs_m') -\\ &\mathbb{E}_{\bfs_m'\mid\bfs_m,a_m}\max_{a_m'}\psi_i(\Theta_i^{(t)};\bfs_m',a_m') \big]^2,
    \end{split}
    \end{equation}
    which is the approximation to \eqref{eqn:MSBE} via replacing  $\max_{a'} \psi_i^\star$ with $\max_{a'} \psi_i(\Theta^{(t)}_i)$.

\section{Theoretical Results}\label{sec: theorem}
\subsection{Summary of Major Theoretical Findings}
To the best of our knowledge, our results (formally presented in Section \ref{subsec:main_theorem}) provide the first theoretical characterization for SF-DQN with GPI, including a comparison with the conventional Q-learning under commonly used assumptions.
Before formally presenting them, we summarize the highlights as follows.

\begin{table}[h]

\centering
    \caption{Important Notations}
    \vspace{-2mm}    
    \begin{tabular}{c|p{6.6cm}}  
    \hline
    \hline
    $~K~$ & Number of neurons in the hidden layer.\\
    \hline
    $~L~$ & Number of the hidden layers.\\
    \hline
    $~d~$ & Dimension of the feature mapping of $(\bfs,a)$.\\
    \hline
    $~T~$ & Number of iterations.\\
    \hline
    $\Theta_i^\star$, $\bfw_i^\star$ & The global optimal to \eqref{eqn:MSBE} and \eqref{eqn: optimization_w} for $i$-th task.\\
    \hline
    $N$ & Replay buffer size.\\
    \hline
    $\rho_1$ & The smallest eigenvalue of ${\mathbb{E} \nabla \boldsymbol{\psi}_i(\Theta_i^\star) \nabla \boldsymbol{\psi}_i(\Theta_i^\star)^\top}$.\\
    \hline
    $\rho_2$ & The smallest eigenvalue of $\mathbb{E}\boldsymbol{\phi}(\bfs,a)\boldsymbol{\phi}(\bfs,a)^\top$.\\
    \hline
    $q^\star$& A variable indicates the relevance between current and previous tasks.\\
    \hline
    $C^\star$ & A constant related to the distribution shift between the behavior and optimal policies.\\
    \hline
    \end{tabular}
    \label{table:problem_formulation}
\end{table}
% \begin{table}[h]
% \scriptsize
% \centering
% \vspace{-3mm}
%     \caption{Important Notations}
%     \vspace{-2mm}    
%     \begin{tabular}{c|p{5.8cm}||c|p{5.8cm}}  
%     \hline
%     \hline
%     $~K~$ & Number of neurons in the hidden layer.&$~L~$ & Number of the hidden layers.\\   
%     \hline
%         $~d~$ & Dimension of the feature mapping of $(\bfs,a)$.&
%     $~T~$ & Number of iterations.\\
%     \hline
%     $\Theta_i^\star$, $\bfw_i^\star$ & The global optimal to \eqref{eqn:MSBE} and \eqref{eqn: optimization_w} for $i$-th task. &
%     $N$ & Replay buffer size.\\
%     \hline
%     $\rho_1$ & The smallest eigenvalue of $\Revise{\mathbb{E} \nabla \boldsymbol{\psi}_i(\Theta_i^\star) \nabla \boldsymbol{\psi}_i(\Theta_i^\star)^\top}$. &
%     $\rho_2$ & The smallest eigenvalue of $\mathbb{E}\boldsymbol{\phi}(\bfs,a)\boldsymbol{\phi}(\bfs,a)^\top$.\\
%     \hline
%     $q^\star$& A variable indicates the relevance between current and previous tasks.&
%     $C^\star$ & A constant related to the distribution shift between the behavior and optimal policies.\\
%     \hline
%     \end{tabular}
%     \label{table:problem_formulation}
%     \vspace{-4mm}
% \end{table}

\textbf{(T1) Learned Q-function converges to the optimal Q-function at a rate of $1/T$ with generalization guarantees.}
We demonstrate that the learned parameters $\Theta_i^{(T)}$ and $\bfw_i^{(T)}$ converge towards their respective ground truths, $\Theta_i^\star$ and $\bfw_i^\star$, indicating that SF-DQN converges to optimal Q-function at a rate of $1/T$ as depicted in \eqref{eqn: thm1_Q} (Theorem \ref{Thm1}).
Moreover, the generalization error of the learned Q-function scales on the order of $\frac{\|\bfw^{(0)}-\bfw^\star\|_2}{1-\gamma - \Omega(N^{-1/2}) - \Omega(C^\star)}\cdot \frac{1}{T}$. By employing a large replay buffer $N$, minimizing the data distribution shift factor $C^\star$, and improving the estimation of task-specific reward weights $\bfw^{(0)}$, we can achieve a lower generalization error. 

\textbf{(T2)  GPI enhances the generalization of the learned model with respect to the task relevance factor $q^\star$.} 
We demonstrate that, when GPI is employed, the learned parameters exhibit improved estimation error with a reduction rate at $\frac{1-c}{1-c\cdot q^\star}$ for some constant $c<1$ (Theorem \ref{Thm2}), where $q^\star$ is defined in \eqref{eqn: q_star}. From \eqref{eqn: q_star}, it is clear that $q^\star$ decreases as the distances between task-specific reward weights, denoted as $\|\bfw_j^\star -\bfw_i^\star\|_2$, become smaller. This indicates a close relationship between the previous tasks and the current task, resulting in a smaller $q^\star$ and, consequently, a larger improvement through the usage of GPI.

% From \eqref{eqn: q_star}, we can see that $q^\star$ becomes small when the previous tasks and current task are close-revelant, namely, the distance between the task-specific reward weights are close to each other as $\|\bfw_j^\star -\bfw_i^\star\|_2$.

\textbf{(T3) SF-DQN achieves a superior performance over conventional DQN by a factor of $\gamma$ for the estimation error of the optimal Q-function.} 
% \MW{Performance in terms of what? sample complexity, reward?}
When we directly transfer the learned knowledge of the Q-function to a new task without any additional training, our results demonstrate that SF-DQN always outperforms its conventional counterpart, DQN, by a factor of $\gamma$ (Theorems \ref{Thm3} and \ref{Thm4}). 
As $\gamma$ approaches one, we raise the emphasis on long-term rewards, making the accumulated error derived from the incorrect Q-function more significant. Consequently, this leads to reduced transferability between the source tasks and the target task. Conversely, when $\gamma$ is small, indicating substantial potential for transfer learning between the source and target tasks, we observe a more significant improvement when using SF-DQN.

\subsection{Assumptions}\label{sec:ass}

In this section, we propose the assumptions in deriving our major theoretical results. These assumptions are commonly used in existing RL and neural network learning theories.
\begin{ass}\label{ass1}
There exists a deep neural network with weights $\Theta^\star_i$ {such that it minimizes \eqref{eqn:MSBE} for the $i$-th task, i.e, $f(\Theta_i^\star) = 0$}. 
\end{ass}
Assumption \ref{ass1} assumes a substantial expressive power of the deep neural network, allowing it to effectively represent $\psi^\star$ in the presence of an unknown ground truth $\Theta^\star$. 
% \footnote{Note that $\bfW^\star$ does not need to be unique. We abuse the notation $\|\bfW-\bfW^*\|_2$ to denote the minimum distance of $\bfW$ to any $\bfW^*$ that satisfies the assumption \ref{ass1}.}. 

\begin{ass}\label{ass2}
     At any fixed outer iteration $t$, the behavior policy $\pi_t$ and its corresponding transition kernel $\mathcal{P}_t$ satisfy
     \begin{equation}
     \begin{split}
        \textstyle\sup_{\bfs\in\mathcal{S}}~ d_{TV}\big( \mathbb{P}(\bfs_\tau\in \cdot)\mid \bfs_0 = \bfs), \mathcal{P}_t \big) \le \lambda \nu^\tau,
     \end{split}
     \end{equation}
     for some constant $\lambda>0$ and $\nu\in(0,1)$, where $d_{TV}$ denotes the total-variation distance.
\end{ass}
Assumption \ref{ass2} assumes the Markov chain $\{\bfs_n,a_n,\bfs_{n+1}\}$ induced by the behavior policy is uniformly ergodic with the corresponding invariant measure $\mathcal{P}_t$. {This assumption is standard in Q-learning \cite{XG20, ZWL19,BRS18}, where the data are non-i.i.d.}

% Compared with i.i.d. cases, we need to handle an additional error term when bounding the distance between the $g_t$ and $\nabla f$ as shown in \eqref{eqn: Lemma2_temp}. Therefore, the upper bound in Lemma \ref{Lemma: first_order_derivative} changes, which suggests an additional term in the final bound.

% $C_t$ 
% can be viewed as the difference between behavior policy and optimal policy. Theorem \ref{coro: convergence_outer} shows the general results for any value of $C_t$. Nevertheless, the greedy policy is improved over time, e.g., updating the weights every few steps (line 11 in Algorithm \ref{main_alg}). 
% Hence, $C_t$ depends on $\bfW^{(t,0)}$ and is expected to decrease as $\bfW^{(t,0)}$ approaching $\bfW^\star$ \cite{PP02,ZWL19}, which will be discussed in Corollary \ref{coro: C_t}.
\begin{ass}\label{ass3}
Let $Q^\star$ and $Q_t$ be the optimal and estimated Q-function, respectively. We assume the greedy policy $\pi_t$, i.e., $\pi_t(a|\bfs) = \arg\max_{a'} Q_t(\bfs,a')$, satisfies 
\begin{equation}\label{eqn:C1}
\vspace{-1mm}
\begin{split}
        &\big| \pi_t(a|\bfs) -\pi^\star(a|\bfs) \Big|\\
        \le&~C\cdot \textstyle\sup_{(\bfs,a)} \|Q_t(\bfs,a) - Q^\star(\bfs,a)\|_F,
\end{split}
\end{equation}
where $C$ is a positive constant. 
Equivalently, when $Q_t = \psi(\Theta_i^{(t)})^\top\bfw_i^{(t)}$, we have 
\begin{equation}\label{eqn:C}
\vspace{-1mm}
\begin{split}
        &\big| \pi_t(a|\bfs) -\pi^\star(a|\bfs) \Big|\\ 
    \le& ~C\cdot\big(\|\Theta_i^{(t)}-\Theta_i^\star\|_2 + \|\bfw_i^{(t)}-\bfw_i^\star\|_2\big).
\end{split}
\end{equation}
% $C_{t}\in[0,1]$ is the  fraction of non-optimal state-action pair $(\bfs,a)$ in the greedy policy with respect to $Q(\bfW^{(t,0)})$, i.e., % Specifically, 
%   the fraction of  $(\bfs,a)$ pairs that satisfy
%   \vspace{-1mm}
%   \begin{equation}
%   \vspace{-0.5mm}
%   a = \argmax_{a'} {Q}(\bfW^{(t,0)};\bfs,a')\neq \pi^\star(\bfs)
%   \end{equation}
%   among all $(\bfs,a)$ pairs following the greedy policy at $t$-th outer loop as $a = \argmax_{a'} {Q}(\bfW^{(t,0)};\bfs,a')$.
  % \begin{equation}
  % a = \argmax_{a'} {Q}(\bfW^{(t,0)};\bfs,a').
  %   \end{equation}
%  we denote the fraction of non-optimal station-action pair $(\bfs,a)$, i.e., $a \neq \pi^\star(\bfs)$, as $C_{t}\in[0,1]$.
\end{ass}
Assumption \ref{ass3} indicates the policy difference between the behavior policy and the optimal policy. Moreover, \eqref{eqn:C} can be considered as a more relaxed variant of condition (2) in \cite{ZWL19} as \eqref{eqn:C} only requires the holding for the distance of an arbitrary function from the ground truth, rather than the distance between two arbitrary functions.
% In \cite{ZWL19}, the behavior policy is assumed to Lipschitz continuous with respect to the parameters of Q-networks. 

\subsection{Main Theoretical Findings}\label{subsec:main_theorem}
\subsubsection{Convergence analysis of SF-DQN}
Theorem \ref{Thm1} demonstrates that the learned Q function converges to the optimal Q function when using SF-DQN for Task 1. Notably, GPI is not employed for the initial task, as we lack prior knowledge about the environment.  Specifically, given conditions (i) the initial weights for $\psi$ are close to the ground truth as shown in  \eqref{mainthm_initial}, (ii) the replay buffer is large enough as in \eqref{mainthm_sample}, {and (iii) the distribution shift between the behavior policy and optimal policy is bounded (as shown in Remark)}, the learned parameters from Algorithm \eqref{Alg} for task 1, $\psi_1(\Theta_1)$ and $\bfw_1$, converge to the ground truth $\psi^\star_1$ and $\bfw^\star_1$ as in  \eqref{eqn: theta_w}, indicating that the learned Q function converges to the optimal Q function as in \eqref{eqn: thm1_Q}.

\begin{theorem}[Convergence analysis of SF-DQN without GPI]\label{Thm1}
{Suppose the assumptions in Section \ref{sec:ass} hold and the initial neuron weights of the SF of task $1$ satisfy 
%\TC{What is $K$ here?}
\vspace{-1mm}
\begin{equation}\label{mainthm_initial}
\begin{split}
    \frac{\|\Theta_1^{(0)}-\Theta_1^\star\|_F}{\|\Theta_1^\star\|_F}  
    \le ~(1-c_N)\cdot \frac{\rho_1}{K^{2}},
\end{split}
\end{equation}
for some positive $c_N$.
When we select the step size as $\eta_t = \frac{1}{t+1}$, and the size of the replay buffer is 
\vspace{-1mm}
    \begin{equation}\label{mainthm_sample}
        N = \Omega(c_N^{-2}\rho_1^{-1}\cdot K^2\cdot L^2  d \log q).
\end{equation}
 Then, with the high probability of at least $1-q^{-d}$, the weights $\W[T]$ from Algorithm \ref{Alg} satisfy  
    \begin{equation}\label{eqn: theta_w}
    \begin{gathered}
        \|\Theta_1^{(T)} -\Theta_1^\star\|_2 \le  \frac{C_1 + C^\star\cdot \|\bfw_1^{(0)} -\bfw_1^\star\|_2}{(1-\gamma-c_N)(1-\gamma)\rho_1-C^\star}\cdot \frac{\log^2 T}{T},\\
        \|\bfw_1^{(T)} - \bfw_1^\star\|_2 \le \Big(1-\frac{\rho_2}{\boldsymbol{\phi}_{\max}}\Big)^T \|\bfw_1^{(0)} -\bfw_1^\star \|_2,
    \end{gathered}
    \end{equation}
    % \MW{The coefficients in (22) and (23) are complicated and repeating.  It is also similar to the one in (25). Maybe can define two constants separately to simplify (22), (23), and (25), and then discuss how GPI affects these two constants.}
        where $C_1 = (2+\gamma)\cdot R_{\max} $, and $C^\star =  |\mathcal{A}|\cdot {R_{\max}}\cdot (1+\log_\nu \lambda^{-1}+\frac{1}{1-\nu})\cdot C$. Specifically, the learned Q-function satisfies}
    \begin{equation}\label{eqn: thm1_Q}
    \begin{split}
        &\max_{\bfs,a} \Big|  \Revise{Q_1^{(T)}} -Q^\star\Big|\\
        \le&~~\frac{C_1 + \|\bfw_1^{(0)} -\bfw_1^\star\|_2}{(1-\gamma-c_N)(1-\gamma)\rho_1-1}\cdot  \frac{\log^2 T}{T}\\
        &+ \frac{\|\bfw_1^{(0)} -\bfw_1^\star \|_2 R_{\max}}{1-\gamma} \Big(1-\frac{\rho_2}{\boldsymbol{\phi}_{\max}}\Big)^T. 
    \end{split} 
    \end{equation}
\end{theorem}
\textbf{Remark 1} (\textbf{upper bound of $C$}):  To ensure the meaningfulness of the upper bound in \eqref{eqn: thm1_Q}, specifically that the denominator needs to be greater than $0$, $C$ has an explicit upper bound as $C \leq \frac{(1-\gamma-c_N)(1-\gamma)\rho_1}{|\mathcal{A}|\cdot R_{\max}}$. Considering the definition of $C$ in Assumption \ref{ass3}, it implies that the difference between the behavior policy and the optimal policy is bounded. In other words, the fraction of bad tuples \footnote{A “bad tuple” refers to the data $(\bfs, a)$ collected based on behavior policy $a = \pi_t(\bfs)$ that differs from the optimal policy $a = \pi^\star(\bfs)$. Intuitively, we can clearly see that the fraction of ``bad tuples'' is positively related to the distance between the behavior policy and the optimal policy (the motivation of Assumption \ref{ass3}). In fact, similar assumptions can be found in many theoretical frameworks when analyzing Q-learning with function approximation \cite{ZWL19} to guarantee that there is a certain fraction of collected data that is useful for estimating the ground-truth Q-value.} in the collected samples is constrained.

\textbf{Remark 2} (\textbf{Initialization}): Note that \eqref{mainthm_initial} requires a good initialization. Firstly, it is still a state-of-the-art practice in analyzing Q-learning via deep neural network approximation. Secondly, according to the NTK theory \cite{JGH18}, there always exist some good local minima, which is almost as good as the global minima, near some random initialization. Finally, such a good initialization can also be adapted from some pre-trained models.

\subsubsection{Improved performance with GPI.}
Theorem \ref{Thm2} establishes that the estimated Q function converges towards the optimal solution with the implementation of GPI as shown in \eqref{eqn: thm2}, leveraging the prior knowledge learned from previous tasks. The enhanced performance associated with GPI finds its expression as $q^\star$ defined in \eqref{eqn: q_star}. Notably, when tasks $i$ and $j$ exhibit a higher degree of correlation, meaning that the distance between $\bfw^\star_i$ and $\bfw^\star_j$ for tasks $i$ and $j$ is smaller, we can observe a more substantial enhancement by employing GPI in transferring knowledge from task $i$ to task $j$ from \eqref{eqn: thm2}. 
\begin{theorem}[Convergence analysis of SF-DQN with GPI]\label{Thm2}
    Let us define 
    \begin{equation}\label{eqn: q_star}
        q^\star = \frac{(1+\gamma)R_{\max}}{1-\gamma}\cdot \frac{ \min_{1\le i\le j-1}~\|\bfw_i^\star - \bfw_j^\star\|_2  }{\|\Theta_j^{(0)}-\Theta_j^\star\|_2}.
    \end{equation}
    Then, with the probability of at least $1-q^{-d}$, the neuron weights $\Theta^{(T)}_j$ for the $j$-th task  satisfy
    \begin{equation}\label{eqn: thm2}
    \begin{split}
        &\|\Theta^{(T)}_j -\Theta^\star_j\|_2\\ 
        \le& \frac{C_1 + C^\star\|\bfw_j^{(0)} -\bfw_j^\star\|_2}{(1-\gamma-c_N)(1-\gamma)\rho_1-\min\{q^\star,1\}\cdot C^\star} \cdot  \frac{\log^2 T}{T}.
    \end{split}
    \end{equation}
\end{theorem} 
\textbf{Remark 3} (\textbf{Improvement via GPI}):
Utilizing GPI enhances the convergence rate from in the order of $\frac{1}{1-C^\star}\cdot \frac{1}{T}$ to in the order of $\frac{1}{1-q^\star \cdot C^\star}\cdot \frac{1}{T}$. When the distance between the source task and target tasks is small, $q^\star$ can approach zero, indicating an improved generalization error by a factor of $1-C^\star$, where $C^\star$ is proportional to the fraction of bad tuples. The improvement achieved through GPI is derived from the reduction of the distance between the behavior policy and the optimal policy, subsequently decreasing the fraction of bad tuples in the collected data. Here, $C^\star$ is proportional to the fraction of bad tuples without using GPI, and $q^\star\cdot C^\star$ is proportional to the fraction of bad tuples when GPI is employed.

\subsubsection{Improved Performance with the Knowledge Transfer}
% From Theorems \ref{Thm1} and \ref{Thm2}, we have successfully 
Using our proposed SF-DQN, we have estimated $Q^{\pi^\star_i}_i$ for task $i$. When the reward changes to $\Revise{r_{n+1}}(\bfs,a,\bfs^\prime) = \boldsymbol{\phi}^\top (\bfs,a,\bfs^\prime)\bfw^\star_{n+1}$ for a new task $\mathcal{T}_{n+1}$, and once $\bfw^\star{n+1}$ is estimated, we can calculate the estimated Q-value function for $\mathcal{T}_{n+1}$ by setting
\begin{equation}\label{eqn: transfer_SF}
    Q^{\pi_{n+1}}_{n+1}(\bfs,a) = \max_{1\le j\le n} \psi(\Theta_j^{(T)};\bfs,a)\bfw^{\star}_{n+1}.
\end{equation}
As $\bfw^{(t)}_{n+1}$ experiences linear convergence to its optimal $\bfw^\star$, which is significantly faster than the sub-linear convergence of $\Theta^{(t)}_{n+1}$, as shown in \eqref{eqn: theta_w}, this derivation of $Q_{n+1}$ in \eqref{eqn: transfer_SF} simplifies the computation of $\Theta_{n+1}^\star$ into a much more manageable supervised setting for approximating $w^\star_{n+1}$ with only a modest performance loss as shown in \eqref{eqn: thm3_main}. This is demonstrated in the following Theorem \ref{Thm3}. 
\begin{theorem}[Transfer learning via SF-DQN]\label{Thm3}
For the $(n+1)$-th task with $r_{n+1} = \boldsymbol{\phi}^\top w^\star_{n+1}$, suppose the Q-value function is derived based on \eqref{eqn: transfer_SF},  we have  
    \begin{equation}\label{eqn: thm3_main}
    \begin{split}
        &\max_{\bfs,a} |Q_{n+1}^{\pi_{n+1}}(\bfs,a) -Q_{n+1}^\star(\bfs,a)|\\ 
        \le&~\frac{2\gamma}{1-\gamma}\phi_{\max} \min_{j\in[n]}\|\bfw_j^\star -\bfw_{n+1}^\star\|_2
        + \frac{\|\bfw_{n+1}^\star\|_2}{(1-\gamma)\cdot T}. 
    \end{split}
    \end{equation}
\end{theorem}
\textbf{Remark 4 (Connection with existing works of SF in tabular cases):} The second term of the upper bound in \eqref{eqn: thm3_main}, $\frac{\|\bfw_{n+1}^\star\|_2}{(1-\gamma)\cdot T}$, characterizes the value of $\epsilon$ assumed in \cite{barreto2017successor}, which results from the approximation error of the optimal Q-functions in the previous tasks \footnote{Our upper bound in \eqref{eqn: thm3_main} differs from the one in \cite{barreto2017successor} in the first term. This distinction arises from our improvement in Lemma \ref{lemma: DQN_difference} compared to Lemma 1 in \cite{barreto2017successor}. See Appendix \ref{app: GPI} for the proof of Lemma \ref{lemma: DQN_difference}.}.

Without the SF decomposition as shown in \eqref{eqn: SF_Q}, one can apply a similar strategy in \eqref{eqn: transfer_SF} for DQN as
\begin{equation}\label{eqn: transfer_DQN}
    Q^{\pi_{n+1}^\prime}_{n+1}(s,a) = \max_{1\le j\le n} Q(\omega_j^{(T)};s,a).
\end{equation}
In Theorem \ref{Thm4}, \eqref{eqn: thm4_main} illustrates the performance of \eqref{eqn: transfer_DQN} through DQN. Compared to Theorem \ref{Thm3}, transfer learning via DQN is worse than that via SF-DQN by a factor of $\frac{1+\gamma}{2}$ when comparing the estimation error of the optimal function $Q_{n+1}^\star$ in \eqref{eqn: thm3_main} and \eqref{eqn: thm4_main}, indicating the advantages of using SFs in transfer reinforcement learning.
\begin{theorem}[Transfer learning via DQN]\label{Thm4}
For the $(n+1)$-th task with $r_{n+1}= \phi\cdot w^\star_{n+1}$, suppose the Q-value function is derived based on \eqref{eqn: transfer_DQN},  we have 
% \vspace{-1mm}
    \begin{equation}\label{eqn: thm4_main}  
    \begin{split}
        &\max_{(\bfs,a)}: |Q_{n+1}^{\pi_{n+1}^\prime}(\bfs,a) -Q_{n+1}^\star(\bfs,a) | \\
        \le& \frac{2}{1-\gamma}\phi_{\max} \cdot \min_{j\in[n]}\|\bfw_j^\star -\bfw_{n+1}^\star\|_2
        + \frac{\|\bfw_{n+1}^\star\|_2}{(1-\gamma)\cdot T}. 
    \end{split}
    \end{equation}
\end{theorem}
\textbf{Remark 5 (Improvement by a factor of $\frac{1+\gamma}{2}$):} Transfer learning performance in SF-DQN is influenced by the knowledge gap between previous and current tasks, primarily attributed to differences in rewards and data distribution. In SF-DQN, the impact of reward differences is relatively small since $\phi$ that plays the role of reward remains fixed. The parameter $\gamma$ affects the influence of data distribution differences. A small $\gamma$ prioritizes immediate rewards, thereby the impact of data distribution on the knowledge gap is not significant. With a small $\gamma$, the impact of reward difference dominates, resulting in a high gap between SF-DQN and DQN in transfer learning.

\begin{table*}[t]
    \centering
    \caption{Normalized average reward for SF-DQN with and without GPI.}
    \setlength{\tabcolsep}{1.0em} % for the horizontal padding
{\renewcommand{\arraystretch}{1.2}% for the vertical padding
    \begin{tabular}{ | c |c | c | c | c | }
    \hline
         $\| w^*_1 - w^*_2 \|$& $=0.01$ & $=0.1$ &  $=1$ & $=10$ \\
         \hline %& $=10$ & $=100$
         SF-DQN (w/ GPI) 
         & $\bm{0.986} \pm 0.007$ & $\bm{0.965} \pm 0.007$  &  $\bm{0.827} \pm 0.008$ & $\bm{0.717} \pm 0.012$  \\ %0.62 & 0.40
         SF-DQN (w/o GPI) 
         & $0.942 \pm 0.004$  & $0.911 \pm 0.013$  & $0.813 \pm 0.009$ & $0.707 \pm 0.011$ \\
         % GPI \% 
         % & 1.49 & 1.62 & 1.62 & 1.34 & 0.82 \\
         \hline %& 0.62 & 0.40
    \end{tabular}}    
    \label{tab:gpi-effect}
\vspace{-0.2cm}
\end{table*}

\subsection{Technical Challenges, and Comparison with Existing Works}
\textbf{Beyond deep learning theory: challenges in deep reinforcement learning.}
The proof of Theorem \ref{Thm1} is inspired by the convergence analysis of one-hidden-layer neural networks within the semi-supervised learning \cite{ZSJB17,ZWLCX22} and a recent theoretical framework in analyzing DQN \cite{zhang2024convergence}. This proof tackles \textit{two primary objectives}: (\textbf{i}) the first objective involves characterizing the local convex region of the objective functions presented in \eqref{eqn:MSBE} and \eqref{eqn: optimization_w}; (\textbf{ii}) the second objective focuses on quantifying the distance between the gradient defined in \eqref{eqn: gradient_descent_both} and the gradient of the objective functions in \eqref{eqn:MSBE} and \eqref{eqn: optimization_w}. 

However, extending this approach from the semi-supervised learning setting to the deep reinforcement learning domain introduces \textit{additional challenges}.
First, we expand our proof beyond the scope of one-hidden-layer neural networks to encompass multi-layer neural networks. This extension requires new technical tools for characterizing the Hessian matrix and concentration bounds, as outlined in Appendix \ref{sec: Proof: second_order_derivative}.
Second, the approximation error bound deviates from the supervised learning scenarios due to several factors: the non-i.i.d. of the collected data, the distribution shift between the behavior policy and the optimal policy, and the approximation error incurred when utilizing \eqref{eqn:MSBE2} to estimate \eqref{eqn:MSBE}. Addressing these challenges requires developing supplementary tools, as mentioned in Lemma \ref{Lemma: first_order_derivative}. 
Notably, this approximation does not exhibit scaling behavior proportional to $\|\Theta_i-\Theta_i^\star\|_2$, resulting in a sublinear convergence rate.

% . By satisfying certain conditions such as having sufficient training samples and a bounded data distribution shift, the approximation error between the PRF and objective function can be bounded. This allows for the characterization of the optimization problem in \eqref{eqn:MSBE} by analyzing the landscape and convergence properties of the PRF.

\textbf{Beyond DQN: challenges in GPI.}
The major challenges in proving Theorems \ref{Thm2}-\ref{Thm4} centers on deriving the improved performance by utilizing GPI. The intuition is as follows. 
Imagine we have two closely related tasks, labeled as $i$ and $j$, with their respective optimal weight vectors, $\bfw_i^\star$ and $\bfw_j^\star$, being close to each other. This closeness suggests that these tasks share similar rewards, leading to a bounded distributional shift in the data, which, in turn, implies that their optimal Q-functions should exhibit similarity.
To rigorously establish this intuition, we aim to characterize the distance between these optimal Q-functions, denoted as $|Q_i^\star - Q_j^\star|$, in terms of the Euclidean distance between their optimal weight vectors, $||\bfw_i^\star - \bfw_j^\star||_2$ (See details in Appendix \ref{app: GPI}).
Furthermore, we can only estimate the optimal Q-function for previous tasks during the learning process, and such an estimation error accumulates in the temporal difference learning, e.g., the case of the SF learning of $\psi^\star$. We developed novel analytical tools to quantify the error accumulating in the temporal difference learning (see details in Appendix \ref{app: thm34}), which is not a challenge for previous works in the supervised learning setting.

% Because of the estimation error, we need to develop novel analytical tools to quantify the error accumulating via the temporal difference learning process (see details in Appendix \ref{app: thm34}).

\section{Experiments}\label{sec:experiments}
This section summarizes empirical validation for the theoretical results obtained in Section \ref{sec: theorem} using a synthetic RL environment. {The experiment setup and additional experimental results for real-world RL benchmarks are summarized in Appendix \ref{app:experiment}.}

\textbf{Convergence of SF-DQN with varied initialization.}
Figure \ref{fig:w-init-gap1} shows the performance of Algorithm \ref{Alg} with different initial $\bfw_1^{(0)}$ to the ground truth $\bfw^\star_1$.
% For this purpose, we initialize both $\Theta_1$ and $\bfw_1$, and train them using Algorithm \ref{Alg}. We use $w^\star_1$ defined at the experiment setup to compute $\| \bfw^{(0)}_1 - \bfw^\star_1\|$, and the reward defined by $\phi \cdot w^\star_1$ to obtain the average reward for Task 1. We repeat the process for different choices of $\bfw^{(0)}_1$; see the results in Figure \ref{fig:w-init-gap}. 
When the initialization is close to the ground truth, we observe an increased accumulated reward, which verifies our theoretical findings in \eqref{eqn: thm1_Q} that the estimation error of the optimal Q-function reduces as $\|\bfw_1^{(0)}-\bfw^\star\|_2$ decreases.

\begin{figure}[h]
\centering
\begin{minipage}{0.48\linewidth}
    \centering
         \includegraphics[width=1.0\linewidth]{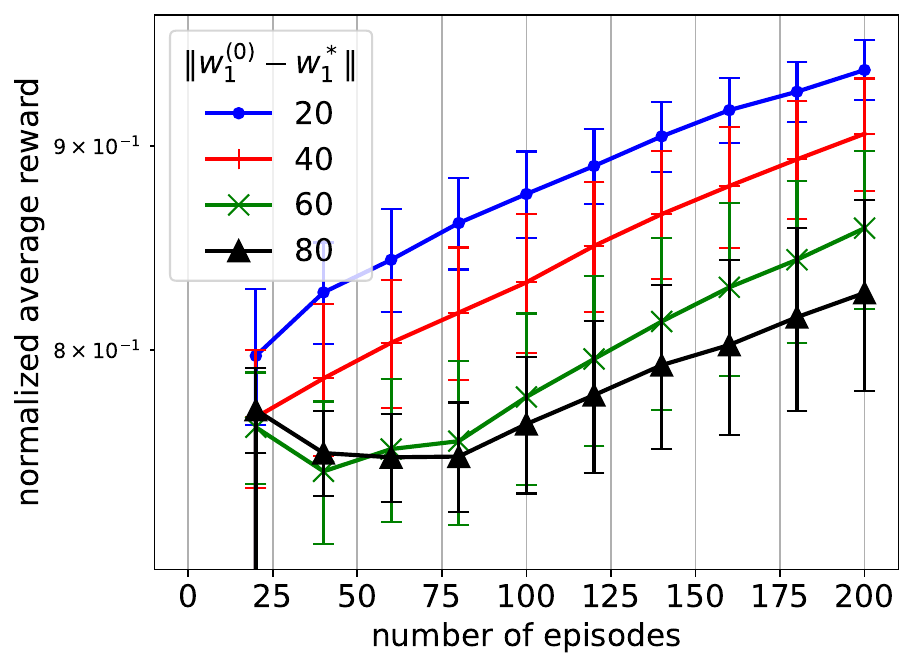}
         \caption{Performance of SF-DQN presented in Algorithm \ref{Alg} on Task 1.}
         \label{fig:w-init-gap1}
\end{minipage}
\hfill
\begin{minipage}{0.48\linewidth}
    \centering
         \includegraphics[width=1.0\linewidth]{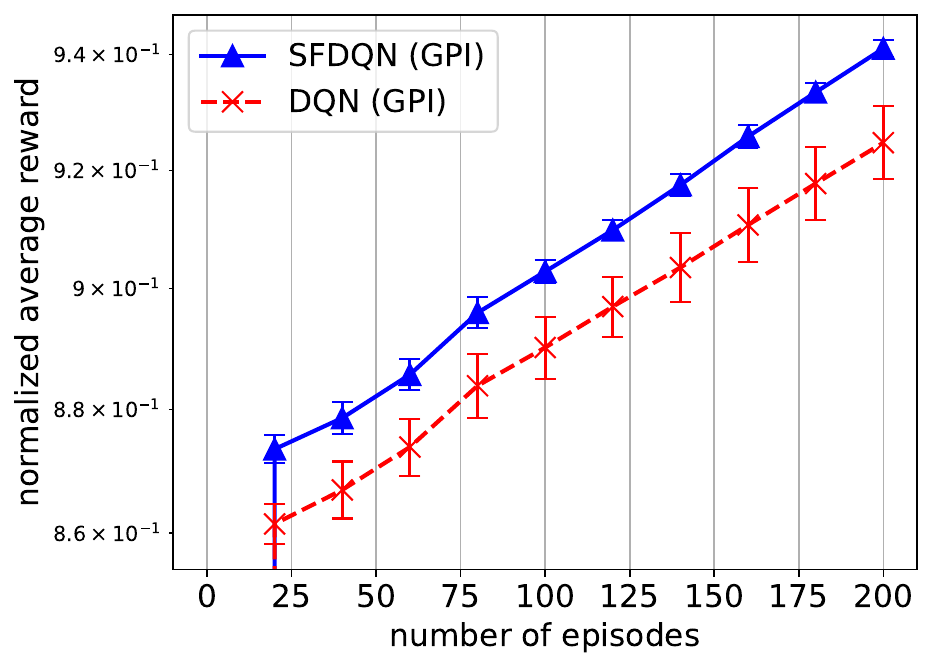}
         \centering

         \caption{Transfer comparison for SF-DQN and DQN (with GPI)}
         \label{fig:com_sfdqn_dqn}
\end{minipage}
\end{figure}

\textbf{Performance of SF-DQN with GPI when adapting to tasks with varying relevance.} 
We conducted experiments to investigate the impact of GPI with varied task relevance.  Since the difference in reward mapping impacts data distribution shift, rewards, and consequently the optimal Q-function, we utilize the metric $\|\bfw_1^\star -\bfw_2^\star\|_2$ to measure the task irrelevance.
The results summarized in Table \ref{tab:gpi-effect} demonstrate that when tasks are similar (i.e., small $\|\bfw^\star_{1} - \bfw^\star_2\|$), SF-DQN with GPI consistently outperforms its counterpart without GPI. However, when tasks are dissimilar (i.e., large $\|\bfw^\star_1 - \bfw^\star_2\|$), both exhibit the same or similar performance, indicating that GPI is ineffective when two tasks are irrelevant. 
The observations in Table 2 validate our theoretical findings in \eqref{eqn: thm2}, showing a more significant improvement in using GPI as  $\|\bfw^\star_1 - \bfw_2^\star\|_2$ decreases.

\textbf{Comparison of the SF-DQN agent and DQN agent.} From Figure \ref{fig:com_sfdqn_dqn}, it is evident that the SF-DQN agent consistently achieves a higher average reward (task 2) than the DQN when starting training on task 2, where transfer learning occurs. These results strongly indicate the improved performance of the SF-DQN agent over the DQN, aligning with our findings in \eqref{eqn: thm3_main} and \eqref{eqn: thm4_main}. SF-DQN benefits from reduced estimation error of the optimal Q-function compared to DQN when engaging in transfer reinforcement learning for relevant tasks.

\vspace{-0.2cm}
\section{Conclusions}\label{sec: conclusion}
This paper analyzes the transfer learning performance of SF \& GPI, with SF being learned using deep neural networks. Theoretically, we present a convergence analysis of our proposed SF-DQN with generalization guarantees and provide theoretical justification for its superiority over DQN without using SF in transfer reinforcement learning.
We further verify our theoretical findings through numerical experiments conducted in both synthetic and benchmark RL environments. {Future directions include exploring the possibility of learning $\phi$ using a DNN approximation and exploring the combination of successor features with other deep reinforcement learning algorithms.}

\section*{Acknowledgment}
Part of this work was done when Shuai Zhang was a postdoc at Rensselaer Polytechnic Institute (RPI). 
 This work was supported by AFOSR FA9550-20-1-0122, ARO W911NF-21-1-0255, NSF 1932196, NSF CAREER project 2047177, Cisco Research Award, and the Rensselaer-IBM AI Research Collaboration (http://airc.rpi.edu), part of the IBM AI Horizons Network (http://ibm.biz/AIHorizons).   We thank all anonymous reviewers for their constructive comments.

\section*{Broader Impact}
This paper presents work whose goal is to advance the field of Machine Learning. There are many potential societal consequences of our work, none of which we feel must be specifically highlighted here.
% \subsubsection*{Acknowledgments}
% Use unnumbered third level headings for the acknowledgments. All
% acknowledgments, including those to funding agencies, go at the end of the paper.

\newpage
\appendix
\onecolumn
~
\vspace{0.5in}
\begin{center}
\Large Appendix of ``Provable Knowledge Transfer using Successor Feature for Deep Reinforcement Learning.''
\end{center}

\vspace{0.5in}

Before moving into the technical details, we provide an overview of the structure of the appendix.  

In Appendix \ref{app:notation}, we define some notations and useful lemmas to simplify the presentation and analysis. Some important notations for understanding the proof is summarized in Table \ref{table:appendix}.

In Appendix \ref{app:thm1}, we provide some preliminary lemmas and proof for Theorem \ref{Thm1}. A proof sketch is included as (\textbf{i}) characterization of the local convex region of the objective function in \eqref{eqn:MSBE} and \eqref{eqn: optimization_w} (Lemma \ref{Lemma: second_order_derivative}), (\textbf{ii}) Characterization of the difference between the empirical gradient in \eqref{eqn: gradient_descent_both} and the gradient of the objective function (Lemma \ref{Lemma: first_order_derivative}), (\textbf{iii}) Characterization of the relation of two consecutive iterations $\Theta^{(t+1)}$ and $\Theta^{(t)}$ in \eqref{eqn: temp1.1}, and (\textbf{iv}) Mathematical induction over $(t+1)\cdot \|\Theta^{(t)}-\Theta^\star\|_2$ from $t=1$ to $T$ to obtain the error bound between the learned model weights $\Theta^{(T)}$ and the optimal $\Theta^\star$. 

In Appendix \ref{app: thm34}, we provide the proof for Theorems \ref{Thm3} and \ref{Thm4}. A proof sketch is included as follows: (1) Characterization of \eqref{eqn: transfer_SF} by assuming knowledge of the optimal Q-function for previous tasks.
(2) Characterization of the accumulated error resulting from the estimation error of the learned Q-function in previous tasks.
(3) Combining the bounds from (1) and (2) leads to the error bound between \eqref{eqn: transfer_SF} derived from the estimated Q-function of previous tasks and the optimal Q-function for the new tasks.

In Appendix \ref{app:thm2}, we provide the proof for Theorem \ref{Thm2}. The proof sketch is a direct application of the existing results of the convergence analysis as shown in Appendix \ref{app:thm1} and the error bound between \eqref{eqn: transfer_SF} derived from the estimated Q-function of previous tasks and the optimal Q-function for the new tasks as shown in  Appendix \ref{app: thm34}.

In Appendix \ref{app:experiment}, we provide additional experiments to further support the proposed SF-DQN in Algorithm \ref{Alg} and our theoretical findings. 

In Appendix \ref{app:proof_of_lemma_1}, we provide the proofs for the preliminary lemmas in proving Theorems \ref{Thm1} and \ref{Thm2}.

In Appendix \ref{app: GPI}, we provide the proofs for the preliminary lemmas in proving Theorems \ref{Thm3} and \ref{Thm4}.

In Appendix \ref{app: proof}, we provide the proof for some additional lemmas.

\section{Notations and preliminary results}\label{app:notation}

    \textbf{Population risk function.} We define a population risk function as
    \begin{equation}\label{eqn:prf}
        f_{\pi^\star}(\theta):=  \mathbb{E}_{(\bfs,a)\sim\pi^\star}\big\| \psi(\theta;\bfs,a) - \mathbb{E}_{\bfs^\prime|(\bfs,a), a'\sim\pi^\star(\bfs^\prime)}\big(\phi(\bfs,a,\bfs^\prime) + \gamma \cdot \psi(\theta^\star;\bfs',a') \big) \big\|^2_2.
    \end{equation}
    We can see that $\theta^\star$ is the global minimal to \eqref{eqn:prf} with Assumption \ref{ass1}. For the convenience of presentation, we simplify $f_{\pi^\star}$ as $f$ in the supplementary materials. 
    % thus $\nabla f(\theta^\star) =0$. 

% Let us define the population risk function as 
% \begin{equation}\label{eqn:prf}
%   \begin{split}
%       f(\bfW) 
%       =&~\mathbb{E}_{(\bfs,a)\sim \mu^\star} \big[ Q(\bfW;\bfs,a) - r(\bfs,a) - \gamma \cdot \mathbb{E}_{\bfs'|\bfs,a}  \max_{a'\in\mathcal{A}} Q(\bfW^\star;\bfs',a') \big]^2\\
%       = &~\mathbb{E}_{(\bfs,a)\sim \mathcal{\mu}^\star} \big[ Q(\bfW;\bfs,a) - Q(\theta^\star;\bfs,a) \big]^2,
%   \end{split}
% \end{equation}
% where $\mu^*$ is the distribution of the sampled data following the optimal policy $\pi^\star$.

Then, the gradient of \eqref{eqn:prf} is 
 \begin{equation}\label{eqn:prf_gradient}
     \begin{split}
         &\nabla f_{\pi^\star}(\theta)\\
         =&~\mathbb{E}_{(\bfs,a)\sim\pi^\star,\bfs'|(\bfs,a)\sim \mathcal{P},a^\prime\sim \pi^\star} \big( \psi(\theta;\bfs,\bfa) - \phi(\bfs,\bfa,\bfs^\prime) - \gamma \cdot \psi(\theta^\star;\bfs',a')\big)\cdot \nabla \psi(\theta;\bfs,a).
     \end{split}
 \end{equation}

% As $\theta^\star$ is one of the ground truths to $f(\theta)$, i.e., $f(\theta^\star)$ achieves the minimum value as $f(\theta^\star)=0 \le f(\theta)$ for any other $\theta$. 
Given $f$ is a smooth function, we have the gradient of $f$ with respect to any $\theta_\ell$ at the ground truth $\theta^\star$ equals to zero, namely,
\begin{equation}\label{eqn: derivative_ell}
\nabla_\ell f(\theta^\star) := \nabla_{\theta_\ell} f(\theta^\star) =\bfzero, \textit{\quad\quad } \forall \ell \in [L].    
\end{equation}

\textbf{Vectorized Gradient of $\theta$ and $\bfw$ at iteration $t$.}  
To avoid unnecessary high-dimensional tensor analysis, the gradient with respect to $\theta$, denoted as $\nabla_{\theta} H$ for some function $H$, is represented as its corresponding vectorized version, $\nabla_{\text{Vec}(\theta)} H$.

Let $n$ denote the dimension of $\bfW$ defined in \eqref{eqn: Q}. We denote $n_l$ as the dimension of the vectorized neuron weights in the $\ell$-th layer, namely, $n_\ell = \text{dim(vec($\theta_\ell$))}$.

Then, the gradient in updating $\theta$ as
\begin{equation}\label{eqn: gradient_theta}
\begin{split}
    &g^{(t)}(\theta^{(t)};\mathcal{D}_t)\\
    =&~\sum_{m\in\mathcal{D}_t} \Big(\psi(\theta^{(t)};\bfs_m,a_m) -\phi(\bfs_m,a_m,\bfs_m^\prime) -\gamma \cdot \psi(\theta^{(t)};\bfs_m',a_m') \Big)\cdot \nabla_{\theta}\psi(\theta^{(t)};\bfs_m,a_m)\\
\end{split}
\end{equation}
with $g^{(t)}(\theta^{(t)};\mathcal{D}_t)\in \mathbb{R}^n$.
Then, we have 
\begin{equation}\label{eqn: theta_update}
    \theta^{(t+1)} = \theta^{(t)} -\eta_t \cdot g^{(t)}(\theta^{(t)};\mathcal{D}_t).
\end{equation}
Similar to \eqref{eqn: gradient_theta}, we define the gradient 
\begin{equation}\label{eqn: gradient_w}
l^{(t)}(\bfw^{(t)};\mathcal{D}_t)= 
    \sum_{m\in\mathcal{D}_t}
    \Big( \phi(\bfs_m,a_m,\bfs^\prime_m)^\top\bfw^{(t)} -r(\bfs_m,a_m,\bfs_m^\prime) \Big)\cdot \phi(\bfs_m,a_m,\bfs_m^\prime).
\end{equation}
% \theta^{(t+1)}= \theta^{(t)} -\eta_t\cdot\sum_{m\in\mathcal{D}_t} \Big(\psi(\theta^{(t)};\bfs_m,a_m) -\phi(\bfs_m,a_m,\bfs_m^\prime) -\gamma \cdot \psi(\theta^{(t)};\bfs_m',a_m') \Big)\cdot \nabla_{\theta}\psi(\theta^{(t)};\bfs_m,a_m)\\
%             \bfw^{(t+1)}  
%     = \bfw^{(t)} -\lambda_t\cdot\sum_{m\in\mathcal{D}_t}
%     \Big( \phi(\bfs_m,a_m,\bfs^\prime_m)^\top\bfw^{(t)} -r(\bfs_m,a_m,\bfs_m^\prime) \Big)\cdot \phi(\bfs_m,a_m,\bfs_m^\prime)\\

In addition, without special descriptions, ${\boldsymbol{\alpha}}=[{{\boldsymbol{\alpha}}}_1^\top, {{\boldsymbol{\alpha}}}_2^\top, \cdots, {{\boldsymbol{\alpha}}}_K^\top ]^\top$  stands for any unit vector that in $\mathbb{R}^{K_\ell K_{\ell-1}}$ with ${{\boldsymbol{\alpha}}}_j\in \mathbb{R}^{K_{\ell-1}}$ ($K_0 = d$).  Therefore, we have 
\begin{equation}\label{eqn: alpha_definition}
\begin{gathered}
    \| \nabla_{\ell} H\|_2 = \max_{\boldsymbol{\alpha}}\|\boldsymbol{\alpha}^\top \nabla_{\ell} H  \|_2
    = 
    \max_{\boldsymbol{\alpha}}
    \Big|
    \sum_{j=1}^{K} \boldsymbol{\alpha}_j^\top\frac{\partial H }{\partial \bfw_{\ell,j}}
    \Big|,
   \\
    \| \nabla_{\ell}^2 H\|_2 = \max_{\boldsymbol{\alpha}}\|\boldsymbol{\alpha}^\top \nabla_{
    \ell}^2~H~\boldsymbol{\alpha} \|_2
    = 
    \max_{\boldsymbol{\alpha}}
    \Big(
    \sum_{j=1}^{K} \boldsymbol{\alpha}_j^\top\frac{\partial H }{\partial \bfw_{\ell,j}}
    \Big)^2.
\end{gathered}
\end{equation}

\textbf{Derivation of the gradient of deep neural networks.}
We use $h^{(\ell)}
(\theta)$ to denote the input in the $\ell$-th layer (or the output in the $(\ell-1)$-th layer) of deep neural network $\psi(\theta)$, and $h^{(1)}=\bfx(\bfs,a)$,
where
\begin{equation}\label{eqn: defi_h}
    \bfh^{(\ell)}(\theta;\bfs,a) = \sigma(\theta_{\ell-1}^\top \bfh^{(\ell-1)}) = \cdots = \sigma\Big(\theta_\ell^\top \sigma\big(\theta_{\ell-1} \cdots \sigma( \theta_1^\top\bfx(\bfs,a) )\big) \Big).
\end{equation}
Then, we denote the dimension of $\bfh^{(\ell)}$ as $K_\ell$.
Then, $\psi(\theta;\bfs,a)$ can be written as 
\begin{equation}\label{eqn: Q_h}
    \psi(\theta;\bfs,a) = \frac{\bfone^\top}{K_L}\sum_{k=1}^{K_L}\sigma(\theta_{L,k}^{\top}\bfh^{(L)}) 
    =\frac{\bfone^\top}{K_L}\sigma\big(\theta_{L}^{\top}\sigma(\theta_{L-1}^\top\bfh^{(L-1)})\big),
\end{equation}
where $\theta_{\ell, k}$ denotes the $k$-th neuron weights in the $\ell$-th layer.
Then, we define a group of functions $\cJ_\ell(\theta)\in\mathbb
{R}^{n} \longrightarrow\mathbb{R}^K$ such that  
\begin{equation}\label{eqn: defi_jc}
\begin{split}
    &\cJ_{\ell}(\theta)\\
    =&\begin{cases}
        \big[\bfone^\top \sigma^\prime(\theta_L^\top\bfh^{(L)})\theta_L^\top \cdot \sigma^\prime(\theta_{L-1}^\top \bfh^{(L-1)})\theta_{L-1}^\top\cdots\sigma^\prime(\theta_{\ell+1}^\top \bfh^{(\ell+1)})\theta_{\ell+1}^\top\big]^\top \quad  \text{if} \quad \ell >1\\
        \bfone \quad  \text{if}\quad  \ell =1
    \end{cases}.
\end{split} 
\end{equation}
Then, the gradient of $\psi$ can be represented as 
\begin{equation}\label{eqn: gradient_multi}
    \frac{\partial \psi}{\partial \theta_{\ell,k}}(\theta) = \frac{1}{K_\ell} \cJ_{\ell,k}(\theta)\sigma^\prime\big(\theta_{\ell,k}^\top \bfh^{(\ell)}(\theta)\big)\bfh^{(\ell)}(\theta),
\end{equation}
where $\cJ_{\ell,k}$ stands for the $k$-th component of $\cJ_\ell$.

\textbf{Order-wise Analysis.}
%  
% Moreover,  to avoid high-dimensional tensors (e.g., tensors in four-dimensional space) in the calculation of the second-order derivative of $Q(\bfW)$,
% we define
% the first-order derivative in the form of vectorized $\bfW$ as 
% \begin{equation}\label{eqn: vec_W}
%     \nabla Q(\bfW) = \Big[\frac{\partial Q}{\partial \bfw_1}^\top, \frac{\partial Q}{\partial \bfw_2}^\top, \cdots , \frac{\partial Q}{\partial \bfw_K}^\top \Big]^\top \in \mathbb{R}^{dK} 
% \end{equation}
% with $\bfW =[\bfw_1, \bfw_2, \cdots, \bfw_K]\in\mathbb{R}^{d\times K}$. 
% Therefore, the second order derivative of the empirical risk function is in $\mathbb{R}^{dk\times dk}$.
% Similar to \eqref{eqn: vec_W}, the high-order derivatives of the population risk functions are defined based on vectorized $\bfW$ as well. 
Most constant numbers will be ignored in most steps. In particular, we use $h_1(z) \gtrsim(\text{or }\lesssim, \eqsim ) h_2(z)$ to denote there exists some positive constant $C$ such that $h_1(z)\ge(\text{or } \le, = ) C\cdot h_2(z)$ when $z\in\mathbb{R}$ is sufficiently large. 
% In addition, Let $\Sigma_i(\ell)$ denote the $i$-th largest singular value of ${\theta}^\star_\ell$.
In this paper, we consider the case where $\theta^\star_\ell$ is well-conditioned, such that its largest singular value $\Sigma_{1}(\ell)$ and the condition number $\Sigma_{1}(\ell)/\sigma_{K}(\ell)$ can be viewed as constants and will be hidden in the order-wise analysis.

\renewcommand{\arraystretch}{1.5}
\begin{table}[H]
    \caption{Notations for the proofs} 
    \centering
    \begin{tabular}{c|p{11.5cm}}
    \hline
    \hline
    % $g_t(\bfW)$ & The gradient function at point $\bfW$ in the $t$-th outer loop, defined in \eqref{eqn: gradient}.\\
    % \hline
    % % \hline
    %  $g_t(\bfW_\ell;\bfW)$   & The gradient function of $g_t(\bfW)$ with respect to the components of $\bfW_\ell$.\\
    % \hline
    $d$ & Dimension of the feature mappings of the state-action pair $(\bfs,a)\in \mathcal{S}\times\mathcal{A}$.\\
    \hline
    $K$ & Number of neurons in the hidden layer.
    \\
    \hline
    $L$ & Number of hidden layers.
    \\
    \hline
    $T$ & Number of iterations.\\
    \hline
    $\bfw_i^{(t)}$ & The estimated value for reward mapping of task $i$ at $t$-th iteration.\\
    \hline
    $\Theta^{(t)}_i$ & The estimated neuron weights for the successor feature of task $i$ at $t$-th iteration.\\
    \hline
    $\theta^{(t)}$ & The value of $\Theta^{(t)}_1$ to simplify the notation in the analyses without GPI.\\
    \hline
    $g^{(t)}(\theta^{(t)};\mathcal{D}_t)$ & The pseudo-gradient function defined in \eqref{eqn: gradient_theta} at point $\theta^{(t)}$ with respect to the dataset $\mathcal{D}_t$.\\
    \hline
    $f_{\pi^\star}$ or $f$ & The population risk function defined in \eqref{eqn:prf}.\\
    \hline
    $\nabla_\ell H(\hat{\theta})$ & The gradient of a function $H$ with respect to the components of $\theta_\ell$ at point $\hat{\theta}$.\\
     \hline
     $\nabla^2_\ell H(\hat{\theta})$   & The Hessian matrix of a function $H$ with respect to the components of $\theta_\ell$ at point $\hat{\theta}$.\\
    \hline
    $Q_i^{\pi}$   & The Q-function of task $i$ for policy $\pi$.\\
    \hline
    $Q_i^{\star}$   & The Q-function of task $i$ for the optimal policy $\pi^\star$.\\
    \hline
    $q^{\star}$   & A constant defined in \eqref{eqn: q_t}, depending on task relevance $\|\bfw_i-\bfw_j\|_2$.\\
    \hline
    $\eta_t$ & The step size for updating neuron weights $\Theta_i$ for the successor feature. \\
    \hline
    $\kappa_t$ & The step size for updating the parameter for the weight mapping. \\
    \hline
    $c_N$ & A constant in the order of ${1}/{\sqrt{N}}$.\\
    \hline
     $n$  &  The dimension of $\theta$.\\
     \hline
    $n_\ell$  & The dimension of vectorized $\theta_\ell$.\\
    \hline
    $K_\ell$ & The dimension of the input for the $\ell$-th layer for the deep neural network. $K_0 =d$.\\
    \hline
    $\cJ_\ell(\bfW)$   & A function in $\mathbb{R}^n\longrightarrow \mathbb{R}^K$, defined in \eqref{eqn: defi_jc}.\\
    \hline
    $C_t$ & The distribution shift between the optimal policy and behavior policy at iteration $t$, defined in Assumption \eqref{ass3}.\\
    \hline
        $N$ & The size of the experience replay buffer.\\
        \hline
        $\phi_{\max}$ & The upper bound of the transition feature.\\
    \hline
    $\rho_1$ & A constant defined in \eqref{def: rho}.\\
    \hline
    $\rho_2$ & The smallest eigenvalue of $\mathbb{E} \phi(\bfs,a)\phi(\bfs,a)^\top\in\mathbb{R}^{d\times d}$.\\
    \hline
    $\phi_{\max}$ & The upper bound of the transition feature.\\
    \hline
    \hline
    \end{tabular}
    \label{table:appendix}
\end{table}

\subsection{Useful Lemmas for matrix concentration}

\begin{lemma}[Weyl's inequality, \cite{B97}] \label{Lemma: weyl}
Let $\bfB =\bfA + \bfE$ be a matrix with dimension $m\times m$. Let $\lambda_i(\bfB)$ and $\lambda_i(\bfA)$ be the $i$-th largest eigenvalues of $\bfB$ and $\bfA$, respectively.  Then, we have 
\begin{equation}
    |\lambda_i(\bfB) - \lambda_i(\bfA)| \le \|\bfE \|_2, \quad \forall \quad i\in [m].
\end{equation}
\end{lemma}

\begin{lemma}[\cite{T12}, Theorem 1.6]\label{prob}
	Consider a finite sequence $\{{\bfZ}_k\}$ of independent, random matrices with dimensions $d_1\times d_2$. Assume that such random matrix satisfies
	\begin{equation*}
	\vspace{-2mm}
	\mathbb{E}({\bfZ_k})=0\quad \textrm{and} \quad \left\|{\bfZ_k}\right\|\le R \quad \textrm{almost surely}.	
	\end{equation*}
	Define
	\begin{equation*}
	\vspace{-2mm}
	\delta^2:=\max\Big\{\Big\|\sum_{k}\mathbb{E}({\bfZ}_k{\bfZ}_k^*)\Big\|,\Big\|\displaystyle\sum_{k}\mathbb{E}({\bfZ}_k^*{\bfZ}_k)\Big\|\Big\}.
	\end{equation*}
	Then for all $t\ge0$, we have
	\begin{equation*}
	\text{Prob}\Bigg\{ \left\|\displaystyle\sum_{k}{\bfZ}_k\right\|\ge t \Bigg\}\le(d_1+d_2)\exp\Big(\frac{-t^2/2}{\delta^2+Rt/3}\Big).
	\end{equation*}
\end{lemma}

\begin{lemma}[Lemma 5.2, \cite{V2010}]\label{Lemma: covering_set}
	Let $\mathcal{B}(0, 1)\in\{ \boldsymbol{\alpha} \big| \|\boldsymbol{\alpha} \|_2=1, \boldsymbol{\alpha}\in \mathbb{R}^d  \}$ denote a unit ball in $\mathbb{R}^{d}$. Then, a subset $\mathcal{S}_\xi$ is called a $\xi$-net of $\mathcal{B}(0, 1)$ if every point $\bfz\in \mathcal{B}(0, 1)$ can be approximated to within $\xi$ by some point $\boldsymbol{\alpha}\in \mathcal{B}(0, 1)$, i.e., $\|\bfz -\boldsymbol{\alpha} \|_2\le \xi$. Then the minimal cardinality of a   $\xi$-net $\mathcal{S}_\xi$ satisfies
	\begin{equation}
	|\mathcal{S}_{\xi}|\le (1+2/\xi)^d.
	\end{equation}
\end{lemma}

\begin{lemma}[Lemma 5.3, \cite{V2010}]\label{Lemma: spectral_norm_on_net}
	Let $\bfA$ be an $d_1\times d_2$ matrix, and let $\mathcal{S}_{\xi}(d)$ be a $\xi$-net of $\mathcal{B}(0, 1)$ in $\mathbb{R}^d$ for some $\xi\in (0, 1)$. Then
	\begin{equation}
	\|\bfA\|_2 \le (1-\xi)^{-1}\max_{\boldsymbol{\alpha}_1\in \mathcal{S}_{\xi}(d_1), \boldsymbol{\alpha}_2\in \mathcal{S}_{\xi}(d_2)} |\boldsymbol{\alpha}_1^T\bfA\boldsymbol{\alpha}_2|.
	\end{equation} 
\end{lemma}

\begin{lemma}[Mean Value Theorem]\label{Lemma: MVT}
Let $\bfU \subset \mathbb{R}^{n_1}$ be open and $\bff : \bfU \longrightarrow \mathbb{R}^{n_2}$ be continuously differentiable, and $\bfx \in \bfU$, $\bfh \in \mathbb{R}^{n_1}$ vectors such that the line segment $\bfx + t\bfh$, $0 \le t \le 1$ remains in $\bfU$. Then we have:
\begin{equation*}
    \bff(\bfx+\bfh)-\bff(\bfx)=\left(\int_{0}^{1}\nabla\bff(\bfx+t\bfh)dt\right)\cdot \bfh,
\end{equation*}
where $\nabla\bff$ denotes the Jacobian matrix of $\bff$.
\end{lemma}

\subsection{Defintions of Sub-Gaussian and Sub-exponential.}
\begin{definition}[Definition 5.7, \cite{V2010}]\label{Def: sub-gaussian}
	A random variable $X$ is called a sub-Gaussian random variable if it satisfies 
	\begin{equation}
		(\mathbb{E}|X|^{p})^{1/p}\le c_1 \sqrt{p}
	\end{equation} for all $p\ge 1$ and some constant $c_1>0$. In addition, we have 
	\begin{equation}
		\mathbb{E}e^{s(X-\mathbb{E}X)}\le e^{c_2\|X \|_{\psi_2}^2 s^2}
	\end{equation} 
	for all $s\in \mathbb{R}$ and some constant $c_2>0$, where $\|X \|_{\psi_2}$ is the sub-Gaussian norm of $X$ defined as $\|X \|_{\psi_2}=\sup_{p\ge 1}p^{-1/2}(\mathbb{E}|X|^{p})^{1/p}$.

	Moreover, a random vector $\bfX\in \mathbb{R}^d$ belongs to the sub-Gaussian distribution if one-dimensional marginal $\boldsymbol{\alpha}^T\bfX$ is sub-Gaussian for any $\boldsymbol{\alpha}\in \mathbb{R}^d$, and the sub-Gaussian norm of $\bfX$ is defined as $\|\bfX \|_{\psi_2}= \sup_{\|\boldsymbol{\alpha} \|_2=1}\|\boldsymbol{\alpha}^T\bfX \|_{\psi_2}$.
\end{definition}

\begin{definition}[Definition 5.13, \cite{V2010}]\label{Def: sub-exponential}
	
	A random variable $X$ is called a sub-exponential random variable if it satisfies 
	\begin{equation}
	(\mathbb{E}|X|^{p})^{1/p}\le c_3 {p}
	\end{equation} for all $p\ge 1$ and some constant $c_3>0$. In addition, we have 
	\begin{equation}
	\mathbb{E}e^{s(X-\mathbb{E}X)}\le e^{c_4\|X \|_{\psi_1}^2 s^2}
	\end{equation} 
	for $s\le 1/\|X \|_{\psi_1}$ and some constant $c_4>0$, where $\|X \|_{\psi_1}$ is the sub-exponential norm of $X$ defined as $\|X \|_{\psi_1}=\sup_{p\ge 1}p^{-1}(\mathbb{E}|X|^{p})^{1/p}$.
\end{definition}

\section{Proof of Theorem \ref{Thm1}}\label{app:thm1}

\begin{lemma}[Local convexity of $f_{\pi^\star}$]\label{Lemma: second_order_derivative}
   Given any $\theta\in \mathbb{R}^{n}$, let $\theta$ satisfy
      \begin{equation}\label{eqn: p_initial}
        \| \theta -\theta^\star \|_2  \lesssim \frac{c_N\cdot \sigma_K}{\rho_1 \cdot K }
    \end{equation}
    for some constant $c_N\in(0,1)$.
    Then, for the $f_{\pi^\star}$ defined in \eqref{eqn:prf}, we have 
    \begin{equation}
        \frac{(1-c_N)\rho_1}{K^2} 
        \preceq \nabla^2_{\ell} {f}_{\pi^\star}(\theta) 
        \preceq \frac{7}{K}.
        % \frac{\lambda\rho_1(\delta)+\w[\lambda]\rho_1(\de)}{{12}\kappa^2\gamma K^2} \preceq \nabla^2 {f}(\theta;p) \preceq \frac{7(\lambda\delta^2+\w[\lambda]\de^2)}{K}.
    \end{equation}
\end{lemma}

\begin{lemma}[Upper bound of the error gradient]\label{Lemma: first_order_derivative}
 Let $f_{\pi^\star}$ be the function defined in \eqref{eqn:prf}. Let $g_{t}$ be the function defined in \eqref{eqn: gradient_theta}. Then, with probability at least $1-q^{-K_{\ell-1}}$, we have
   \begin{equation}
        \begin{split}
        \left\|\nabla_\ell f_{\pi^\star}(\theta) - g_{\ell}(\theta^{(t)};\mathcal{D}_t) \right\|_2
       \lesssim& \frac{1}{K_\ell} \cdot \|\theta -\theta^\star \|_2 \cdot \sqrt{\frac{K_{\ell-1}\log q}{|\mathcal{D}_t|}} + \frac{\gamma}{K_\ell}\cdot \| \theta^{(t)} -\theta^\star\|_2\\
        &+ \frac{R_{\max}}{1-\gamma}\cdot (1+\gamma)\tau^\star\cdot \eta_{t-\tau^\star}\\
        &+ |\mathcal{A}|\cdot \frac{R_{\max}}{1-\gamma}\cdot (1+\log_\nu \lambda^{-1} + \frac{1}{1-\nu})\cdot C_t,
        \end{split}
        \end{equation}
    where $\tau^\star = \min\{t\mid \lambda\nu^t\le \eta_T \}$, and $\nu$ \& $\lambda$ are defined in Assumption \ref{ass2}.
\end{lemma}

\begin{lemma}[Convergence of $\bfw^{(t)}$]\label{lemma:convergence_of_w}
    With probability at least $1-q^{-d}$, $\bfw$ enjoys a linear convergence rate to $\bfw^\star$ as 
    \begin{equation}
        \begin{split}
            \|\bfw^{(t+1)} -\bfw^\star \|_2 
            \le& \Big(1-\frac{\rho-c_N}{\phi_{\max}}\Big)\cdot \|\bfw^{(t)}-\bfw^\star\|_2.\\
        \end{split}
    \end{equation}
\end{lemma}

\begin{proof}[Proof of Theorem  \ref{Thm1}]
    From Algorithm \ref{Alg}, the update of $\theta$  can be written as 
    \begin{equation}\label{eqn: Thm3_temp1}
    \begin{split}
       \W[t+1] 
       = &  \W[t] -\eta_t\cdot  g^{(t)}(\W[t];\mathcal{D}_t)\\
       = & \W[t] -\eta_t\cdot \nabla{f}(\W[t])
       + \eta_t \cdot \big( \nabla f(\W[t]) - g^{(t)}(\W[t];\mathcal{D}_t) \big).
    \end{split}
    \end{equation}
    Since $\nabla f$ is a smooth function and $\theta^*$ is a local (global) optimal to $f$, then we have 
    \begin{equation}\label{eqn: Thm3_temp2}
        \begin{split}
            \nabla{f}(\W[t])
            =& \nabla{f}(\W[t]) - \nabla{f}(\theta^\star)\\
            =&\int_{0}^1 \nabla^2 f\Big(\W[t] + u \cdot(\W[t]-\theta^\star) \Big)du \cdot(\W[t]-\theta^\star),
        \end{split}
    \end{equation}
    where the last equality comes from Mean Value Theory in Lemma \ref{Lemma: MVT}. 
   For notational convenience, we use $\bfA^{(t)}$ to denote the integration as
    \begin{equation}
        \bfA^{(t)} : = \int_{0}^1 \nabla^2 f\Big(\theta^{(t)} + u \cdot(\theta^{(t)}-\theta^\star) \Big)du.    
    \end{equation}
    
    Then, we have 
    \begin{equation}\label{thm1:key1}
    \begin{split}
        \|\W[t+1] - \theta^\star\|_2 \le& \|\bfI - \eta_t \bfA^{(t)}\|_2 \cdot 
        \| \theta^{(t)} -\theta^\star \|_2 +\eta_t \cdot \|\nabla f(\W[t]) - g^{(t)}(\W[t];\mathcal{D}_t) \|_2\\
        \le & \|\bfI - \eta_t \bfA^{(t)}\|_2 \cdot 
        \| \theta^{(t)} -\theta^\star \|_2 +\eta_t \cdot \sum_{\ell=1}^L\left\|\nabla_\ell f(\W[t]) - g^{(t)}(\theta^{(t)}_\ell;\mathcal{D}_t) \right\|_2.
    \end{split}
    \end{equation}
    From Lemma \ref{Lemma: second_order_derivative},  we have 
    \begin{equation}
        \|\bfI - \eta_t\bfA^{(t)}\|_2 \le  1-\eta_t\cdot \frac{(1-c_N)\cdot \rho_1}{K^2}.
    \end{equation}
    From Lemma \ref{Lemma: first_order_derivative}, we have 
     \begin{equation}
        \begin{split}
        \left\|\nabla_\ell f_{\pi^\star}(\theta^{(t)}) - g_{\ell}(\theta^{(t)};\mathcal{D}_t) \right\|_2
       \lesssim& \frac{1}{K_\ell} \cdot \|\theta^{(t)} -\theta^\star \|_2 \cdot \sqrt{\frac{K_{\ell-1} \log q}{|\mathcal{D}_t|}} + \frac{\gamma}{K_\ell}\cdot \| \theta^{(t)} -\theta^\star\|_2\\
        &+ \frac{R_{\max}}{1-\gamma}\cdot (1+\gamma)\tau^\star\cdot \eta_{t-\tau^\star}\\
        &+ |\mathcal{A}|\cdot \frac{R_{\max}}{1-\gamma}\cdot (1+\log_\nu \lambda^{-1} + \frac{1}{1-\nu})\cdot C_t.
        \end{split}
        \end{equation}
    % \begin{equation}
    %     \|\nabla f(\W[t]) - g^{(t)}(\W[t];\mathcal{D}_t) \|_2 \le \big( C^\star + \frac{\gamma\cdot \rho_1}{K^2}\big) \cdot \|\W[t] - \theta^\star \|_2 + \|\Delta_t\|_2.
    % \end{equation}
    With Assumption \ref{ass3}, we have 
    $$C_t \le C\cdot\big(\|\theta^{(t)}-\theta^\star\|_2 + \|\bfw^{(t)}-\bfw^\star\|_2\big).$$

    When we have a sufficiently large number of samples at iteration $t$ as 
    \begin{equation}
    \begin{split}
        |\mathcal{D}_t| \gtrsim& ~~c_N^{-2}\cdot \rho_1^{-1} \cdot\big(\sum_{\ell=1}^L K_{\ell} \sqrt{K_{\ell-1}}\big)^2 \cdot \log q,
    \end{split}
    \end{equation}
     \eqref{thm1:key1} can be simplified as 
    \begin{equation}\label{eqn: thm1_thm2}
        \|\W[t+1] - \theta^\star\|_2 \le (1-\eta_t \cdot \xi)\cdot \|\W[t]-\theta^\star\|_2 + \eta_t \cdot \Delta_t + \eta_t\cdot C^\star \|\bfw^{(t)}-\bfw^\star\|_2.
    \end{equation}
    where 
    \begin{equation}\label{defi:delta_t}
        \begin{gathered}
        C^\star = |\mathcal{A}|\cdot \frac{R_{\max}}{1-\gamma} \cdot  (1+\log_\nu \lambda^{-1}+\frac{1}{1-\nu})\cdot C\\
        \xi =~\frac{(1-\gamma-c_N)\rho_1}{K^2} - C^\star\\
        \Delta_t = \frac{R_{\max}}{1-\gamma}\cdot (1+\gamma)\tau^\star\cdot \eta_{t-\tau^\star}.
        \end{gathered}
    \end{equation}
    Let $\eta_t = \frac{1}{\xi \cdot (t+1)}$, we have 
    \begin{equation}
        (t+1)\cdot \|\theta^{(t+1)}-\theta^\star\|_2 \le t \cdot \|\W[t]-\theta^\star\|_2 + \xi^{-1}\cdot \Delta_t + \xi^{-1}\cdot C^\star \|\bfw^{(t)}-\bfw^\star\|_2.
    \end{equation}
    Next, we have 
    \begin{equation}\label{eqn: temp1.1}
    \begin{split}
        &\sum_{t=0}^{T-1}(t+1)\cdot \|\theta^{(t+1)}-\theta^\star\|_2 - t \cdot \|\W[t]-\theta^\star\|_2 \\
        \le & \sum_{t=0}^{T-1} \xi^{-1}\cdot (\Delta_t +  C^\star \|\bfw^{(t)}-\bfw^\star\|_2).\\
    \end{split}
    \end{equation}
    With the definition of $\Delta_t$ in \eqref{defi:delta_t}, we have
    \begin{equation}
        \begin{split}
        \sum_{t=0}^{T-1}  \Delta_t 
        \le & \sum_{t=0}^{\tau^\star}  \Delta_t + \sum_{t=\tau^\star}^{T-1} \lambda^{-1}\cdot \Delta_t\\
        \le &\sum_{t=0}^{\tau^\star} \tau^\star \cdot \frac{R_{\max}}{1-\gamma} + \sum_{t=\tau^\star}^{T-1}\cdot \frac{R_{\max} \cdot (1+\gamma)}{1-\gamma} \cdot \tau^\star \cdot \frac{1}{T-\tau^\star +1 }\\
        \lesssim& \frac{R_{\max}\cdot \log^2 T}{1-\gamma} + \frac{R_{\max} \cdot (1+\gamma)\cdot \log^2 T}{1-\gamma}.
        \end{split}
    \end{equation}
    With Lemma \ref{lemma:convergence_of_w} that $\bfw$ enjoys a geometric decay, we have 
    \begin{equation}
        \begin{split}
            \sum_{t=0}^{T-1}  \| \bfw^{(t)}-\bfw^\star\|_2 \lesssim  \|\bfw^{(0)} -\bfw^\star\|_2.
        \end{split}
    \end{equation}
    By multiplying $1/T$ on both sides of \eqref{eqn: temp1.1}, we have 
    \begin{equation}
    \begin{split}
        \|\theta^{(T)} -\theta^\star\|_2 
        \le \frac{(2+\gamma)\cdot R_{\max}\cdot \log^2 T + C^\star \|\bfw^{(0)}-\bfw^\star\|_2}{{(1-\gamma-c_N)\rho_1 }{K^{-2}}-  C^\star} \cdot  \frac{1}{T}.
    \end{split}
    \end{equation}

\end{proof}

\section{Proofs of Theorems \ref{Thm3} and \ref{Thm4}}\label{app: thm34}
\begin{lemma}\label{lemma: DQN_difference}
We have
    \begin{equation}
        \big|Q_i^{\pi_i^\star}(\bfs,a) - Q_i^{\pi_j^\star}(\bfs,a)\big| \le \frac{2\gamma}{1-\gamma} \cdot \max_{\bfs,a} |r_i(\bfs,a)- r_j(\bfs,a)|.
    \end{equation}
\end{lemma}
\begin{proof}[Proof of Theorem \ref{Thm3}]
We have
    \begin{equation}\label{eqn: temp10_9}
        \begin{split}
            Q^\star_{n+1}(\bfs,a) -Q_{n+1}^{\pi_{n+1}}(\bfs,a)
            =& ~\big(Q^\star_{n+1}(\bfs,a) -\max_j Q_{n+1}^{\pi_j^\star}(\bfs,a)\big)
            +\big( \max_j Q_{n+1}^{\pi_j^\star}(\bfs,a)-Q_{n+1}^{\pi_{n+1}}(\bfs,a)\big)\\
            \le & \big(Q^\star_{n+1}(\bfs,a) - Q_{n+1}^{\pi_j^\star}(\bfs,a)\big)
            +\big( \max_j Q_{n+1}^{\pi_j^\star}(\bfs,a)-Q_{n+1}^{\pi_{n+1}}(\bfs,a)\big)\\
        \end{split}
    \end{equation}

For any task $j\in[n]$, we have 
\begin{equation}\label{thm3:key}
    \begin{split}
         Q^{\pi_j^\star}_{n+1}(\bfs,a) - Q^{\pi_j}_{n+1}(\bfs,a)
        =~ \big(\psi_j(\Theta_j^\star) - \psi_j(\Theta_j^{(T)})\big)\cdot \bfw_{n+1}^\star.
    \end{split}
\end{equation}
According to Theorem \ref{Thm1}, we have
\begin{equation}
    \begin{split}
            \|\psi_j(\Theta_j^\star) - \psi_j(\Theta_j^{(T)})\|_2 
        \le  & \frac{(2+\gamma)\cdot R_{\max}\cdot \log^2 T + C^\star \|\bfw_j^{(0)}-\bfw_j^\star\|_2}{{(1-\gamma-c_N)\rho_1 }{K^{-2}}-  C^\star} \cdot  \frac{1}{T} := \frac{C_3}{T}.
    \end{split}
\end{equation}
Hence, we have 
\begin{equation}
    Q^{\pi_j^\star}_{n+1}(\bfs,a) - Q^{\pi_j}_{n+1}(\bfs,a)\le \frac{C_3 \|\bfw_{n+1}^\star\|}{T}.
\end{equation}

Then, we have 
\begin{equation}
    \begin{split}
        \mathcal{T}^{\pi_{n+1}} \Big[\max_j
    Q_{n+1}^{\pi_{j}}(s,a)\Big] 
        =& r_{n+1}(s,a)+\gamma \mathbb{E}_{s'\sim (s,a)} \max_j
    Q_{n+1}^{\pi_{j}}(s',\pi_{n+1}(s'))\\
     =& r_{n+1}(s,a)+\gamma \mathbb{E}_{s'\sim (s,a)} \max_{a'} \big(\max_{j}
    Q_{n+1}^{\pi_{j}}(s',a')\big)\\
    \ge& r_{n+1}(s,a)+\gamma \mathbb{E}_{s'\sim (s,a)} \max_{a'} \big(\max_{j}
    Q_{n+1}^{\pi_{j}^\star}(s',a')\big) -\gamma \frac{C_3 \|\bfw_{n+1}^\star\|}{T}\\
    \ge & r_{n+1}(s,a)+\gamma \mathbb{E}_{s'\sim (s,a)} \max_{a'} 
    Q_{n+1}^{\pi_{j}^\star}(s',a') -\gamma \frac{C_3 \|\bfw_{n+1}^\star\|}{T}\\
    \ge & r_{n+1}(s,a)+\gamma \mathbb{E}_{s'\sim (s,a)}  
    Q_{n+1}^{\pi_{j}^\star}(s',\pi_j^\star(s')) -\gamma \frac{C_3 \|\bfw_{n+1}^\star\|}{T}\\
    =& \mathcal{T}^{\pi_j^\star} Q_{n+1}^{\pi_j^\star}(s,a)  -\gamma \frac{C_3 \|\bfw_{n+1}^\star\|}{T}\\
    =&Q_{n+1}^{\pi_j^\star}(s,a)  -\gamma \frac{C_3 \|\bfw_{n+1}^\star\|}{T}\\
    \end{split}
\end{equation}

Since the previous inequality holds for all $j$, we have 
\begin{equation}
    \begin{split}
        \mathcal{T}^\pi\max_j Q_{n+1}^{\pi_j}(\bfs,a) 
        \ge \max_j Q_{n+1}^{\pi_j^\star}(\bfs,a) -\gamma\cdot \frac{C_3\|\bfw_{n+1}^\star\|_2}{T}.
    \end{split}
\end{equation}
Therefore, with the contraction property of the Bellman operator $\mathcal{T}^\pi$, we have 
\begin{equation}\label{eqn: temp10_2}
\begin{split}
    Q^{\pi_{n+1}}_{n+1}(\bfs,a) 
    = & ~ \lim_{k\rightarrow \infty}(\mathcal{T}^{\pi_{n+1}})^k \max_j Q_{n+1}^{\pi_j}(\bfs,a)\\
    \ge & ~ \lim_{k\rightarrow \infty}(\mathcal{T}^{\pi_{n+1}})^{k-1} \Big(\max_j Q_{n+1}^{\pi_j^\star}(\bfs,a) -\gamma \frac{C_3\|\bfw_{n+1}^\star\|_2}{T}\Big)\\
    \ge & \lim_{k\rightarrow \infty}(\mathcal{T}^{\pi_{n+1}})^{k-2} \cdot \mathcal{T}^{\pi_{n+1}}\Big(\max_j Q_{n+1}^{\pi_j}(\bfs,a) -(1+\gamma) \frac{C_3\|\bfw_{n+1}^\star\|_2}{T}\Big)\\
    \ge & \lim_{k\rightarrow \infty}(\mathcal{T}^{\pi_{n+1}})^{k-2} \cdot\Big(  \max_j Q_{n+1}^{\pi_j^\star}(\bfs,a) - \gamma(1+\gamma)\frac{C_3\|\bfw_{n+1}^\star\|_2}{T} -\gamma \frac{C_3\|\bfw_{n+1}^\star\|_2}{T}\Big)\\
    \ge & \lim_{k\rightarrow \infty}(\mathcal{T}^{\pi_{n+1}})^{k-2}  \Big( \max_j Q_{n+1}^{\pi_j}(\bfs,a) -\gamma(1+\gamma)\frac{C_3\|\bfw_{n+1}^\star\|_2}{T} - (1+\gamma)\frac{C_3\|\bfw_{n+1}^\star\|_2}{T}\Big)\\
    \ge & \max_j  Q_{n+1}^{\pi_j}(\bfs,a) - \sum_{k=1}^{\infty} \gamma^{k-1}(1+\gamma)\frac{C_3\|\bfw_{n+1}^\star\|_2}{T}\\
    = & \max_j Q_{n+1}^{\pi_j}(\bfs,a) - \frac{1+\gamma}{1-\gamma}\frac{C_3\|\bfw_{n+1}^\star\|_2}{T}\\
    \ge &\max_j  Q_{n+1}^{\pi_j^\star}(\bfs,a) - \frac{C_3\|\bfw_{n+1}^\star\|_2}{T} - \frac{1+\gamma}{1-\gamma}\frac{C_3\|\bfw_{n+1}^\star\|_2}{T}\\
    = & \max_j  Q_{n+1}^{\pi_j^\star}(\bfs,a) - \frac{1}{1-\gamma}\frac{C_3\|\bfw_{n+1}^\star\|_2}{T}.
\end{split}
\end{equation}

    From Lemma \ref{lemma: DQN_difference}, we have
    \begin{equation}
        Q^\star_{n+1}(\bfs,a) -Q_{n+1}^{\pi_j^\star}(\bfs,a) \le \frac{2\gamma}{1-\gamma}\cdot \max_{\bfs,a}|r_{n+1}(\bfs,a) - r_j(\bfs,a)|,
    \end{equation}
    and \eqref{eqn: temp10_2} suggests that 
    \begin{equation}
        \max_j Q_{n+1}^{\pi_j^\star}(\bfs,a)-Q_{n+1}^{\pi_{n+1}}(\bfs,a) \le \frac{{C_3\|\bfw_{n+1}^\star\|_2}}{(1-\gamma)T}.
    \end{equation}
    Therefore, \eqref{eqn: temp10_9} can be upper bounded as 
        \begin{equation}\label{eqn: temp10_1}
        \begin{split}
            Q^\star_{n+1}(\bfs,a) -Q_{n+1}^{\pi_{n+1}}(\bfs,a)
            \le & \big(Q^\star_{n+1}(\bfs,a) - Q_{n+1}^{\pi_j^\star}(\bfs,a)\big)
            +\big( \max_j Q_{n+1}^{\pi_j^\star}(\bfs,a)-Q_{n+1}^{\pi_{n+1}}(\bfs,a)\big)\\
            \le &~\frac{2\gamma}{1-\gamma}\cdot \max_{\bfs,a}|r_{n+1}(\bfs,a) - r_j(\bfs,a)| + \frac{{C_3\|\bfw_{n+1}^\star\|_2}}{(1-\gamma)T}\\
            \le &~\frac{2\gamma\cdot \phi_{\max}}{1-\gamma} \| \bfw_{n+1}-\bfw_j \|_2 + \frac{{C_3\|\bfw_{n+1}^\star\|_2}}{(1-\gamma)T}.
        \end{split}
    \end{equation}
    
    Since \eqref{eqn: temp10_1} holds for any $j$, we have 
    \begin{equation}\label{eqn: temp10_3}
        |Q^\star_{n+1}(\bfs,a) -Q_{n+1}^{\pi_{n+1}}(\bfs,a)| \le 
        \frac{2\gamma\cdot \phi_{\max}}{1-\gamma} \min_{j\in[n]}\| \bfw_{n+1}-\bfw_j \|_2 + \frac{{C_3\|\bfw_{n+1}^\star\|_2}}{(1-\gamma)T}.
    \end{equation}
\end{proof}

\begin{proof}[Proof of Theorem \ref{Thm4}]
    Let $\pi^\prime_{n+1}$ be generalized policy with DQN via GPI, 
    we have 
    \begin{equation}
        \begin{split}
            Q^\star_{n+1}(\bfs,a) -Q_{n+1}^{\pi_{n+1}^\prime}(\bfs,a)
            =& ~\big(Q^\star_{n+1}(\bfs,a) -\max_j Q_{n+1}^{\pi_j^\star}(\bfs,a)\big)
            +\big( \max_j Q_{n+1}^{\pi_j^\star}(\bfs,a)-Q_{n+1}^{\pi_{n+1}^\prime}(\bfs,a)\big)\\
            \le & \big(Q^\star_{n+1}(\bfs,a) - Q_{n+1}^{\pi_j^\star}(\bfs,a)\big)
            +\big( \max_j Q_{n+1}^{\pi_j^\star}(\bfs,a)-Q_{n+1}^{\pi_{n+1}^\prime}(\bfs,a)\big)\\
        \end{split}
    \end{equation}
    Similar to \eqref{thm3:key}, we have 
    \begin{equation}
    \begin{split}
        &~ Q^{\pi_j^\star}_{n+1}(\bfs,a) - Q^{\pi_j^\prime}_{n+1}(\bfs,a)\\
        =&~ \psi_j(\Theta_j^\star)\bfw_{n+1}^\star - \psi_j(\Theta_j^{(T)}) \bfw_{j}^{(t)}\\
        \approx&~ \psi_j(\Theta_j^\star)\bfw_{n+1}^\star - \psi_j(\Theta_j^{(T)}) \bfw_{j}^\star\\
        =&~ \psi_j(\Theta_j^\star)\bfw_{n+1}^\star - \psi_j(\Theta_j^{(T)})\bfw_{n+1}^\star + \psi_j(\Theta_j^{(T)})\bfw_{n+1}^\star- \psi_j(\Theta_j^{(T)}) \bfw_{j}^\star\\
        \ge & - \| \Theta_j^\star -\Theta_j^{(T)}\|\cdot \|\bfw_{n+1}^\star \|_2 - \frac{1}{1-\gamma}\phi_{\max} \cdot \|\bfw_{n+1}^\star - \bfw_j^\star\|_2.
    \end{split}
\end{equation}
Following similar steps in \eqref{eqn: temp10_2}, we have  
    \begin{equation}
        \max_j Q_{n+1}^{\pi_j^\star}(\bfs,a)-Q_{n+1}^{\pi_{n+1}^\prime}(\bfs,a) \le \frac{{C_3\|\bfw_{n+1}^\star\|_2}}{(1-\gamma)T} + \phi_{\max} \min_{j\in [n]}\|\bfw^\star_{n+1}-\bfw_j^\star\|_2.
    \end{equation}
    Therefore, we have 
\begin{equation}
\begin{split}
     |Q^\star_{n+1}(\bfs,a) -Q_{n+1}^{\pi_{n+1}^\prime}(\bfs,a)|\
     \le& 
        \frac{2\gamma\cdot \phi_{\max}}{1-\gamma} \min_{j\in[n]}\| \bfw_{n+1}-\bfw_j \|_2 + \frac{{C_3\|\bfw_{n+1}^\star\|_2}}{(1-\gamma)T}\\
        &+ \phi_{\max} \cdot \min_{j\in[n]}\|\bfw_{n+1}^\star - \bfw_j^\star\|_2\\
        \le & \frac{2\cdot \phi_{\max}}{1-\gamma} \min_{j\in[n]}\| \bfw_{n+1}-\bfw_j \|_2 + \frac{{C_3\|\bfw_{n+1}^\star\|_2}}{(1-\gamma)T}.\\
\end{split}    
    \end{equation}
\end{proof}

\section{Proof of Theorem \ref{Thm2}}\label{app:thm2}

\begin{proof}[Proof of Theorem \ref{Thm2}]

For task $i$, let $\pi_j$ be the policy derived from $\psi_j(\Theta_j^{(T)})\bfw_i^\star$ with $1\le j\le i$, where $\Theta_j^{(T)}$ is the returned neuron weights for the successor feature of task $j$.

Similar to \eqref{eqn: temp10_3}, we have
\begin{equation}
    Q^\star_{i}(\bfs,a) -Q_{i}^{\pi_{j}}(\bfs,a) \le 
        \frac{2\gamma\cdot \phi_{\max}}{1-\gamma} \| \bfw_{j}-\bfw_i \|_2 + \frac{{C_3\|\bfw_{i}^\star\|_2}}{(1-\gamma)T}.
\end{equation}
    Let $\pi^{\prime}$ be the policy derived from $\psi_i(\Theta_i^{(t)})\bfw_i^\star$ at iteration $t$ for task $i$, we have 
    \begin{equation}
        Q_i^\star(\bfs,a ) - Q_i^{\pi^\prime} \le  \|\Theta^{(t)}_i-\Theta_i^\star\|_2
    \cdot\|\bfw_i^\star\|_2.
    \end{equation}
    Therefore, at iteration $t$ for task $i$, we have 
    \begin{equation}\label{eqn: thm2_ct}
        \begin{split}
            C_t = 
            &|Q^\star_{i}(\bfs,a) - Q_i^{\pi_i^{(t)}}|\\
             \le& \min\Big\{ \frac{2\gamma\cdot \phi_{\max}}{1-\gamma} \min_{1\le j \le i}\| \bfw_{j}-\bfw_i \|_2 + \frac{{C_3\|\bfw_{i}^\star\|_2}}{(1-\gamma)T}, \|\Theta^{(t)}_i-\Theta_i^\star\|_2
    \cdot\|\bfw_i^\star\|_2 \Big\}\\
    \lesssim & \min\Big\{ \frac{2\gamma\cdot \phi_{\max}}{1-\gamma} \min_{1\le j \le i}\| \bfw_{j}-\bfw_i \|_2, \|\Theta^{(t)}_i-\Theta_i^\star\|_2  
    \cdot\|\bfw_i^\star\|_2 \Big\} ~ (\text{As $T$ is sufficiently large})\\
    = & \min\{q_t,1\}\cdot \|\Theta^{(t)}_i-\Theta_i^\star\|_2, 
        \end{split}
    \end{equation}
    where 
    \begin{equation}\label{eqn: q_t}
        q_t = \frac{2\gamma \cdot R_{\max}}{1-\gamma}\cdot \frac{ \min_{1\le i\le j-1}~\|\bfw_i^\star - \bfw_j^\star\|_2  }{\|\Theta_j^{(t)}-\Theta_j^\star\|_2}.
    \end{equation}
    Following similar steps in \eqref{eqn: thm1_thm2} in the proof of Theorem \ref{Thm1}, with $C_t$ satisfying \eqref{eqn: thm2_ct}, we have 
     \begin{equation}
    \begin{split}
        \|\theta^{(T)} -\theta^\star\|_2 
        \le \frac{1}{T}\sum_{t=1}^{T-1}\frac{(2+\gamma)\cdot R_{\max}\cdot \log^2 T + C^\star \|\bfw^{(0)}-\bfw^\star\|_2}{{(1-\gamma-c_N)\rho_1 }{K^{-2}}-  \min\{1,q_t\}\cdot C^\star} \cdot  \frac{1}{T}.
    \end{split}
    \end{equation}
\end{proof}

\section{Additional numerical experiments}\label{app:experiment}
In this section we empirically validate the theoretical results obtained in the previous section, using synthetic and real-world RL benchmarks.

\subsection{Synthetic data settings}

{Here, we define an MDP that contains two tasks with shared state transition dynamics. The MDP consists of a state space with $|\mathcal{S}|=10,000$, an action space with $|\mathcal{A}|=4$.
For the first task, its successor feature is parameterized by a deep neural network with the randomly generated neuron weights $\Theta^\star_1$, and $\bfw^\star_1$ are randomly generated as the corresponding reward mapping. We then generate $\phi$ based on \eqref{eqn: sfdqn} with $\psi(\Theta_1^\star)$.
Since $\phi$ is shared across all tasks, for Task 2, we randomly generate the reward mapping $\bfw_2^\star$ and then calculate $\psi_2^\star$ accordingly.}

\subsection{Additional experiments on synthetic RL benchmarks}

\paragraph{Comparison for transfer from multiple source tasks.} In addition to the single source task case discussed in Section \ref{sec:experiments}, we also investigate the transfer performance of SFDQN (with and without GPI) and DQN (GPI) agents when trained on multiple source tasks. For this purpose, we generate $\phi$ as described in the previous section, and generate additional source tasks and a target task by pertubing $\bfw^\star_1$. Thus, we obtain $\bfw^\star_2$, $\bfw^\star_2$, $\bfw^\star_2$, the reward vectors for three additional source tasks. The norm of all weight vectors is set to 1 to make sure the reward scales are similar across multiple source tasks. Then we train each learning agent on the four source tasks and apply transfer using GPI. Note that while we also test the case for SFDQN without GPI (thus no transfer), this agent leverages the similarity of source tasks and the target task. The results are shown in Figure \ref{fig:reacher-4-source}. It can be seen that the SFDQN agent performs the best, which can leverage the task closeness due to the proximity of source task weight vectors to that of the target task, and the transition dynamics information via GPI. SFDQN agent without GPI on the other hand can only leverage the task closeness due to the proximity of source weight vectors. DQN-GPI agent cannot leverage the task closeness information or the transition dynamic information explicitly as the SFDN agent, and hence performs worse.

\paragraph{Effect of $\|w^\star_1 - w^\star_2\|$ on knowledge transfer.} We investigate the effect of the distance between $\bfw^\star_1$ to $\bfw^\star_2$, on the transfer performance of the SFDQN. For this purpose, we assume SF-DQN agents have access to optimal reward mappings when training on Tasks 1 and 2. 
% We initialize $\Theta^\star_1$ and $\bfw^\star_1$, and obtain $\phi$ to minimize Bellman error induced by $\Theta^\star_1$ and $\bfw^\star_1$. 
After obtaining $\phi$ as described earlier, we initialize and train $\Theta_2$ using $\phi$ and $\bfw^\star_2$, with GPI.  Reward defined by $\phi \cdot w^\star_2$ is used to obtain the average reward for Task 2. We repeat the process for different choices of $\bfw^\star_2$, and the results are shown in Figure \ref{fig:w-star-gap}. It can be seen that, when the task similarity is low (i.e. $\|w^\star_1 - w^\star_2\|$ is large), the performance of the SF-DQN agent with GPI is poor. On the other hand, when the task similarity is high, the performance becomes significantly better.

\begin{figure}[ht]
  % \hspace{-0.15cm}
  \centering
\begin{subfigure}[b]{0.48\textwidth}
 \centering
 \includegraphics[width=0.99\textwidth]{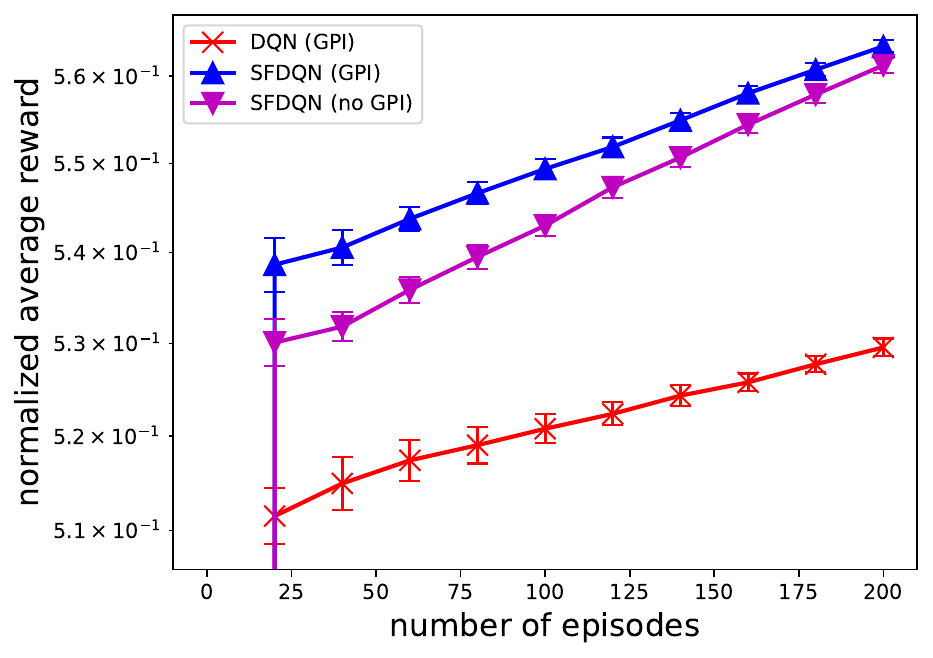}
 \caption{Comparison of four different source task transfer performance for SF-DQN (with and without GPI) and DQN with GPI.}
 \label{fig:reacher-4-source}
\end{subfigure}
\hspace{0.2cm}
  \begin{subfigure}[b]{0.48\textwidth}
 \centering
 \includegraphics[width=0.99\textwidth]{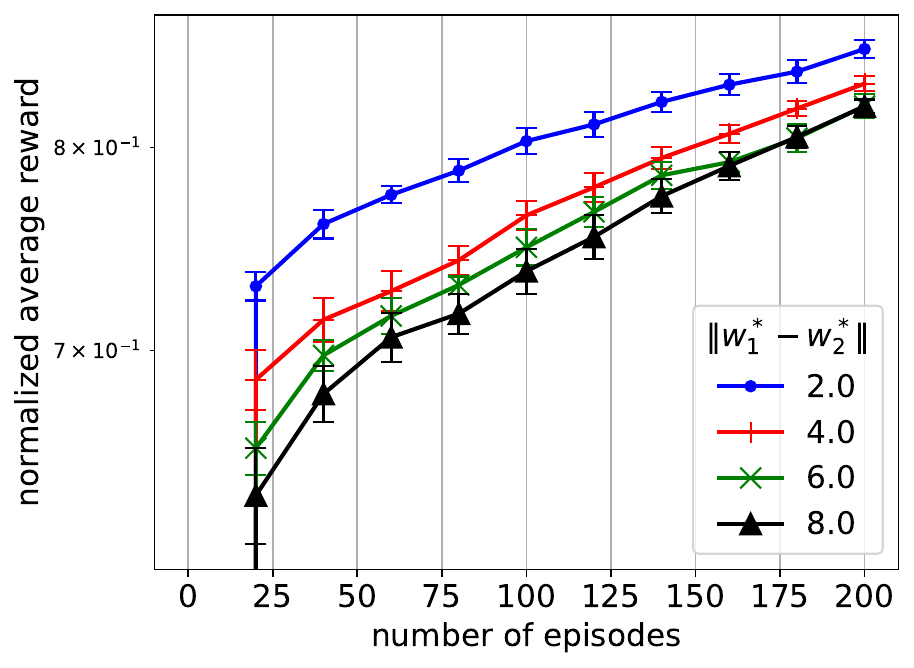}
 \caption{Effect  $\|w^\star_1 - w^\star_2 \|$ on the convergence of SF-DQN agent when training on task 2 with GPI}\label{fig:w-star-gap}
\end{subfigure}
\caption{Additional experiments on synthetic environment }
\label{fig:add-exp-synthetic}
% \vspace{-0.4cm}
\end{figure} 

\subsection{Real Data: Reacher environment}

The reacher environment is a robotic arm manipulation task consisting of a robotic arm with two joint torque controls. The state space is continuous, and the state features consist of angular displacement and angular velocity of the two joints. The actual action space for the robot arm is continuous consists of the torques applied to the two joints, and is discretized for 3 values (for each joint torque). Thus, the total discretized action space consists of 9 actions ($|\mathcal{A}|=9$). The discount factor used is $\gamma=0.9$. Multiple tasks in this environment are defined by goal locations, and the objective of each task is to move the tip of the robotic arm towards the goal location. 

The reward of each task is defined by the distance $\delta$, measured from the tip of the robotic arm to the corresponding goal location. Specifically, a reward of $1-\delta$ is given to the agent at each time step. There are 12 predefined tasks and $\phi$ for a given state (common to all 12 tasks) is defined by stacking the reward for each of the 12 tasks for a given state as a vector. The corresponding reward weights $\bfw^\star_i$ for $i=1, \dots, 12$ are defined by one hot vectors, where the $i^{th}$ element of $\bfw^\star_i$ is 1 and other elements are 0. Thus, the inner product $\phi^\top \bfw^\star_i$ naturally recovers the reward for the $i^{th}$ task. For running experiments with this task, we use the open source code base \url{https://github.com/mike-gimelfarb/deep-successor-features-for-transfer.git}.
\begin{figure}[ht]
  % \hspace{-0.15cm}
  \centering
  \begin{subfigure}[b]{0.48\textwidth}
 \centering
 \includegraphics[width=0.99\textwidth]{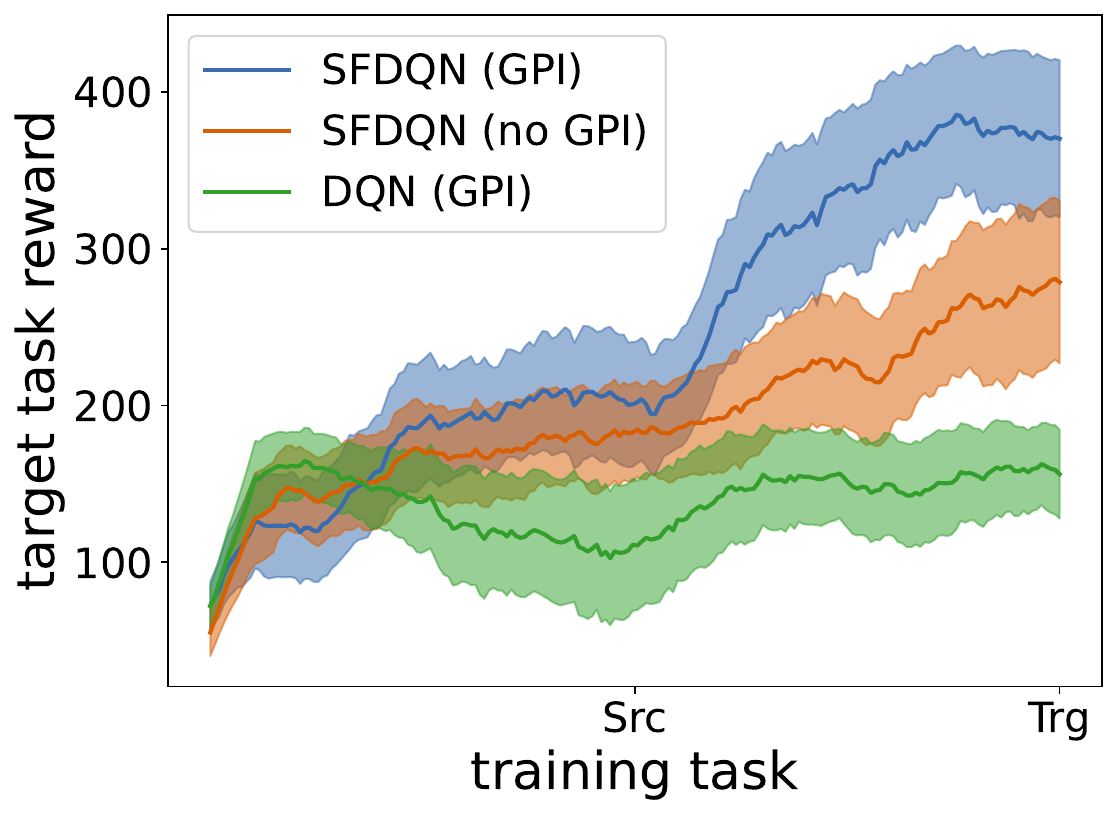}
 \caption{Comparison of DQN (GPI) and SF-DQN (with and without GPI)}\label{fig:reacher-sfdn-dqn}
\end{subfigure}
% \subcaption{Comparing SFDQN}
\hspace{0.2cm}
\begin{subfigure}[b]{0.48\textwidth}
 \centering
 \includegraphics[width=0.99\textwidth]{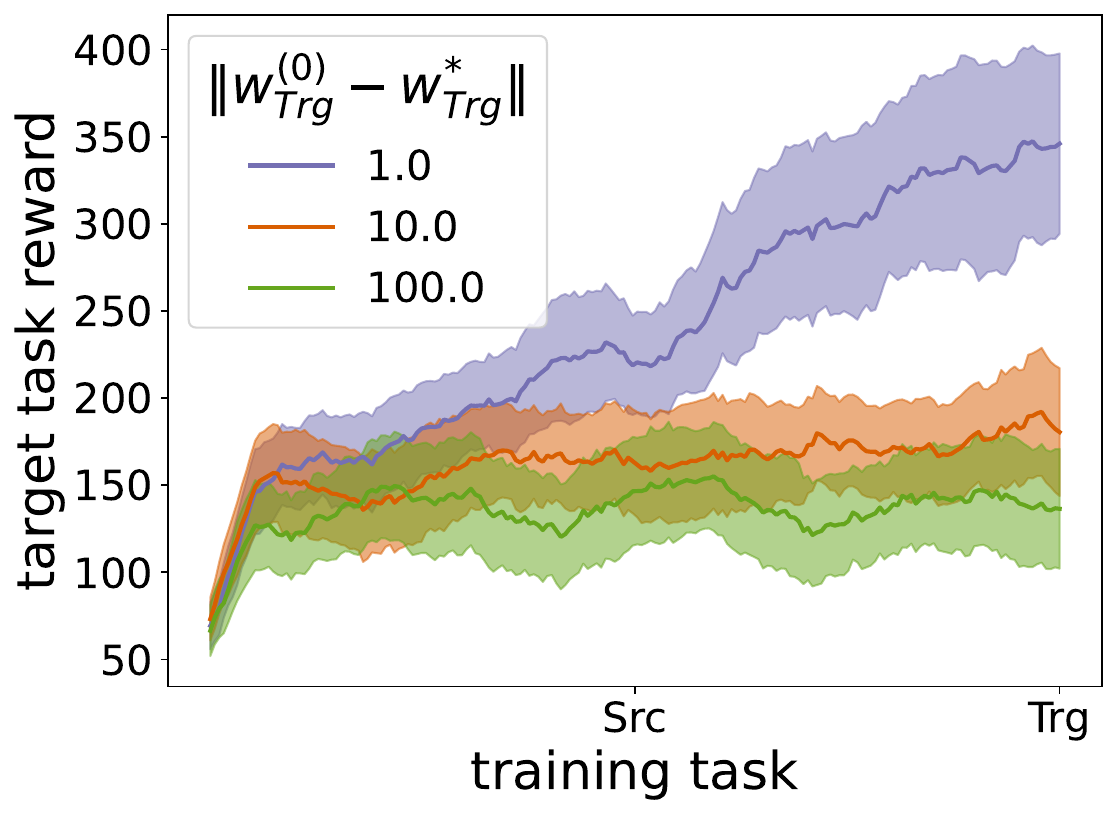}
 \caption{Comparison of different initializations for target task reward mapping}
 \label{fig:reacher-w-init}
\end{subfigure}

\caption{Single source to single target task transfer experiments on Reacher environment }
\label{fig:reacher-perf-uniform-source}
% \vspace{-0.4cm}
\end{figure} 

\paragraph{Comparison of SF-DQN (with and without GPI) and DQN (GPI).} We first provide a comparison of the performance of SF-DQN with GPI, SF-DQN without GPI, and DQN with GPI, in Figure \ref{fig:reacher-sfdn-dqn}. Here we consider the average transfer performance for four tasks, after training on a source task. It can be seen that SFDQN with GPI performs better compared to its no GPI counterpart. Both of these agents perform significantly better compared to DQN with GPI. hence, this result validates our theoretical results for the performance of these three methods.

\paragraph{Effect of $\|w^{(0)}_{Trg} - w^\star_{Trg}\|$.} Next, we investigate the performance of the SFDQN agent when the target task reward mappings are not known and learned simultaneously with successor features. We consider varying distances from the initial target task reward mapping to the true target task reward mapping. The results are shown in Figure \ref{fig:reacher-w-init}. It can be seen that when the reward mappings are initialized far away from the true reward mappings, the convergence of the SF-DQN agent is slower compared to that is initialized closer to the true reward mappings. This aligns with our convergence analysis for the SF-DQN agent with GPI.

\paragraph{Single source task to multiple target tasks transfer learning.}Next we compare the performance of SFDQN and DQN with GPI for transferring knowledge from single source tasks to multiple target tasks. The results are given in Figures \ref{fig:reacher-source-to-target-1}, \ref{fig:reacher-source-to-target-2}, and \ref{fig:reacher-source-to-target-3}. It can be seen that SFDQN outperforms DQN significantly for most target tasks, which shows the efficacy of knowledge transfer in SFDQN with GPI. The gap of performance seems to be different for different target tasks, suggesting that the performance gain for SFDQN with GPI can vary depending on the source and target task relationship.

\paragraph{Multiple source tasks to single target task transfer learning.} Next we investigate the effect of GPI for transferring knowledge from multiple source tasks to single target task. The results are given in Figure \ref{fig:reacher-4-source}. It can be seen that SFDQN outperforms DQN significantly, which shows the efficacy of knowledge transfer in SFDQN with GPI.

\begin{figure}[ht]
  % \hspace{-0.15cm}
  \centering
  \begin{subfigure}[b]{0.23\textwidth}
 \centering
 \includegraphics[width=0.99\textwidth]{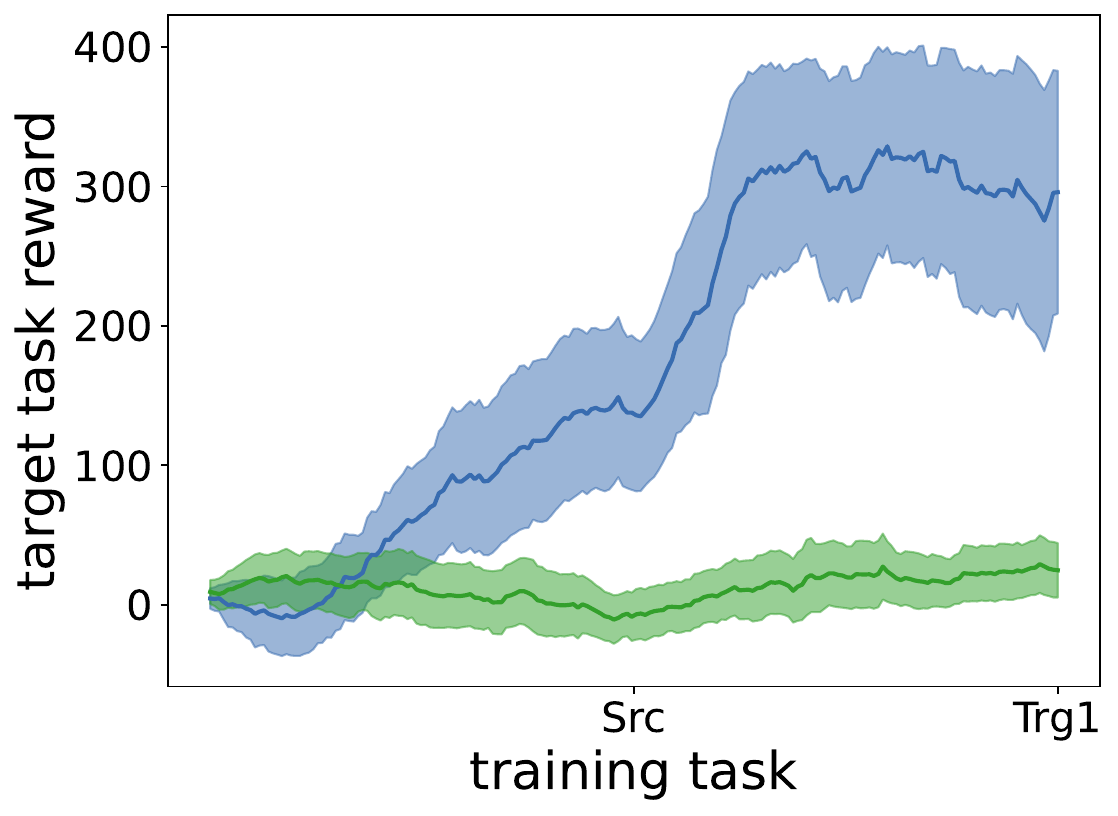}
 \caption{Source to target task 1 transfer performance }\label{fig:reacher-source-to-target-1}
\end{subfigure}
  \begin{subfigure}[b]{0.23\textwidth}
 \centering
 \includegraphics[width=0.99\textwidth]{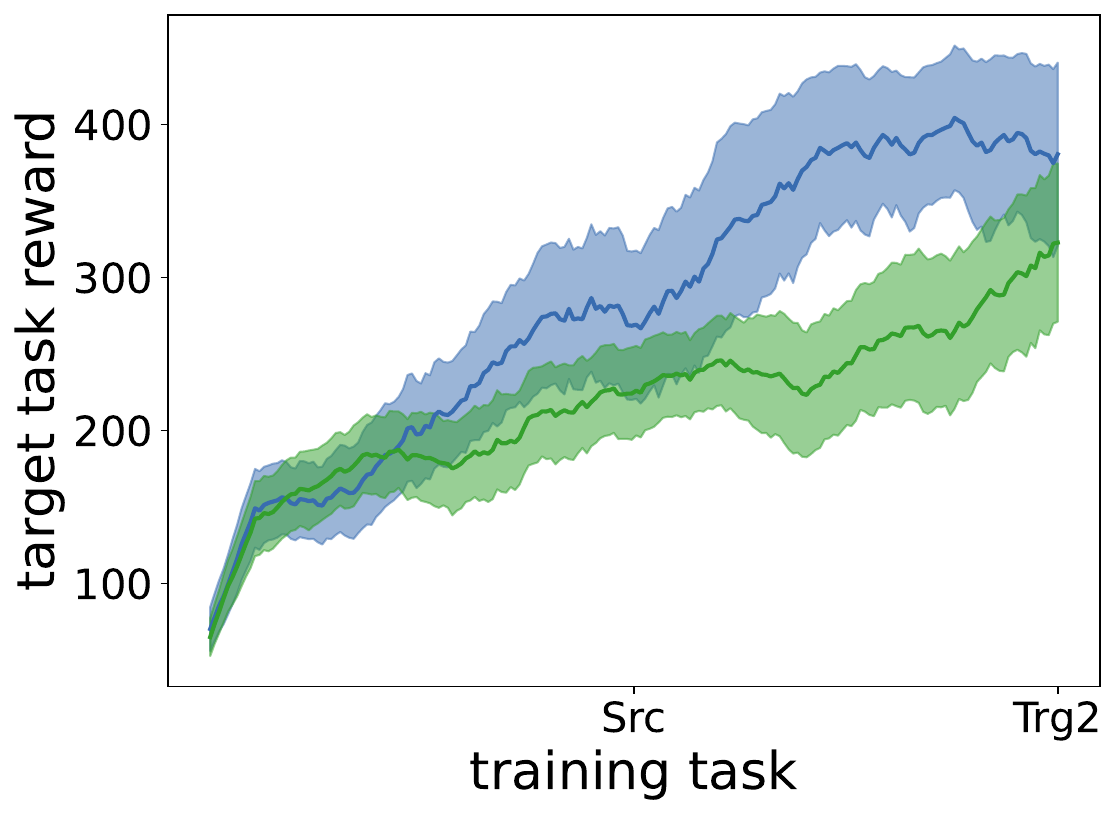}
 \caption{Source to target task 2 transfer performance }\label{fig:reacher-source-to-target-2}
\end{subfigure}
  \begin{subfigure}[b]{0.23\textwidth}
 \centering
 \includegraphics[width=0.99\textwidth]{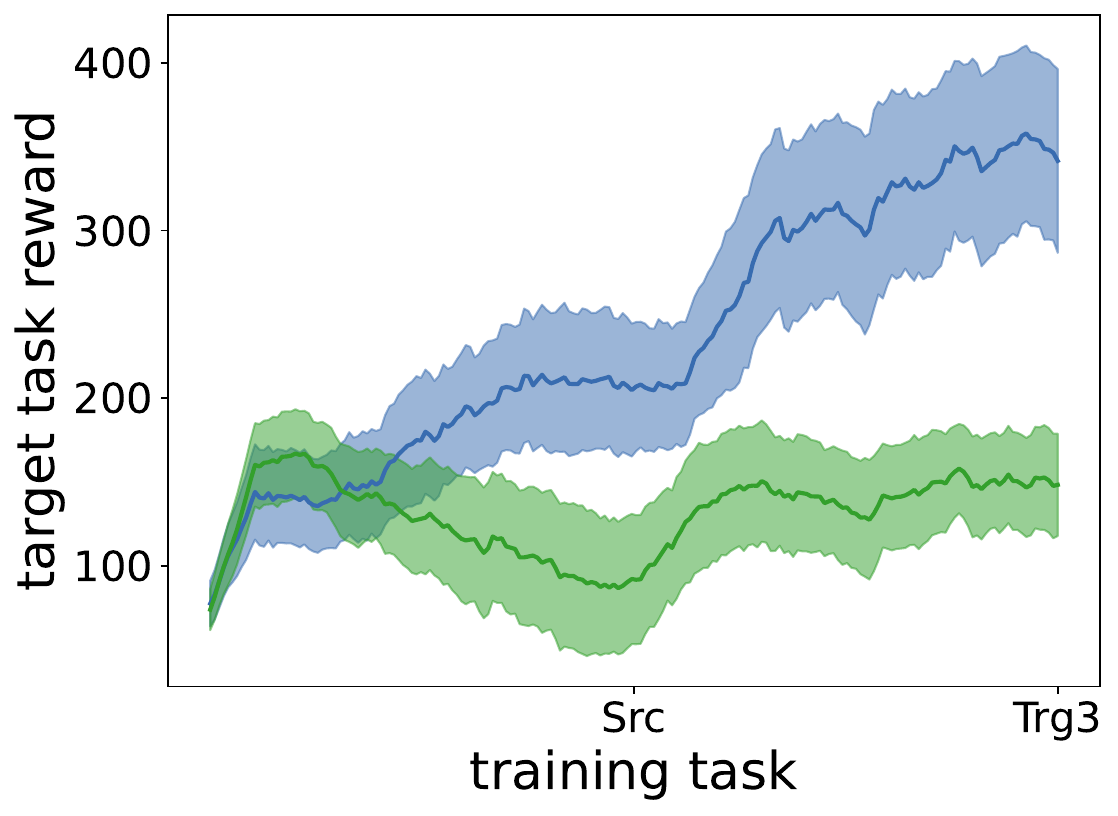}
 \caption{Source to target task 3 transfer performance }\label{fig:reacher-source-to-target-3}
\end{subfigure}
\begin{subfigure}[b]{0.23\textwidth}
 \centering
 \includegraphics[width=0.99\textwidth]{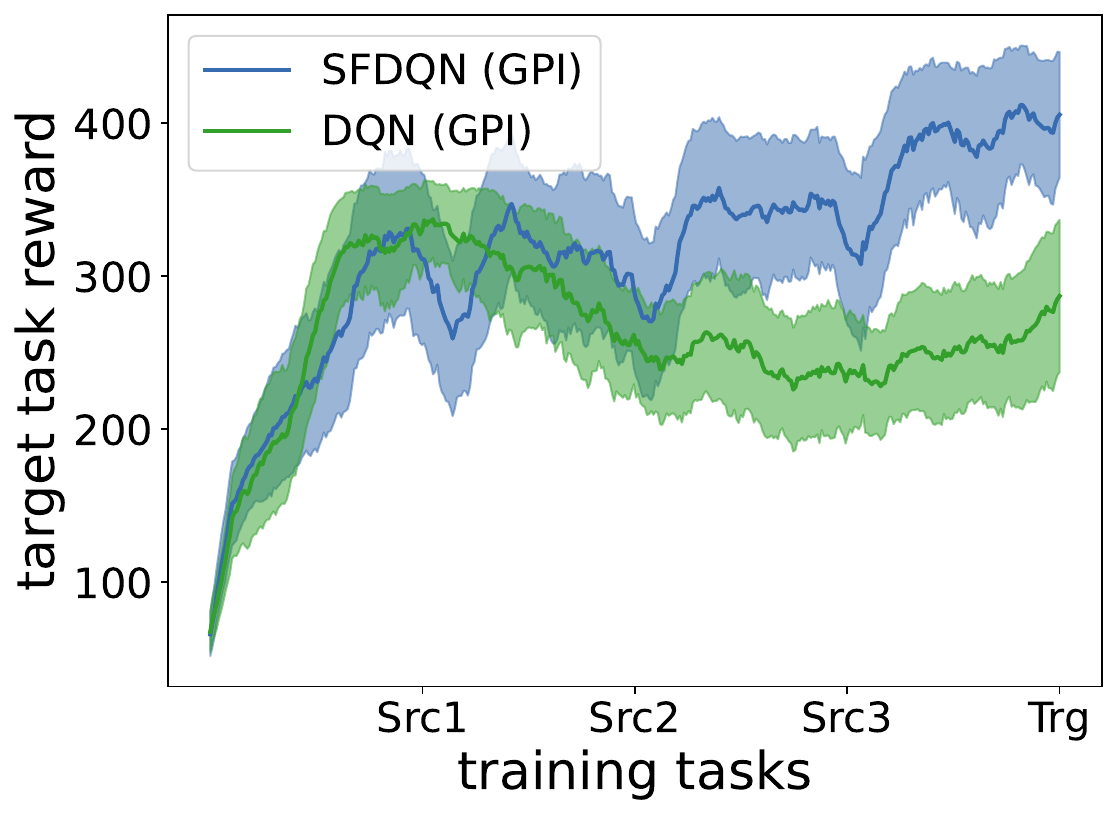}
 \caption{Multiple source to target task transfer performance}
 \label{fig:reacher-multi-source-to-target}
\end{subfigure}
\caption{Multple source/target tasks transfer experiments on Reacher environment}
\end{figure}

\section{Proof of lemmas in Appendix \ref{app:thm1}}\label{app:proof_of_lemma_1}
\subsection{Proof of Lemma \ref{Lemma: second_order_derivative}} \label{sec: Proof: second_order_derivative}
Lemma \ref{Lemma: second_order_derivative} provides the lower and upper bounds for the eigenvalues of the Hessian matrix of population risk function in \eqref{eqn:prf}.
According to Weyl's inequality in Lemma \ref{Lemma: weyl}, the eigenvalues of $\nabla^2_\ell  f(\cdot)$ at any fixed point $\theta$ can be bounded in the form of \eqref{eqn: thm1_main}.
Therefore, we first provide the lower and upper bounds for $\nabla^2_\ell  f$ at the desired ground truth $\theta^\star$. Then, the bounds for $\nabla^2_\ell  f$ at any other point $\theta$ is bounded through \eqref{eqn:prf} by utilizing the conclusion in Lemma \ref{Lemma: distance_Second_order_distance}.
Lemma \ref{Lemma: distance_Second_order_distance} illustrates the distance between the Hessian matrix of $f$ at $\theta$ and $\theta^*$.
Lemma \ref{Lemma: Zhong} provides the lower bound of  $\mathbb{E}_{\bfx}\big(\sum_{j=1}^K \bfal_j^\top\frac{\partial \psi}{\partial \theta_{\ell,k}}(\theta^\star) \big)^2$ when $\bfx$ belongs to sub-Gaussian distribution, which is used in proving the lower bound of the Hessian matrix in \eqref{eqn: lower}.

\begin{lemma}\label{Lemma: distance_Second_order_distance}
   Let $f(\theta)$ be the population risk function defined in \eqref{eqn:prf}. If $\theta$ is close to $\theta^\star$ such that 
   \begin{equation}\label{eqn: lemma2_initial}
        \|\theta-\theta^\star\|_2 \lesssim \frac{\rho_1}{K}    
   \end{equation}
   we have 
   \begin{equation}
       \| \nabla^2_\ell  f(\theta)- \nabla^2_\ell  f(\theta^\star) \|_2 \lesssim \frac{1}{K}\cdot\|\theta-\theta^\star\|_2.
   \end{equation}

\end{lemma}
\begin{lemma}\label{Lemma: Zhong}
Suppose the following assumptions hold:
\begin{enumerate}
    \item $\{\theta_{j}\}_{j=1}^{K}\in \mathbb{R}^{K_\ell}$ are linear independent,
    \item Let $p(\bfh): \mathbb{R}^{K_\ell}\longrightarrow 
        [~0~~1~]$ be the probability density for $\bfh$ such that $\mathbb{E}_{\bfh} \|\bfh\|_2^2\le +\infty$.
\end{enumerate}   
Let $\bfal \in \mathbb{R}^{K_\ell K_{\ell-1}}$ be the unit vector defined in \eqref{eqn: alpha_definition},
we have
   \begin{equation}\label{def: rho}
       \rho_1:=\min_{\|\bfal\|_2=1}\int_{\mathcal{R}} \Big(\sum_{j=1}^K \bfal^\top\bfh \phi^{\prime}(\theta_{\ell,j}^{\top}\bfh) \Big)^2 p_{H}(\bfh)\cdot  d \bfh >0,
   \end{equation}
    where $\mathcal{R}\subset \mathbb{R}^{K_\ell}$ with $\int_{\mathcal{R}} f_H(\bfh)>0$.
    Moreover, if further assuming $\bfh$ belongs to Gaussian distribution, we have $\rho_1 >0.091$.
\end{lemma}
\begin{lemma}\label{Lemma: h_bound}
Let $\bfh^{(\ell)}(\theta)$ be the function defined in \eqref{eqn: defi_h}. When $\theta$ is sufficiently close to $\theta^\star$, i.e., $\|\theta -\theta^\star\|_2$ is smaller than some positive constant $c<1$, we have 
\begin{equation}
\begin{gathered}
    \|\bfh^{(\ell)}(\theta)\|_2 \lesssim \|\bfx\|_2,\\
        \|\bfh^{(\ell)}(\theta) -\bfh^{(\ell)}(\theta^\star)\|_2\lesssim \|\theta-\theta^\star\|_2\cdot \|\bfx\|_2.
\end{gathered}
\end{equation}
\end{lemma}

\begin{proof}[Proof of Lemma \ref{Lemma: second_order_derivative}]
    Let $\lambda_{\max}(\theta)$ and $\lambda_{\min}(\theta)$ denote the largest and smallest eigenvalues of $\nabla^2_\ell  f(\theta)$ at $\theta$, respectively. Then, from Lemma \ref{Lemma: weyl}, we have 
    \begin{equation}\label{eqn: thm1_main}
    \begin{gathered}
        \lambda_{\max}(\theta) \le \lambda_{\max}(\theta^\star) + \| \nabla^2_\ell  f (\theta) - \nabla^2_\ell  f (\theta^\star) \|_2,\\
        \lambda_{\min}(\theta) \ge \lambda_{\min}(\theta^\star)  - \| \nabla^2_\ell  f (\theta) - \nabla^2_\ell  f (\theta^\star) \|_2.
    \end{gathered}
    \end{equation}
     Then, we provide the lower bound of the Hessian matrix of the population function at $\theta^\star$. Let $\mathcal{P}$ be the distribution for $\bfh^{(\ell)}(\theta)$ when $\bfx \sim \mu^\star$ with probability density function denoted as $p_H$. For any $\bfal\in\mathbb{R}^{K_\ell K}$ with $\|\bfal\|_2 =1$, we have 
    \begin{equation}\label{eqn: lower}
    \begin{split}
        &\min_{\|\boldsymbol{\alpha}\|_2=1} \bfal^\top \nabla^2_\ell  f(\theta^\star) \bfal\\
        =~&\frac{1}{K^2} \min_{\|\bfal\|_2=1 } \mathbb{E}_{\bfh\sim \mathcal{P}}\Big(\sum_{j=1}^K \bfal_j^\top\bfh^{(\ell)} \cJ_{\ell,k}\phi^{\prime}(\theta_{\ell,j}^{\star \top}\bfh^{(\ell)}) \Big)^2
         \\
        =~&\frac{1}{K^2} \min_{\|\bfal\|_2=1 } \int_{\mathbb{R}^{K_{\ell}-1}}\Big(\sum_{j=1}^K \bfal_j^\top\bfh^{(\ell)} \cJ_{\ell,k}\phi^{\prime}(\theta_{\ell,j}^{\star \top}\bfh^{(\ell)}) \Big)^2 p_H(\bfh^{(\ell)}) \cdot d \bf\bfh^{(\ell)}
         \\
        =~&\frac{1}{K^2} \min_{\|\bfal\|_2=1 } \int_{ \{\bfh^{(\ell)}\mid \cJ_{\ell,k}\neq 0 \} }\Big(\sum_{j=1}^K \bfal_j^\top\bfh^{(\ell)} \phi^{\prime}(\theta_{\ell,j}^{\star \top}\bfh^{(\ell)}) \Big)^2 p_H(\bfh^{(\ell)}) \cdot d \bf\bfh^{(\ell)}\\
        \gtrsim~&\frac{\rho_1}{K^2},
        \end{split}
    \end{equation}
    where the last inequality comes from Lemma \ref{Lemma: Zhong}, and Lemma \ref{Lemma: Zhong} holds since $\bfh^{(\ell)}$ belongs to sub-Gaussian distribution and $\theta_\ell$ is full rank.
    
    Next, the upper bound of $\nabla^2_\ell f$ can be bounded as 
    \begin{equation}
        \begin{split}
            &\max_{\|\boldsymbol{\alpha}\|_2=1} \bfal^\top \nabla^2_\ell  f(\theta^\star) \bfal\\
        =& \frac{1}{K^2} \max_{\|\bfal\|_2=1 }  \mathbb{E}_{\bfx}\Big(\sum_{j=1}^K \bfal_j^\top\bfh^{(\ell)} \cdot \cJ_{\ell,k}\phi^{\prime}(\theta_{\ell,j}^{\star\top}\bfh^{(\ell)}) \Big)^2
         \\
        =& \frac{1}{K^2} \max_{\|\bfal\|_2=1 } \mathbb{E}_{\bfx}\sum_{j_1=1}^K \sum_{j_2=1}^K \bfal_{j_1}^\top\bfh^{(\ell)} \cdot \cJ_{\ell,k}\phi^{\prime}(\theta_{\ell, j_1}^{\star \top}\bfh^{(\ell)}) \cdot  \bfal_{j_2}^\top\bfh^{(\ell)} \cdot  \cJ_{\ell,k}\phi^{\prime}(\theta_{\ell, j_2}^{\star \top}\bfh^{(\ell)})\\
        =&\frac{1}{K^2} \sum_{j_1=1}^K \sum_{j_2=1}^K \mathbb{E}_{\bfx} \bfal_{j_1}^\top\bfh^{(\ell)} \cdot \cJ_{\ell,k}\phi^{\prime}(\theta_{\ell,j_1}^{\star T}\bfh^{(\ell)}) \cdot \bfal_{j_2}^\top\bfh^{(\ell)} \cdot \cJ_{\ell,k}\phi^{\prime}(\theta_{\ell,j_2}^{\star\top}\bfh^{(\ell)})\\
        \le & \frac{1}{K^2} \max_{\|\bfal\|_2=1 }  \sum_{j_1=1}^K \sum_{j_2=1}^K 
            \Big[\mathbb{E}_{\bfx} (\bfal_{j_1}^\top\bfh^{(\ell)})^4 \cdot \mathbb{E}(\phi^{\prime}(\theta_{\ell,j_1}^{\star\top}\bfh^{(\ell)}))^4
            \cdot 
            \mathbb{E}_{\bfx} (\bfal_{j_2}^\top\bfh^{(\ell)})^4 
            \cdot
            \mathbb{E}_{\bfx} (\phi^{\prime}(\theta_{\ell,j_2}^{\star \top}\bfh^{(\ell)}))^4\Big]^{1/4}\\
        \le &\frac{1}{K^2} \max_{\|\bfal\|_2=1 }  \sum_{j_1=1}^K \sum_{j_2=1}^K 
            \Big[\mathbb{E}_{\bfx} 
            (\bfal_{j_1}^\top\bfx)^4 
            \cdot 
            \mathbb{E}_{\bfx} (\bfal_{j_2}^\top\bfx)^4\Big]^{1/4} \\
        \le & \frac{3}{K^2} \sum_{j_1=1}^K \sum_{j_2=1}^K 
        \|\bfal_{j_1}\|_2\cdot  \|\bfal_{j_2}\|_2
        \le \frac{6}{K^2}\sum_{j_1=1}^K \sum_{j_2=1}^K \frac{1}{2}\Big(\|\bfal_{j_1}\|_2^2+ \|\bfal_{j_2}\|_2^2\Big)\\
        =&\frac{6}{K}.
        \end{split}
    \end{equation}

    Therefore, we have 
        \begin{equation}
        \begin{split}
            \lambda_{\max}(\theta^\star)
            = 
            \max_{\|\boldsymbol{\alpha}\|_2=1} \bfal^\top \nabla^2_\ell  f(\theta^\star;p) \bfal
        % %  \le &\frac{1}{K^2} \max_{\|\bfal\|_2=1 } \lambda\sum_{j=1}^K \mathbb{E}_{\bfx}\Big( \bfal_j^\top\bfx \phi^{\prime}(\theta_{\ell,j}^{\star T}\bfx) \Big)^2
        %  +\max_{\|\bfal\|_2=1 }\widetilde{\lambda}\sum_{j=1}^K \mathbb{E}_{\widetilde{\bfx}}\Big( \bfal_j^\top\w[\bfx] \phi^{\prime}(\theta_{\ell,j}^{\star T}\w[\bfx]) \Big)^2\\
        \le& \frac{6}{K}.
        \end{split}
    \end{equation}
    Then, given \eqref{eqn: lemma2_initial}, we have 
    \begin{equation}\label{eqn: thm11_temp}
        \|\theta - \theta^\star \|_2 \lesssim \frac{2\rho_1}{K}.
    \end{equation}
    Combining \eqref{eqn: thm11_temp} and Lemma \ref{Lemma: distance_Second_order_distance}, we have 
    \begin{equation}\label{eqn: thm11_temp2}
        \| \nabla^2_\ell  f (\theta) - \nabla^2_\ell  f (\theta^\star) \|_2\lesssim \frac{\rho_1}{{ K^2}}.
    \end{equation}
    Therefore, from \eqref{eqn: thm11_temp2} and \eqref{eqn: thm1_main}, we have 
    \begin{equation}
        \begin{gathered}
            \lambda_{\max}(\theta) \le \lambda_{\max}(\theta^\star) + \| \nabla^2_\ell  f (\theta) - \nabla^2_\ell  f (\theta^\star) \|_2\le \frac{6}{K} + \frac{\rho_1}{2 K^2}\le \frac{7}{K},\\
        \lambda_{\min}(\theta) \ge \lambda_{\min}(\theta^\star)  - \| \nabla^2_\ell  f (\theta) - \nabla^2_\ell  f (\theta^\star) \|_2\ge \frac{\rho_1}{K^2} - \frac{\rho_1}{{2 K^2}} = \frac{\rho_1}{2 K^2},
        \end{gathered}
    \end{equation}
    which completes the proof. 
\end{proof}

\subsection{Proof of Lemma \ref{Lemma: first_order_derivative}}
The error bound between $\|\nabla_\ell f -g_t \|_2$ is divided into bounding $\bfI_1$, $\bfI_2$, $\bfI_3$, and $I_4$ as shown in \eqref{eqn:I}. 
$\bfI_1$  represents the deviation of the gradient of $\mathcal{D}_t$ to their expectation, which can be bounded through concentration inequality. 
$\bfI_2$ is derived from the distribution shift between the trajectory and its stationary distribution, which can be bounded with assumption \ref{ass2}.
$\bfI_3$  come from the data distribution shift between the behavior policy and optimal policy.
$\bfI_4$  comes from the inconsistency of the "noisy" label and the "ground truth" label in the population risk function \eqref{eqn:prf}. To ensure a smooth presentation, we will defer the proof of $I_1 - I_4$ until we have completed the main proof of Lemma \ref{Lemma: first_order_derivative}.
\begin{proof}[Proof of Lemma \ref{Lemma: first_order_derivative}]
    From \eqref{eqn: gradient_theta}, we know that 
    \begin{equation}\label{eqn_2:1}
        \begin{split}
        &g^{(t)}(\theta^{(t)}_{\ell,k};\mathcal{X}_m)\\
        =& \sum_{m\in\mathcal{D}_t}\big(\psi(\theta^{(t)};\bfs_m,a_m) -y_m^{(t)}\big)\cdot \frac{\partial \psi(\theta^{(t)};\mathcal{X}_m)}{\partial \theta_{\ell,k}}\\ 
        = &\sum_{m\in\mathcal{D}_t}\Big( \psi(\theta^{(t)};\bfs_m,a_m) - \phi(\theta^\star;\bfs_m,a_m) 
        % + \gamma \cdot \max_a \psi(\bfs_m,a;\theta^\star)
         - \gamma \cdot \psi(\bfs_m^\prime,a_m^\prime;\theta^{(t)})  \Big)
        \cdot \frac{\partial \psi(\theta^{(t)};\mathcal{X}_m)}{\partial \theta_{\ell,k}}\\
        = &\sum_{m\in\mathcal{D}_t}\Big( \psi(\theta^{(t,n)};\bfs_m,a_m) - \psi(\theta^\star;\bfs_m,a_m) + \gamma \cdot \max_{a^\prime}\psi(\bfs_m^\prime,a^\prime;\theta^\star)\\
        &\qquad \qquad - \gamma \cdot\psi(\bfs_m^\prime,a_m^\prime;\theta^{(t)})  \Big)
        \cdot \frac{\partial \psi(\theta^{(t,n)};\mathcal{X}_m)}{\partial \theta_{\ell,k}}\\
        =& \sum_{m\in\mathcal{D}_t}\Big( \psi(\theta^{(t)};\bfs_m,a_m) - \psi(\theta^\star;\bfs_m,a_m) \Big)\cdot \frac{\partial \psi(\theta^{(t)};\mathcal{X}_m)}{\partial \theta_{\ell,k}} \\
         & +\gamma \cdot \Big(\max_{a^\prime}\psi(\bfs_m^\prime,a^\prime;\theta^\star) -  \psi(\bfs_m^\prime,a_m^\prime;\theta^{(t)}) \Big)\cdot \frac{\partial \psi(\theta^{(t)};\mathcal{X}_m)}{\partial \theta_{\ell,k}}\\
         % &- \frac{1}{N}\sum_{n=1}^N\xi_n \cdot \frac{\partial \psi(\bfW;\bfs_m,a_m)}{\partial \bfw_{\ell,k}}.
        :=& \sum_{m\in\mathcal{D}_t} b^{(t)}_{\ell,k}(\theta^{(t)};\mathcal{X}_m) +\Delta b^{(t)}_{\ell,k}(\theta^{(t)};\mathcal{X}_m),
        \end{split}
    \end{equation}
    where we have 
    \begin{equation}\label{eqn: h}
        b^{(t)}_{\ell,k}(\theta^{(t)};\mathcal{X}_m)
        =\Big( \psi(\theta^{(t)};\bfs_m,a_m) - \psi(\theta^\star;\bfs_m,a_m) \Big)\cdot \frac{\partial \psi(\theta^{(t)};\mathcal{X}_m)}{\partial \theta_{\ell,k}} 
    \end{equation}
    and
    \begin{equation}
        \Delta b^{(t)}_{\ell,k}(\theta^{(t)};\mathcal{X}_m)
        = \Big(\max_{a^\prime}\psi(\theta^\star;\bfs_m^\prime,a^\prime) -  \psi(\theta^{(t-1)};\bfs_m^\prime,a_m^\prime )\Big)\cdot \frac{\partial \psi(\theta^{(t)};\mathcal{X}_m)}{\partial \theta_{\ell,k}}.
    \end{equation}
    Then, let us define $\bar{b}^{(t)}_{\ell,k}$ as 
\begin{equation}\label{eqn: h_bar}
\begin{split}
    &\bar{b}^{(t)}_{\ell,k}(\theta;\mathcal{X}) 
    = \mathbb{E}_{(\bfs,a)\sim \mu_t}  \Big(\psi(\theta;\bfs,a) - \psi(\theta^\star;\bfs,a)\Big)\cdot \nabla_{\theta}\psi(\theta;\bfs,a).
\end{split}
\end{equation}
    
    From \eqref{eqn:prf}, we know that
    \begin{equation}\label{eqn_2:2}
        \begin{split}
        \frac{\partial f_{\pi^\star}}{\partial \theta_{\ell,k}}(\theta^{(t)}) 
        = \mathbb{E}_{(\bfs,a)\sim \mu^{\star}} \Big( \phi(\theta^{(t)};\bfs,a) - \phi(\theta^\star;\bfs,a) \Big)\cdot \frac{\partial \phi(\theta^{(t)};\bfs,a)}{\partial \theta_{\ell,k}}.
        \end{split}
    \end{equation}
    Then, from \eqref{eqn_2:1} and \eqref{eqn_2:2}, we have 
    \begin{equation}\label{eqn:I}
        \begin{split}
            &g^{(t)}(\theta^{(t)}_{\ell,k};\mathcal{X}_m) - \frac{\partial {f}_{\pi^\star} }{ \partial \theta_{\ell,k}} (\theta^{(t)};\mathcal{X}_m)\\
            =& \sum_{m\in\mathcal{D}_t} b^{(t)}_{\ell,k}(\theta^{(t)};\mathcal{X}_m) +\Delta b^{(t)}_{\ell,k}(\theta^{(t)};\mathcal{X}_m) - \frac{\partial {f}_{\pi^\star} }{ \partial \theta_{\ell,k}} (\theta^{(t)};\mathcal{X}_m)\\
            % =&g^{(t)}(\theta^{(t)}_{\ell,k};\mathcal{X}_m) - \mathbb{E}_{\mathcal{X}_m\sim \mathcal{D}_t}~g^{(t)}(\theta^{(t)}_{\ell,k};\mathcal{X}_m) + \mathbb{E}_{\mathcal{X}_m\sim \mathcal{D}_t}~g^{(t)}(\theta^{(t)}_{\ell,k};\mathcal{X}_m) - \frac{\partial {f}_{\pi^\star} }{ \partial \theta_{\ell,k}} (\theta^{(t)};\mathcal{X}_m)\\
            = &~\bigg[b^{(t)}_{\ell,k}(\theta^{(t)}_{\ell,k};\mathcal{X}_m) - \mathbb{E}_{\mathcal{X}_m\sim \mathcal{D}_t}~b^{(t)}_{\ell,k}(\theta^{(t)}_{\ell,k};\mathcal{X}_m)\bigg] + \bigg[ 
            \mathbb{E}_{\mathcal{X}_m\sim \mathcal{D}_t}~\hkt(\theta^{(t)};\mathcal{X}_m) - \bar{b}^{(t)}_{\ell,k}(\theta^{(t)};\mathcal{X}_m) \bigg]\\
            &+ \bigg[\bar{b}^{(t)}_{\ell,k}(\theta^{(t)}) - \frac{\partial {f}_{\pi^\star} }{ \partial \theta_{\ell,k}}(\theta^{(t)})\bigg] + \mathbb{E}_{\mathcal{X}_m\sim \mathcal{D}_t}\Delta \hkt(\theta^{(t)};\mathcal{X}_m)\\
            :=& \bfI_1 + \bfI_2 +\bfI_3 +\bfI_4.
        \end{split}
    \end{equation}
    Therefore, we have 
    \begin{equation}
        \begin{split}
            &\Big\|g^{(t)}(\theta_{\ell,k}^{(t)};\mathcal{X}_m) - \frac{\partial {f}_{\pi^\star} }{ \partial \theta_{\ell,k}} (\theta^{(t)})\Big\|_2 \le \|\bfI_1\|_2 + \|\bfI_2\|_2 + \|\bfI_3\|_2 +\|\bfI_4\|_2.\\
        \end{split}
    \end{equation}

    Next, we first provide the bound for $\|\bfI_1\|_2$, $\|\bfI_2\|_2$, $\|\bfI_3\|_2$, and $\|\bfI_4\|_2$ as
    \begin{equation}
        \begin{gathered}
            \|\bfI_1\|_2 \le \frac{1}{K_\ell} \cdot \|\theta -\theta^\star \|_2 \cdot \sqrt{\frac{d\log q}{|\mathcal{D}_t|}},\\
            \|\bfI_2\|_2 \le \frac{R_{\max}}{1-\gamma}\cdot (1+\gamma)\tau^\star\cdot \eta_{t-\tau^\star},\\
            \|\bfI_3\|_2 \le |\mathcal{A}|\cdot \frac{R_{\max}}{1-\gamma}\cdot (1+\log_\nu \lambda^{-1} + \frac{1}{1-\nu})\cdot C_t,\\
            \|\bfI_4\|_2 \le \frac{\gamma}{K_\ell}\cdot \| \theta^{(t} -\theta^\star\|_2,
        \end{gathered}
    \end{equation}
    where $|\mathcal{A}|$ is the size of action space.
    The details for the derivation of $I_1$- $I_4$ can be found after the proof.

    Let $\boldsymbol{\alpha}\in \mathbb{R}^{Kd}$ and  $\boldsymbol{\alpha}_{j}\in \mathbb{R}^d$ with $\boldsymbol{\alpha} =[\boldsymbol{\alpha}_{1}^T, \boldsymbol{\alpha}_{2}^T, \cdots, \boldsymbol{\alpha}_{K}^T]^T$, with probability at least $1-q^{-d}$, we have 
    \begin{equation}
        \begin{split}
            \|g^{(t)}(\theta_\ell;\theta) -\nabla_\ell {f}_{\pi^\star}(\theta) \|_2^2
            = & \Big| \boldsymbol{\alpha}^T \big( g^{(t)}(\theta) -\nabla {f}_{\pi^\star}(\theta) \big)  \Big|^2\\
            \le & \sum_{k=1}^K\Big| \boldsymbol{\alpha}_k^T \big(g^{(t)}(\theta_{\ell,k};\theta) - \frac{\partial {f}_{\pi^\star} }{ \partial \theta_{\ell,k} } (\theta)\big) \Big|^2\\
            \le & \sum_{k=1}^K\Big\|g^{(t)}(\theta_{\ell,k};\theta) - \frac{\partial {f}_{\pi^\star} }{ \partial \theta_{\ell,k} } (\theta)\Big\|_2^2\cdot \|\boldsymbol{\alpha}_k\|_2^2\\
            \le & \max_k \Big\|g^{(t)}(\theta_{\ell,k};\theta) - \frac{\partial {f}_{\pi^\star} }{ \partial \theta_{\ell,k} } (\theta)\Big\|_2^2.
            % \lesssim & \sum_{k=1}^K
            % \Big( \sum_{j=1}^K  \frac{2-\epsilon}{K^2}\cdot \| \theta_{\ell,j} - \theta_{\ell,j}^\star \|_2\sqrt{\frac{d\log q}{N}}
            % +  \sum_{j=1}^K \frac{(1-\epsilon/2)\gamma}{K^2} \|\theta_{\ell,j}^\star -\theta_{j,t-1} \|_2\\ 
            % &+ \frac{1}{K}\sqrt{\frac{\big(C_t+\epsilon(1-C_t)\big)\cdot d\cdot \log q}{N}}\cdot |\xi|
            % \Big)\cdot \|\boldsymbol{\alpha}_k\|_2\\
            % \le & \frac{\gamma}{K} \cdot \|\theta^{(t-1)} -\theta^\star\|_2+\frac{R_{\max}}{1-\gamma} \cdot |\mathcal{A}|\cdot C_t
        \end{split}
    \end{equation}
    In conclusion, we have 
    \begin{equation}
        \begin{split}
            &\|g^{(t)}(\theta_\ell;\theta) -\nabla_\ell {f}_{\pi^\star}(\theta) \|_2\\ 
        \le& \max_k \Big\|g^{(t)}(\theta_{\ell,k};\theta) - \frac{\partial {f}_{\pi^\star} }{ \partial \theta_{\ell,k} } (\theta)\Big\|_2^2\\
        \le & \max_k \|\bfI_1(k) \|_2 + \|\bfI_2(k) \|_2 +\|\bfI_3(k) \|_2 +\|\bfI_4(k) \|_2\\
        \le & \frac{1}{K_\ell} \cdot \|\theta -\theta^\star \|_2 \cdot \sqrt{\frac{d\log q}{|\mathcal{D}_t|}} 
        + \frac{R_{\max}}{1-\gamma}\cdot (1+\gamma)\tau^\star\cdot \eta_{t-\tau^\star}\\
        &+ |\mathcal{A}|\cdot \frac{R_{\max}}{1-\gamma}\cdot (1+\log_\nu \lambda^{-1} + \frac{1}{1-\nu})\cdot C_t
        + \frac{\gamma}{K_\ell}\cdot \| \theta^{(t)} -\theta^\star\|_2,
        \end{split}
        \end{equation}
    where $\tau^\star = \min\{t\mid \lambda\nu^t\le \eta_T \}$
\end{proof}
\subsubsection{Proof of upper bound of $I_1$}
\begin{proof}
           We define a random variable 
    $$Z^{(\ell)}(k) = \big( \psi(\theta;\bfs,a) - \psi(\theta^\star;\bfs,a)\big)\cdot \cJ_{\ell,k}\cdot \boldsymbol{\alpha}^T\bfh^{(\ell)}(\theta)$$ with $(\bfs,a)\sim \mathcal{D}_{t}$ and 
    $$Z_{m}^{(\ell)}(k) = \big( Q(\bfx_m;\theta) - Q(\bfx_m;\theta^\star)\big)\cdot \cJ_{\ell,k}\cdot \boldsymbol{\alpha}^T\bfh_n^{(\ell)}(\theta)$$ as the realization of $Z^{(\ell)}$ for $m\in\mathcal{D}_t$, where $\boldsymbol{\alpha}$ is any fixed unit vector.

    According to the definition of $I_1$ in \eqref{eqn:I}, we can rewrite $\bfI_1$ as 
    \begin{equation}\label{eqn:rewrite_I_1}
            \bfI_1 = \frac{1}{K_{\ell}} \Big[\sum_{m\in\mathcal{D}_t} Z_m^{(\ell)}(k) - \mathbb{E}_{(\bfs,a)\sim \mathcal{D}_t} Z^{(\ell)}(k) \Big].
    \end{equation}
    
    Then, for any $p\in \mathbb{N}^+$, we have 
    \begin{equation}\label{eqn: temppp}
        \begin{split}
        \big(\mathbb{E}|Z^{(\ell)}|^p\big)^{1/p}
        =& \Big( \mathbb{E}_{\mathcal{X}\sim\mathcal{D}_{t}} |\psi(\theta;\bfs,a) - \psi(\theta^\star;\bfs,a)|^p \cdot 
        |\cJ_{\ell,k}\sigma^\prime(\bfw_{\ell,k}^\top\bfx)| \cdot|\boldsymbol{\alpha}^T\bfh^{(\ell)}|^p \Big)^{1/p}\\
        \le& \Big( \mathbb{E}_{\mathcal{X}\sim\mathcal{D}_{t,1}} |\psi(\theta;\bfs,a) - \psi(\theta^\star;\bfs,a)|^p \cdot 
        |\boldsymbol{\alpha}^T\bfh^{(\ell)}|^p \Big)^{1/p}\\
        \le & \Big( \mathbb{E}_{\mathcal{X}\sim\mathcal{D}_{t}} \Big| \|\theta-\theta^\star\|_2\cdot \|\bfx(\bfs,a)\|_2\Big|^p \cdot 
        \big|\boldsymbol{\alpha}^T\bfx(\bfs,a)\big|^p \Big)^{1/p}\\
        \lesssim& \cdot \|\theta-\theta^\star\|_2 \cdot p.
        \end{split}
    \end{equation}
    From Definition \ref{Def: sub-exponential}, we know that $Z^{(\ell)}$ belongs to sub-exponential distribution with $\|Z^{(\ell)}\|_{\psi_1}\lesssim \| \theta-\theta^\star \|_2$.
    Therefore, by Chernoff inequality, for any $s\in\mathbb{R}$, we have 
    \begin{equation}
        \mathbb{P}\bigg\{ \Big| \frac{1}{|\mathcal{D}_t|}\sum_{m\in\mathcal{D}_t} Z_m^{(\ell)}(k) -\mathbb{E}Z^{(\ell)}(k) \Big| < t \bigg\} \le 1- \frac{e^{-(\|\theta-\theta^\star\|_2)^2\cdot |\mathcal{D}_t|\cdot s^2}}{e^{|\mathcal{D}_t|\cdot st}}.
    \end{equation}
    
    Let $t= \|\theta-\theta^\star\|_2\sqrt{\frac{d\log q}{N}}$ and $s = \frac{2}{\|\theta-\theta^\star\|_2}\cdot t$ for some large constant $q>0$. Then,    with probability at least $1-q^{-d}$, we have
    \begin{equation}\label{eqn: tempppp}
        \begin{split}
                \Big| \frac{1}{|\mathcal{D}_t|}\sum_{m\in\mathcal{D}_t} Z_m^{(\ell)}(k) -\mathbb{E}Z^{(\ell)}(k) \Big| \lesssim \|\theta-\theta^\star\|_2 \cdot \sqrt{\frac{d\log q}{|\mathcal{D}_t|}}.
        \end{split}
    \end{equation}

    From Lemma \ref{Lemma: spectral_norm_on_net} and \eqref{eqn:rewrite_I_1}, with probability at least $1-|\mathcal{S}_{\frac{1}{2}}(d)|\cdot q^{-d}$, we have 
    \begin{equation}\label{eqn: 2_temp1}
        \begin{split}
            \|\bfI_1\|_2 
            \le 2\cdot \frac{1}{K_\ell}
            \Bigg| \frac{1}{|\mathcal{D}_t|}\sum_{m\in\mathcal{D}_{t}} Z_m^{(\ell)} -\mathbb{E}Z^{(\ell)} \Bigg|
            \lesssim   \frac{1}{K_\ell}  \|\theta-\theta^\star\|_2\cdot \sqrt{\frac{d\log q}{|\mathcal{D}_t|}}.
        \end{split}
    \end{equation}

    From Lemma \ref{Lemma: covering_set}, we know that $|\mathcal{S}_{\frac{1}{2}}(d)|\le 5^d$. Therefore, the probability for \eqref{eqn: 2_temp1} holds is at least $1-\big(\frac{q}{5}\big)^{-d}$. Because $q\gg 5$, we denote the probability as $1-q^{-d}$ for convenience.
\end{proof}

\subsubsection{Proof of upper bound of $I_2$}
 \begin{proof}
 $\bfI_2$ is the bias of the data because the data $(\bfs,a)$ at iteration $t$ depends on the neural network parameters $\theta^{(t)}$. 
    Recall the definition of $\hkt$ and $\bar{b}^{(t)}_{\ell,k}$, we define
    \begin{equation}
        \Delta_t = \hkt(\theta^{(t)};\mathcal{X}_m) - \bar{b}^{(t)}_{\ell,k}(\theta^{(t)};\mathcal{X}_m).
    \end{equation}
It is easy to verify that 
\begin{equation}
\begin{gathered}
    \|\hkt(\theta;\mathcal{X}_m) - \hkt(\tilde{\theta};\mathcal{X}_m) \|_2\le (1+\gamma)\cdot \|\theta-\tilde{\theta}\|_2,\\
    \|\bar{b}^{(t)}_{\ell,k}(\theta;\mathcal{X}_m) - \bar{b}^{(t)}_{\ell,k}(\tilde{\theta};\mathcal{X}_m) \|_2\le (1+\gamma)\cdot \|\theta-\tilde{\theta}\|_2,\\
   \textit{and}\qquad  \|\hkt\|\lesssim \frac{R_{\max}}{1-\gamma}.
\end{gathered}
\end{equation}
Then, we have 
\begin{equation}
    \Delta_t(\theta)- \Delta_t(\tilde{\theta}) \lesssim (1+\gamma)\cdot \|\theta-\tilde{\theta}\|_2.
\end{equation}
Therefore, we have 
\begin{equation}
    \Delta_t(\theta^{(t)})\le \Delta_{t}(\theta^{(t-\tau)}) + \frac{1+\gamma}{1-\gamma}\cdot R_{\max}\cdot \sum_{i=t-\tau}^{t-1}\eta_i. 
\end{equation}
Then, we need to bound $\delta_t(\theta^{(t-\tau)})$.

Let us define the observed tuple $O_t(\bfs, a, s^\prime)$ as the collection of the state, action, and the next state at the $t$-th iteration. Note that 
\begin{equation}
    \theta^{(t-\tau)} \longrightarrow \bfs_{t-\tau} \longrightarrow \bfs_{t} \longrightarrow O_{t} 
\end{equation}
forms a Markov chain introduced by the policy $\pi_{t}$.

Let $\tilde{\theta}^{(t-\tau,0)}$ and $\widetilde{O}_t$ be independently drawn from the marginal distributions of ${\theta}^{(t-\tau,0)}$ and $O_t$, respectively.

With Lemma 9 in \cite{BRS18}, we have 
\begin{equation}
    \mathbb{E}~\Delta_t(\theta^{(t-\tau)},O_t) - \mathbb{E}~\Delta_t(\tilde{\theta}^{(t-\tau)},\tilde{O}_t) \lesssim~ 2\sup_{\theta, O}|\Delta_t(\theta,O)|\cdot \lambda\cdot \nu^\tau.
\end{equation}
By definition, we have $\mathbb{E}~\Delta_m(\tilde{\theta}^{(t-\tau)},\widetilde{O}_t) = 0$ and 
\begin{equation}
    |\Delta_t(\theta,O)| \le \frac{2~R_{\max}}{1-\gamma}.
\end{equation}
Therefore, we have 
\begin{equation}
\begin{split}
    \mathbb{E}\Delta_t(\theta^{(t)})
    \le& \mathbb{E}\Delta_{t}(\theta^{(t-\tau)}) + \frac{1+\gamma}{1-\gamma}\cdot R_{\max}\cdot \sum_{i=t-\tau}^{t-1}\eta_i\\
    \le & \frac{R_{\max}}{1-\gamma}\Big(
    \lambda\cdot \nu^\tau + (1+\gamma)\cdot \tau \cdot \eta_{t-\tau}
 \Big),
\end{split}
\end{equation}
where the last inequality comes from the fact that the step size $\eta_m$ is non-increasing. 

Choose $\tau^\star = \min\big\{ t=0, 1, 2, \cdots\mid\lambda \nu^\tau \le \eta_T \big\}$. When $t\le \tau^\star$, we choose $\tau =t$ and have
\begin{equation}\label{eqn: iid_1}
    \mathbb{E}\Delta_t(\theta^{(t)}) \le \frac{R_{\max}}{1-\gamma}\cdot \tau^\star\cdot \eta_{0}.
\end{equation}
When $n>\tau^\star$, we can choose $\tau = \tau^\star$ and obtain 
\begin{equation}\label{eqn: iid_2}
    \mathbb{E}\Delta_t(\theta^{(t)}) \le \frac{R_{\max}}{1-\gamma}\cdot (1+\gamma)\tau^\star\cdot \eta_{t-\tau^\star}.
\end{equation}
Combining \eqref{eqn: iid_1} and \eqref{eqn: iid_2}, we have 
\begin{equation}
    |\bfI_2| \le \frac{R_{\max}}{1-\gamma}\cdot (1+\gamma)\tau^\star\cdot \eta_{\max\{0,t-\tau^\star\}},
\end{equation}
where $\tau^\star = \min\{t\mid \lambda\nu^t\le \eta_T \}$.   
 \end{proof}  

 \subsubsection{Proof of bound of $I_3$}
\begin{proof}
    We have 
    \begin{equation}
    \begin{split}
        \bfI_3
        = & \bar{b}^{(t)}_{\ell,k}(\theta^{(t)}) - \frac{\partial {f}_{\pi^\star} }{ \partial \theta_{\ell,k}}(\theta^{(t)})\\
        = & \mathbb{E}_{(\bfs,a)\sim \mu_t} \Big( \psi(\theta;\bfs,a) - \psi(\theta^\star;\bfs,a) \Big)\cdot \frac{\partial \psi(\theta;\bfs,a)}{\partial \theta_{\ell,k}}\\
        &\qquad -\mathbb{E}_{(\bfs,a)\sim \mu^{\star}} \Big( \psi(\theta;\bfs,a) - \psi(\theta^\star;\bfs,a) \Big)\cdot \frac{\partial \psi(\theta;\bfs,a)}{\partial \theta_{\ell,k}}\\
        = & \mathbb{E}_{(\bfs,a)\sim \mu_t} \Big( \psi(\theta;\bfs,a) - r(\bfs,a) -\gamma\cdot \mathbb{E}_{\bfs'\sim p_{\bfs,\bfs'}^a} \max_{a'}\psi(\theta^\star;\bfs',a')  \Big)\cdot \frac{\partial \psi(\theta;\bfs,a)}{\partial \theta_{\ell,k}}\\
        &-\mathbb{E}_{(\bfs,a)\sim \mu^{\star}} \Big( \psi(\theta;\bfs,a) - r(\bfs,a) -\gamma\cdot \mathbb{E}_{\bfs'\sim p_{\bfs,\bfs'}^a} \max_{a'}\psi(\theta^\star;\bfs',a') \Big)\cdot \frac{\partial \psi(\theta;\bfs,a)}{\partial \theta_{\ell,k}}\\
        = &\mathbb{E}_{(\bfs,a)\sim \mu_t, \bfs'\sim p_{\bfs,\bfs'}^a}\Big( \psi(\theta;\bfs,a) - r(\bfs,a) -\gamma\cdot \max_{a'}\psi(\theta^\star;\bfs',a') \Big)\cdot \frac{\partial \psi(\theta;\bfs,a)}{\partial \theta_{\ell,k}}\\
        &- \mathbb{E}_{(\bfs,a)\sim \mu^\star, \bfs'\sim p_{\bfs,\bfs'}^a} \Big( \psi(\theta;\bfs,a) - r(\bfs,a) -\gamma\cdot  \max_{a'}\psi(\theta^\star;\bfs',a') \Big)\cdot \frac{\partial \psi(\theta;\bfs,a)}{\partial \theta_{\ell,k}}
    \end{split}    
    \end{equation}
    % Then, we have 
    %     \begin{equation}
    %     \begin{split}
    %     &\Big|\int_{(\bfs,a)}\int_{\bfs'}\big(\mu^\star(d\bfs,da) \mathcal{P}(d\bfs'|\bfs,a) - \mu_{t}(d\bfs,da) \mathcal{P}(d\bfs'|\bfs,a) \big)\Big|\\
    %         = & \Big|\int_{(\bfs,a)}\int_{\bfs'}\big(\mathcal{P}^\star(d\bfs)\pi^\star(da|\bfs) \mathcal{P}(d\bfs'|\bfs,a) - \mathcal{P}_{t}(d\bfs)\pi_{t,1}(da|d\bfs) \mathcal{P}(d\bfs'|\bfs,a) \big)\Big|\\
    %         \le & \Big|\int_{(\bfs,a)}\int_{\bfs'}\big(\mathcal{P}^\star(d\bfs) -\mathcal{P}_{t}(d\bfs)\big)\pi^\star(da|\bfs) \mathcal{P}(d\bfs'|\bfs,a)\Big|\\
    %         &+\Big|\int_{(\bfs,a)}\int_{\bfs'} \mathcal{P}_{t}(d\bfs)\big(\pi_{t}(da|d\bfs) - \pi^\star(da|d\bfs)\big) \mathcal{P}(d\bfs'|\bfs,a)\Big|.\\
    %         \end{split}
    %         \end{equation}
    Then, we have 
    \begin{equation}
        \begin{split}
        &\Big|\int_{(\bfs,a)}\int_{\bfs'}\big(\mu^\star(d\bfs,da) \mathcal{P}(d\bfs'|\bfs,a) - \mu_{t}(d\bfs,da) \mathcal{P}(d\bfs'|\bfs,a) \big)\Big|\\
            = & \Big|\int_{(\bfs,a)}\int_{\bfs'}\big(\mathcal{P}^\star(d\bfs)\pi^\star(da|\bfs) \mathcal{P}(d\bfs'|\bfs,a) - \mathcal{P}_{t}(d\bfs)\pi_{t}(da|d\bfs) \mathcal{P}(d\bfs'|\bfs,a) \big)\Big|\\
            \le & \Big|\int_{(\bfs,a)}\int_{\bfs'}\big(\mathcal{P}^\star(d\bfs) -\mathcal{P}_{t}(d\bfs)\big)\pi^\star(da|\bfs) \mathcal{P}(d\bfs'|\bfs,a)\Big|\\
            &+\Big|\int_{(\bfs,a)}\int_{\bfs'} \mathcal{P}_{t}(d\bfs)\big(\pi_{t}(da|d\bfs) - \pi^\star(da|d\bfs)\big) \mathcal{P}(d\bfs'|\bfs,a)\Big|.\\
        \end{split}
    \end{equation}
     From Theorem 3.1 in \cite{Yu05}, we know that 
    \begin{equation}
    \begin{gathered}       \Big|\int_{(\bfs,a)}\big(\mathcal{P}^\star(d\bfs) -\mathcal{P}_{t}(d\bfs)\big)\Big| \le |\mathcal{A}|(\log_{\nu}\lambda^{-1}+\frac{1}{1-\nu})C_t\\
        \text{and}\qquad \big\|\pi_{t}(da|d\bfs) - \pi^\star(da|d\bfs)\big\| \le C_t.
    \end{gathered}
    \end{equation}
   Therefore, the bound of $\bfI_3$ can be found as
    \begin{equation}
    \begin{split}
        \|\bfI_3\|_2 \le&~\frac{R_{\max}}{1-\gamma} \cdot |\mathcal{A}|\cdot C_t \cdot (1 + \log_\nu \lambda^{-1} +\frac{1}{1-\nu})\\
        = & |\mathcal{A}|\cdot \frac{R_{\max}}{1-\gamma} \cdot  (1+\log_\nu \lambda^{-1}+\frac{1}{1-\nu})\cdot C_t.
    \end{split}
    \end{equation}
    
\end{proof}

 \subsubsection{Proof of bound of $I_4$}
 \begin{proof}
   We have 
    \begin{equation}
    \begin{split}
        \|\bfI_4\| 
        =&  \|\Delta \hkt(\theta^{(t)};\mathcal{X}_m)\|_2\\
        =&\max_{\bfs,a} \gamma \cdot \Big(\max_{a^\prime}\psi(\bfs_m^\prime,a^\prime;\theta^\star) -  \psi(\bfs_m^\prime,a_m^\prime;\theta^{(t)}) \Big)\cdot \Big\|\frac{\partial \psi(\theta^{(t)};\mathcal{X}_m)}{\partial \theta_{\ell,k}}\Big\|_2\\
        \le & \max_{\bfs,a} \gamma \cdot \Big(\max_{a^\prime}\psi(\bfs_m^\prime,a^\prime;\theta^\star) -  \max_{a^\prime}\psi(\bfs_m^\prime,a^\prime;\theta^{(t)}) \Big)\cdot \Big\|\frac{\partial \psi(\theta^{(t)};\mathcal{X}_m)}{\partial \theta_{\ell,k}}\Big\|_2\\
        \le &\gamma \cdot  \max_{s,a,a'} \Big|\psi(\bfs_m^\prime,a^\prime;\theta^\star) -  \psi(\bfs_m^\prime,a^\prime;\theta^{(t)}) \Big|\cdot \Big\|\frac{\partial \psi(\theta^{(t)};\mathcal{X}_m)}{\partial \theta_{\ell,k}}\Big\|_2\\
        \lesssim & \gamma \cdot \|\theta^{(t)} -\theta^\star\|_2 \cdot \frac{1}{K_\ell}\\
        \le & \frac{\gamma}{K_\ell} \|\theta^{(t)} -\theta^\star\|_2.
    \end{split}
    \end{equation}
\end{proof}

\subsection{Proof of Leamma \ref{lemma:convergence_of_w}}
\begin{proof}[Proof of Lemma \ref{lemma:convergence_of_w}]
    From the update rule of $\bfw$ in Algorithm \ref{Alg}, we have 
    \begin{equation}
    \begin{split}
        \bfw^{(t+1)} -\bfw^\star 
        = & \bfw^{(t)} -\bfw^\star - \kappa_t \cdot \sum_{m\in\mathcal{D}_t}(\phi_m^\top \bfw^{(t)} -r_m) \cdot \phi_m\\
        =& \bfw^{(t)} -\bfw^\star - \kappa_t \cdot \sum_{m\in\mathcal{D}_t}(\phi_m^\top \bfw^{(t)} -\phi_m\bfw^\star) \cdot \phi_m\\
        =&\Big(\bfI- \kappa_t \sum_{m\in\mathcal{D}_m} \phi_m^\top\phi_m \Big)\cdot (\bfw^{(t)}-\bfw^\star).
    \end{split}
    \end{equation}
        % Let $\rho_2$ denote the minimal eigenvalue of $\mathbb{E}_{\mathcal{D}_t} \phi^\top\phi$. 
    For any unit vector $\alpha \in dim(\bfw)$, we have 
    \begin{equation}
        \begin{gathered}
            |\bfal^\top \mathbb{E}_{\mathcal{D}_t} \phi^\top\phi \bfal|\le \max_{\|\phi\|_2}|\bfal^\top\phi|^2 \le \phi_{\max}^2,\\
            |\bfal^\top \mathbb{E}_{\mathcal{D}_t} \phi^\top\phi \bfal|\ge |\bfal^\top\phi_{\min}|^2 \ge 0.
        \end{gathered}
    \end{equation}
    Also, it is easy to verify that $|\bfal^\top \mathbb{E}_{\mathcal{D}_t} \phi^\top\phi \bfal| =0$ if only and if $\phi_m$ are all parallel to each other. As $\phi_m$ does not parallel to each other, let $\rho_2>0$ denote the minimal eigenvalue of $\mathbb{E}_{\mathcal{D}_t} \phi^\top\phi$.

    Given $\phi$ is bounded, $\phi$ belongs to the sub-Gaussian distribution. Similar to \eqref{eqn: tempppp}, with Chebyshev's inequality, we have 
    \begin{equation}
        \left\|\sum_{m\in\mathcal{D}_m} \phi_m^\top\phi_m - \mathbb{E}_{\mathcal{D}_t}\phi^\top \phi \right\|_2\le \sqrt{\frac{d\log q}{|\mathcal{D}_t|}}
    \end{equation}
    with probability at least $1-d^{-q}$.
    Let $N \ge c_N^{-2} d\log q$,  according to Lemma \ref{Lemma: weyl}, we have 
    \begin{equation}
        \lambda_{\min}(\sum_{m\in\mathcal{D}_m} \phi_m^\top\phi_m) \le \lambda_{\min} (\mathbb{E}_{\mathcal{D}_t}\phi^\top \phi) - c_N \le \rho_2-c_N.
    \end{equation}
    
    When we choose $\kappa_t=\frac{1}{\phi_{\max}}$, we have 
    \begin{equation}
        \begin{split}
            \|\bfw^{(t+1)} -\bfw^\star \|_2 
            \le& \Big(1-\frac{\rho_2 -c_N}{\phi_{\max}}\Big)\cdot \|\bfw^{(t)}-\bfw^\star\|_2.\\
        \end{split}
    \end{equation}
\end{proof}

\section{Proof of lemmas in Appendix \ref{app: thm34}}\label{app: GPI}

\begin{proof}[Proof of Lemma \ref{lemma: DQN_difference}]
$\big|Q_i^{\pi_i^\star}(\bfs,a) - Q_i^{\pi_j^\star}(\bfs,a)\big|$ can be upper bounded as
    \begin{equation}
        \begin{split}
            &\big|Q_i^{\pi_i^\star}(\bfs,a) - Q_i^{\pi_j^\star}(\bfs,a)\big|\\
            = & \Big| r_i +\gamma \cdot \sum_{\bfs'} p_{\bfs,\bfs'}^a Q_{i}^{\pi_i^\star}\big(\bfs',\pi_i^\star(\bfs')\big) - \Big( r_i +\gamma \cdot \sum_{\bfs'} p_{\bfs,\bfs'}^a Q_{i}^{\pi_j^\star}\big(\bfs',\pi_j^\star(\bfs')\big) \Big)\Big|\\
             = & \gamma\cdot \Big| \sum_{\bfs'} p_{\bfs,\bfs'}^a Q_{i}^{\pi_i^\star}\big(\bfs',\pi_i^\star(\bfs')\big)  
            -
            \sum_{\bfs'} p_{\bfs,\bfs'}^a Q_{i}^{\pi_j^\star}\big(\bfs',\pi_j^\star(\bfs')\big)\Big|\\
            \le &\gamma\cdot \sum_{\bfs'} p_{\bfs,\bfs'}^a \cdot 
            \Big| Q_{i}^{\pi_i^\star}\big(\bfs',\pi_i^\star(\bfs')\big)- Q_{i}^{\pi_j^\star}\big(\bfs',\pi_j^\star(\bfs')\big)\Big|\\
             \le &\gamma\cdot \sum_{\bfs'} p_{\bfs,\bfs'}^a \cdot 
             \Big[\Big| Q_{i}^{\pi_i^\star}\big(\bfs',\pi_i^\star(\bfs')\big)- Q_{j}^{\pi_j^\star}\big(\bfs',\pi_j^\star(\bfs')\big)\Big|
             +\Big| Q_{j}^{\pi_j^\star}\big(\bfs',\pi_j^\star(\bfs')\big)- Q_{i}^{\pi_j^\star}\big(\bfs',\pi_j^\star(\bfs')\big)\Big|\Big]\\
              = &\gamma\cdot \sum_{\bfs'} p_{\bfs,\bfs'}^a \cdot\Big[ 
             \Big| \max_{a'}Q_{i}^{\pi_i^\star}\big(\bfs',a'\big)- \max_{a'}Q_{j}^{\pi_j^\star}\big(\bfs',a'\big)\Big|
             +\Big| Q_{j}^{\pi_j^\star}\big(\bfs',\pi_j^\star(\bfs')\big)- Q_{i}^{\pi_j^\star}\big(\bfs',\pi_j^\star(\bfs')\big)\Big|\Big]\\
             \le &\gamma\cdot \sum_{\bfs'} p_{\bfs,\bfs'}^a \cdot \Big[ 
             \max_{a'} \Big| Q_{i}^{\pi_i^\star}\big(\bfs',a'\big)- Q_{j}^{\pi_j^\star}\big(\bfs',a'\big)\Big|
             +\Big| Q_{j}^{\pi_j^\star}\big(\bfs',\pi_j^\star(\bfs')\big)- Q_{i}^{\pi_j^\star}\big(\bfs',\pi_j^\star(\bfs')\big)\Big|\Big]\\
             \le &\gamma\cdot \sum_{\bfs'} p_{\bfs,\bfs'}^a \cdot \Big[ 
             \max_{\bfs',a'} \Big| Q_{i}^{\pi_i^\star}\big(\bfs',a'\big)- Q_{j}^{\pi_j^\star}\big(\bfs',a'\big)\Big|
             +\max_{\bfs'}\Big| Q_{j}^{\pi_j^\star}\big(\bfs',\pi_j^\star(\bfs')\big)- Q_{i}^{\pi_j^\star}\big(\bfs',\pi_j^\star(\bfs')\big)\Big|\Big]\\
        \end{split}
    \end{equation}    
    Let $$I_5 = \max_{\bfs,a} \Big| Q_{i}^{\pi_i^\star}\big(\bfs,a\big)- Q_{j}^{\pi_j^\star}\big(\bfs,a\big)\Big|$$ and $$I_6 =\max_{\bfs,a}\Big| Q_{j}^{\pi_j^\star}\big(\bfs,a\big)- Q_{i}^{\pi_j^\star}\big(\bfs,a\big)\Big| \ge \max_{\bfs}\Big| Q_{j}^{\pi_j^\star}\big(\bfs,\pi_j^\star(\bfs)\big)- Q_{i}^{\pi_j^\star}\big(\bfs,\pi_j^\star(\bfs)\big)\Big|.$$
    Then, we have 
    \begin{equation}\label{eqn:temp8_1}
    \begin{split}
        I_5 
        =& \max_{\bfs,a} \Big|  r_i + \gamma\cdot \sum_{\bfs'} p_{\bfs,\bfs'}^a \max_{a'} Q_i^{\pi_i^\star}(\bfs',a') - r_j - \gamma\cdot \sum_{\bfs'} p_{\bfs,\bfs'}^a \max_{a'} Q_j^{\pi_j^\star}(\bfs',a')  \Big|\\
        \le & \max_{\bfs,a} |r_i(\bfs,a)-r_j(\bfs,a)| + \gamma \max_{\bfs,a}\sum_{\bfs'} p_{\bfs,\bfs'}^a\cdot \max_{a'}|Q_i^{\pi_i^\star}(\bfs',a')-Q_j^{\pi_j^\star}(\bfs',a')|\\
        \le & \max_{\bfs,a} |r_i(\bfs,a)-r_j(\bfs,a)| + \gamma \cdot I_5.
    \end{split}
    \end{equation}
    Therefore, we have 
    \begin{equation}
        I_5\le \frac{1}{1-\gamma}  \max_{\bfs,a} |r_i(\bfs,a)-r_j(\bfs,a)|.
    \end{equation}
    Similar to \eqref{eqn:temp8_1}, we have 
    \begin{equation}
    \begin{split}
        I_6 
        \le & \max_{\bfs,a} |r_i(\bfs,a)-r_j(\bfs,a)| + \gamma \max_{\bfs,a}\sum_{\bfs'} p_{\bfs,\bfs'}^a\cdot |Q_j^{\pi_j^\star}(\bfs',\pi_j^\star(\bfs'))-Q_i^{\pi_j^\star}(\bfs',\pi_j^\star(\bfs'))|\\
        \le & \max_{\bfs,a} |r_i(\bfs,a)-r_j(\bfs,a)| + \gamma \cdot I_6.  
    \end{split}
    \end{equation}
    Therefore, we have \begin{equation}
        I_6\le \frac{1}{1-\gamma}  \max_{\bfs,a} |r_i(\bfs,a)-r_j(\bfs,a)|.
    \end{equation}
    Therefore, we have 
    \begin{equation}
        \big|Q_i^{\pi_i^\star}(\bfs,a) - Q_j^{\pi_i^\star}(\bfs,a)\big| \le \gamma(I_5+I_6) \le \frac{2\gamma}{1-\gamma} \cdot  \max_{\bfs,a} |r_i(\bfs,a)-r_j(\bfs,a)|.
    \end{equation}
\end{proof}

% \begin{proof}
%     We have 
%     \begin{equation}
%     \begin{split}
%         |Q_i^\star - Q_j^\star| 
%         = & |\psi_i^\star \bfw_i^\star - \psi_j^\star \bfw_j^\star|\\
%         \le&|\psi_i^\star \bfw_i^\star - \psi_j^\star \bfw_i^\star| + |\psi_j^\star \bfw_i^\star - \psi_j^\star \bfw_j^\star|\\
%         \le & \frac{2\gamma}{1-\gamma}\phi_{\max}\cdot \|\bfw_i^\star -\bfw_j^\star\|_2 + \phi_{\max}\cdot \|\bfw_i^\star -\bfw_j^\star\|_2\\
%         = & \frac{1+\gamma}{1-\gamma}\phi_{\max}\cdot \|\bfw_i^\star -\bfw_j^\star\|_2.
%     \end{split}
%     \end{equation}
% \end{proof}

\section{Additional proof of the lemmas}\label{app: proof}

\subsection{Proof of Lemma \ref{Lemma: distance_Second_order_distance}}
The distance of the second order derivatives of the population risk function $f(\cdot)$ at point $\theta$ and $\theta^\star$ can be converted into bounding  $\bfP_1$, $\bfP_2$, which are defined in \eqref{eqn: Lemma11_main}. The major idea in proving $\bfP_1$ is to connect the error bound to the angle between $\theta$ and $\theta^\star$ given $\bfh^{(\ell)}$ belongs to the sub-Gaussian distribution.
\begin{proof}[Proof of Lemma \ref{Lemma: distance_Second_order_distance}]
	From the definition of $f$ in \eqref{eqn:prf}, we have 
	\begin{equation}
		\begin{gathered}
		\frac{\partial^2 f}{\partial {\theta}_{\ell, j_1}\partial {\theta}_{\ell, j_2}}(\theta^\star)
		=\frac{1}{K^2} \mathbb{E}_{\bfx} \cJ_{\ell,k}\sigma^{\prime} (\theta^{\star \top}_{j_1}\bfh) \cdot \cJ_{\ell,k}\sigma^{\prime} (\theta^{\star \top}_{j_2}\bfh)\cdot \bfh^\star \bfh^{\star\top},\\
	\text{and \quad}  \frac{\partial^2 f}{\partial {\theta}_{\ell, j_1}\partial {\theta}_{\ell, j_2}}(\theta)
	=\frac{1}{K^2} \mathbb{E}_{\bfx} 
 \sigma^{\prime} \cJ_{\ell,k}^\star({\theta}_{\ell, j_1}^\top\bfh)
 \cdot \cJ_{\ell,k}^\star\sigma^{\prime} ({\theta}_{\ell, j_2}^\top\bfh) 
 \cdot \bfh \bfh^{\top},
	\end{gathered}
	\end{equation}
	where $\bfh =\bfh^{(\ell)}(\theta)$ and $\bfh^\star =\bfh^{(\ell)}(\theta^\star)$.
 % $\theta_j$ is the $j$-th column of $\theta$ and {${\theta}^{\star }_j$} is the $j$-th column of $\theta^\star$, respectively.
 
	Then, we have 
	\begin{equation}\label{eqn: Lemma11_main}
	\begin{split}
	&\frac{\partial^2 f}{\partial {\theta}_{\ell, j_1}\partial {\theta}_{\ell, j_2}}(\theta^*)
	-
	\frac{\partial^2 f}{\partial {\theta}_{\ell, j_1}\partial {\theta}_{\ell, j_2}}(\theta)\\
	=&\frac{1}{K^2} \mathbb{E}_{\bfx} \big[
	{\cJ_{\ell,k}^\star\sigma^{\prime} (\theta^{\star T}_{\ell,j_1}\bfh^\star)}\cJ_{\ell,k}^\star\sigma^{\prime} (\theta^{\star T}_{\ell,j_2}\bfh^\star)\bfh^\star \bfh^{\star\top}\\
	&\qquad \qquad -
	\cJ_{\ell,k}\sigma^{\prime} ({\theta}_{\ell, j_1}^\top\bfh)\cJ_{\ell,k}\cJ_{\ell,k}\sigma^{\prime} ({\theta}_{\ell, j_2}^\top\bfh)\bfh \bfh^\top\big] \\
	=&\frac{1}{K^2}\mathbb{E}_{\bfx} 
	\big[ \cJ_{\ell,k}^\star\sigma^{\prime} (\theta^{\star T}_{\ell,j_1}\bfh^\star)
	\big(\cJ_{\ell,k}^\star\sigma^{\prime} (\theta^{\star T}_{\ell,j_2}\bfh^\star)\bfh^\star\bfh^{\star\top}
	- \cJ_{\ell,k}\sigma^{\prime} ({\theta}_{\ell, j_2}^\top\bfh)\bfh\bfh^\top
	\big)\\
	&\qquad +
	\cJ_{\ell,k}\sigma^{\prime} ({\theta}_{\ell, j_2}^\top\bfh)
	\big(\cJ_{\ell,k}^\star\sigma^{\prime} (\theta^{\star T}_{\ell,j_1}\bfh)\bfh^\star\bfh^{\star\top}
	- \cJ_{\ell,k}\sigma^{\prime} ({\theta}_{\ell, j_1}^\top\bfh)\bfh\bfh^{\top}
	\big)
	\big]\\
	% =&\frac{1}{K^2}\big[
	% \mathbb{E}_{\bfx} 
	% \cJ_{\ell,k}\sigma^{\prime} (\theta^{\star T}_{\ell,j_1}\bfh)
	% \big(\cJ_{\ell,k}\sigma^{\prime} (\theta^{\star T}_{\ell,j_2}\bfh)
	% - \cJ_{\ell,k}\sigma^{\prime} ({\theta}_{\ell, j_2}^\top\bfh)
	% \big)\bfh \bfh^\top\\
	% &\qquad +
	% \mathbb{E}_{\bfx}
	% \cJ_{\ell,k}\sigma^{\prime} ({\theta}_{\ell, j_2}^\top\bfh)
	% \big(\cJ_{\ell,k}\sigma^{\prime} (\theta^{\star T}_{j_1}\bfh)
	% - \cJ_{\ell,k}\sigma^{\prime} ({\theta}_{\ell, j_1}^\top\bfh)
	% \big)\bfh \bfh^\top
	% \big]\\
	:=& \frac{1}{K^2} (\bfP_1 + \bfP_2).
	\end{split}
	\end{equation}
	For any $\bfa\in\mathbb{R}^{K_{\ell}}$ with {$\|\bfa\|_2=1$}, we have 
	\begin{equation}
	\begin{split}
	\bfa^\top\bfP_1 \bfa 
	=&\mathbb{E}_{\bfx} 
	\cJ_{\ell,k}^\star\sigma^{\prime} (\theta^{\star T}_{\ell,j_1}\bfh^\star)
	\Big(\cJ_{\ell,k}^\star\sigma^{\prime} (\theta^{\star T}_{\ell,j_2}\bfh^\star) (\bfa^\top\bfh^\star)^2
	- \cJ_{\ell,k}\sigma^{\prime} ({\theta}_{\ell, j_2}^\top\bfh)(\bfa^\top\bfh)^2
	\Big).
% 	\le&\mathbb{E}_{\bfh} 
% 	\sigma^{\prime} (\theta^{*T}_{j_2}\bfh)
% 	\big(\sigma^{\prime} (\theta^{\star T}_{j_2}\bfh)
% 	- \sigma^{\prime} (\theta_{\ell, j_2}^\top\bfh)
% 	\big) 
% 	\cdot
% 	(\bfa^\top\bfh)^2,	
	\end{split}
	\end{equation}
	 % Let $I =
  %   \sigma^{\prime} (\theta^{\star T}_{j_1}\bfh)
  %   \big(\sigma^{\prime} (\theta^{\star T}_{j_2}\bfh)
  %   - \sigma^{\prime} (\theta_{\ell, j_2}^\top\bfh)
  %   \big) 
  %   \cdot
  %   (\bfa^\top\bfh)^2$. 
    Then, we have 
    \begin{equation}
        \begin{split}
            |\bfa^\top\bfP_1 \bfa| 
	=&\Big|\mathbb{E}_{\bfx} 
	\cJ_{\ell,k}^\star\sigma^{\prime} (\theta^{\star T}_{\ell,j_1}\bfh^\star)
	\Big(\cJ_{\ell,k}^\star\sigma^{\prime} (\theta^{\star T}_{\ell,j_2}\bfh^\star) (\bfa^\top\bfh^\star)^2
	- \cJ_{\ell,k}\sigma^{\prime} ({\theta}_{\ell, j_2}^\top\bfh)(\bfa^\top\bfh)^2
	\Big)\Big|\\
 \le & \mathbb{E}_{\bfx}
 \Big|\cJ_{\ell,k}^\star\sigma^{\prime} (\theta^{\star T}_{\ell,j_2}\bfh^\star) (\bfa^\top\bfh^\star)^2
	- \cJ_{\ell,k}\sigma^{\prime} ({\theta}_{\ell, j_2}^\top\bfh)(\bfa^\top\bfh)^2\Big|\\
 \le& \mathbb{E}_{\bfx}
 \Big|\cJ_{\ell,k}^\star\sigma^{\prime} (\theta^{\star T}_{\ell,j_2}\bfh^\star) (\bfa^\top\bfh^\star)^2
	- \cJ_{\ell,k}^\star\sigma^{\prime} ({\theta}_{\ell, j_2}^{\star\top}\bfh^\star)(\bfa^\top\bfh)^2\Big|\\
  &+ \mathbb{E}_{\bfx}
 \Big|\cJ_{\ell,k}^\star\sigma^{\prime} ({\theta}_{\ell, j_2}^{\star\top}\bfh^\star)(\bfa^\top\bfh)^2
	- \cJ_{\ell,k}\sigma^{\prime} ({\theta}_{\ell, j_2}^{\star\top}\bfh)(\bfa^\top\bfh)^2\Big|\\
 &+ \mathbb{E}_{\bfx}
 \Big| \cJ_{\ell,k}\sigma^{\prime} ({\theta}_{\ell, j_2}^{\star\top}\bfh)(\bfa^\top\bfh)^2- \cJ_{\ell,k}\sigma^{\prime} ({\theta}_{\ell, j_2}^\top\bfh)(\bfa^\top\bfh)^2\Big|\\
 \lesssim& \|\theta-\theta^\star\|_2 + \|\theta-\theta^\star\|_2
 +\mathbb{E}_\bfx \Big|\big(\sigma^{\prime} ({\theta}_{\ell, j_2}^{\star\top}\bfh)
	- \sigma^{\prime} ({\theta}_{\ell, j_2}^{\star\top}\bfh)\big)\cdot(\bfa^\top\bfh)^2\Big|\\
 \lesssim& \|\theta-\theta^\star\|_2 +\mathbb{E}_\bfx \Big|\big(\sigma^{\prime} ({\theta}_{\ell, j_2}^{\star\top}\bfh)
	- \sigma^{\prime} ({\theta}_{\ell, j_2}^{\star\top}\bfh)\big)\cdot(\bfa^\top\bfh)^2\Big|.
        \end{split}
    \end{equation}
    % \begin{equation}
    % \begin{split}
    %     |\bfa^\top \bfP_1\bfa| 
    %     \le& \mathbb{E}_{\bfx} \Big|\sum_{i=1}^{L}\big(\sigma^\prime(\theta^{\star\top}_{\ell,j_2}\bfh^{(i)}) - \sigma^\prime(\theta^{\top}_{\ell,j_2}\bfh^{(i)})\big) \cdot \big(\bfa^\top\bfh^{(i)}\big)\Big|\\
    %     \le &\sum_{i=1}^{L}\mathbb{E}_{\bfx} \Big|\big(\sigma^\prime(\theta^{\star\top}_{\ell,j_2}\bfh^{(i)}) - \sigma^\prime(\theta^{\top}_{\ell,j_2}\bfh^{(i)})\big) \cdot \big(\bfa^\top\bfh^{(i)}\big)\Big|\\
    %     \le & \sum_{i=1}^{L}\mathbb{E}_{\bfx} \Big|\big(\sigma^\prime(\theta^{\star\top}_{\ell,j_2}\bfh^{(i)}) - \sigma^\prime(\theta^{\top}_{\ell,j_2}\bfh^{(i)})\big) \cdot \big(\bfa^\top\bfh^{(i)}\big)\Big|
    % \end{split}
    % \end{equation}
    Utilizing the Gram-Schmidt process, we can demonstrate the existence of a set of normalized orthonormal vectors denoted as $\mathcal{B}=\{\bfa, \bfb, \bfc, \bfa_4^{\perp}, \cdots, \bfa_d^{\perp}\} \in \mathbb{R}^{d}$. This set forms an orthogonal and normalized basis for $\mathbb{R}^{d}$, wherein the subspace spanned by ${\bfa, \bfb, \bfc }$ includes $\bfa, \theta_{\ell, j_2}$, and $\theta_{\ell, j_2}^*$. 
    Then, for any $\bfx\in\mathbb{R}^{d}$, we have a unique $\bfz=[
		z_1, ~ z_2,~\cdots, ~ z_d]^\top$ such that 
		\begin{equation*}
		\bfh = z_1\bfa + z_2 \bfb +z_3\bfc + \cdots + z_d\bfa^{\perp}_d.
		\end{equation*}
        Because (i) $\bfa, \theta_{\ell, j_2}$, and $\theta_{\ell, j_2}^*$ belongs to the subspace spanned by vectors $\{\bfa,\bfb,\bfc\}$ and (ii) $\bfa_4^\perp,\cdots, \bfa_d^\perp,\cdots$ are orthogonal to $\bfa,\bfb,$ and $\bfc$. Then, we know that 
        \begin{equation}\label{eqn: temp7_l}
        \begin{split}
            \theta^{\star\top}_{\ell,j_2}\bfh
            = & \theta^{\star\top}_{\ell,j_2}(z_1\bfa + z_2 \bfb +z_3\bfc + \cdots + z_d\bfa^{\perp}_d)\\
            =&  z_1\theta^{\star\top}_{\ell,j_2}\bfa + z_2\theta^{\star\top}_{\ell,j_2} \bfb +z_3\theta^{\star\top}_{\ell,j_2}\bfc + \cdots + z_d\theta^{\star\top}_{\ell,j_2}\bfa^{\perp}_d\\
            =& z_1\theta^{\star\top}_{\ell,j_2}\bfa + z_2\theta^{\star\top}_{\ell,j_2} \bfb +z_3\theta^{\star\top}_{\ell,j_2}\bfc + 0\\
            =&\theta^{\star\top}_{\ell,j_2}(z_1\bfa + z_2 \bfb +z_3\bfc)\\
            : =& \theta^{\star\top}_{\ell,j_2} \widetilde{\bfh}.
        \end{split}
        \end{equation}
        where $\widetilde{\bfh}= z_1\bfa + z_2\bfb + z_3\bfc.$ 
        Similar to \eqref{eqn: temp7_l}, we have $\theta^{\top}_{\ell,j_2}\bfh = \theta^{\top}_{\ell,j_2}\widetilde{\bfh}$ and $\bfa^\top\bfh = \bfa^\top\widetilde{\bfh}$.
        
        {Then, we define $I_4$ as
		\begin{equation}
		\begin{split}
		I_4
		:=& \mathbb{E}_{\bfh} \Big|\big(\sigma^\prime(\theta^{\star\top}_{\ell,j_2}\bfh) - \sigma^\prime(\theta^{\top}_{\ell,j_2}\bfh)\big) \cdot \big(\bfa^\top\bfh\big)\Big|\\
        =& \int_{\mathcal{R}_\bfh}|\sigma^{\prime}\big(\theta_{\ell, j_2}^\top{\bfh}\big)-\sigma^{\prime}\big({\theta^{\star T}_{\ell,j_2}}{\bfh}\big)|\cdot 
		|\bfa^\top{\bfh}|^2 \cdot f_H(\bfh)d\bfh\\
        =&\int_{\mathcal{R}_\bfz}|\sigma^{\prime}\big(\theta_{\ell, j_2}^\top{\bfh}\big)-\sigma^{\prime}\big({\theta^{\star T}_{\ell,j_2}}{\bfh}\big)|\cdot 
		|\bfa^\top{\bfh}|^2 \cdot f_Z(\bfz)\cdot |\bfJ_{\bfh}(\bfz)|d\bfz\\
		\end{split}
		\end{equation}
        where $|\bfJ_{\bfh}(\bfz)|$ is the determinant of the Jacobian matrix $\frac{\partial \bfh}{\partial \bfz}$. Since $\bfz$ is a representation of $\bfh$ based on an orthogonal and normalized basis, we have $|\bfJ_{\bfh}(\bfz)|=1$.
        According to \eqref{eqn: temp7_l}, $I_4$ can be rewritten as
        \begin{equation}
        \begin{split}
            I_4
            =&\int_{\mathcal{R}_z}|\sigma^{\prime}\big(\theta_{\ell, j_2}^\top\widetilde{\bfh}\big)-\sigma^{\prime}\big({\theta^{\star T}_{\ell,j_2}}\widetilde{\bfh}\big)|\cdot 
		|\bfa^\top\widetilde{\bfh}|^2 \cdot f_Z(\bfz)d\bfz\\
        =& \int_{\mathcal{R}_z}|\sigma^{\prime}\big(\theta_{\ell, j_2}^\top\widetilde{\bfh}\big)-\sigma^{\prime}\big({\theta^{\star T}_{\ell,j_2}}\widetilde{\bfh}\big)|\cdot 
		|\bfa^\top\widetilde{\bfh}|^2 \cdot f_Z(z_1, z_2, z_3)dz_1 dz_2 dz_3
        \end{split}
        \end{equation}
        where in the last equality we abuse $f_Z(z_1, z_2,z_3)$ to represent the probability density function of $(z_1, z_2, z_3)$ defined in region $\mathcal{R}_z$.}

        Next, we show that $\bfz$ is rotational invariant over $\mathcal{R}_z$.
        Let $\bfR = [\bfa~\bfb~\bfc~\cdots~\bfa_d^\perp]$, we have 
            $\bfh = \bfR \bfz$.  For any $\bfz^{(1)}$ and $\bfz^{(2)}$ with $\|\bfz^{(1)}\|_2 =\|\bfz^{(2)}\|_2$. We define  $\bfh^{(1)} = R\bfz^{(1)}$ and $\bfh^{(2)} = R\bfz^{(2)}$. Since $\bfx$ is rotational invariant and $\|\bfh^{(1)}\|_2 = \|\bfh^{(2)}\|_2 =\|\bfz^{(1)}\|_2 = \|\bfz^{(2)}\|_2$, then we know $\bfh^{(1)}$ and $\bfh^{(2)}$ has the same distribution density. Then, $\bfz^{(1)}$  and $\bfz^{(2)}$ has the same distribution density as well. Therefore, $\bfz$ is rotational invariant over $\mathcal{R}_z$.
  
        Then, we consider spherical coordinates with $z_1= Rcos\sigma_1, z_2 = Rsin\sigma_1sin\sigma_2, {z_3= Rsin\sigma_1cos\sigma_2}$. 
        Hence, we have
		\begin{equation}
		\begin{split}
		I_4
		=&\int|\sigma^{\prime}\big(\theta_{\ell, j_2}^\top{\widetilde{\bfh}}\big)-\sigma^{\prime}\big(\theta^{\star \top}_{\ell,j_2}{\widetilde{\bfh}}\big)|\cdot |R\cos\sigma_1|^2
		\cdot f_Z(R, \sigma_1, \sigma_2)\cdot R^2\sin \sigma_1 \cdot dR d\sigma_1 d \sigma_2.
		\end{split}
		\end{equation}
        Since $\bfz$ is rotational invariant, we have that 
        \begin{equation}
            f_Z(R, \sigma_1, \sigma_2) = f_Z(R).    
        \end{equation}
		% In addition, it is easy to verify that $\sigma^{\prime}\big(\theta_{\ell, j_2}^\top{\bfh}\big)$ only depends on the direction of ${\bfh}$. 
		% {\begin{equation*}
		% f_Z(R, \sigma_1, \sigma_2) = \frac{1}{(2\pi)^{\frac{3}{2}}}e^{-\frac{z_1^2+z_2^2+z_3^2}{2}}=\frac{1}{(2\pi)^{\frac{3}{2}}}e^{-\frac{R^2}{2}}
		% \end{equation*}}
		% only depends on $R$. 
  
        Then, we have 
		\begin{equation}\label{eqn: ean111}
		\begin{split}
		I_4
		=&\int|\sigma^{\prime}\big(\theta_{\ell, j_2}^\top({\widetilde{\bfh}}/R)\big)-\sigma^{\prime}\big({\theta^{\star T}_{\ell,j_2}}({\widetilde{\bfh}}/R)\big)|
		\cdot |R\cos\sigma_1|^2
		\cdot f_Z(R)R^2\sin \sigma_1 dRd\sigma_1d\sigma_2\\
		=& \int_0^{\infty} R^4f_z(R)dR\int_{0}^{\psi_1(R)}\int_{0}^{\psi_2(R)}|\cos \sigma_1|^2\cdot \sin\sigma_1\\
		&\cdot|\sigma^{\prime}\big(\theta_{\ell, j_2}^\top({\widetilde{\bfh}}/R)\big)-\sigma^{\prime}\big({\theta^{\star T}_{\ell,j_2}}({\widetilde{\bfh}}/R)\big)|d\sigma_1d\sigma_2\\
        \le & \int_0^{\infty} R^4f_z(R)dR\int_{0}^{\pi}\int_{0}^{2\pi}\sin\sigma_1
		\cdot|\sigma^{\prime}\big(\theta_{\ell, j_2}^\top\bar{\bfx}\big)-\sigma^{\prime}\big({\theta^{\star T}_{\ell,j_2}}\bar{\bfx}\big)|d\sigma_1d\sigma_2,
		\end{split}
		\end{equation}
    where 
    the first equality holds because $\sigma^{\prime}\big(\theta_{i,, j_2}^\top{\bfh}\big)$ only depends on the direction of ${\bfh}$, and $\bar{\bfx}:= {\bfh}/{R}= (\cos\sigma_1, \sin\sigma_1\sin\sigma_2, \sin\sigma_1\cos\sigma_2)$ in the last inequality. 
    % the last inequality comes from $|cos\sigma_1|\le 1$ and .
  %    It is clear that $\sin\sigma_1
		% \cdot|\sigma^{\prime}\big(\theta_{\ell, j_2}^\top\bar{\bfx}\big)-\sigma^{\prime}\big({\theta^{\star T}_{j_2}}\bar{\bfx}\big)|d\sigma_1d\sigma_2$ is independent of $R$. 
    
    Because $\bfz$ belongs to the sub-Gaussian distribution, we have $F_z(R)\ge 1-2 e^{-\frac{R^2}{\sigma^2}}$ for some constant $\sigma>0$. Then, the integration of $R$ can be represented as
    \begin{equation}
    \begin{split}
        \int_{0}^{\infty} R^4 f_Z(R) dR
         =&  \int_{0}^{\infty}  R^4d\big(1-F_z(R)\big)\\
         \le&  \int_{0}^\infty 4R^3\big(1-F_z(R)\big)dR\\
         \le & \int_{0}^\infty 8R^3e^{-\frac{R^2}{\sigma^2}} dR\\
         \le & \frac{32}{\sqrt{2\pi}}\sigma \int_{0}^\infty R^2 e^{-\frac{R^2}{\sigma^2}} dR\\
         = & 32\sigma^2 \int_{0}^\infty R^2 \frac{1}{\sqrt{2\pi\sigma^2}}e^{-\frac{R^2}{\sigma^2}} dR,
    \end{split}
    \end{equation}
    where the last inequality comes from the calculation that 
    \begin{equation}
    \begin{gathered}
        \int_{0}^\infty 2R^2e^{-\frac{R^2}{\sigma^2}} dR = \sqrt{2\pi}\sigma^3,\\
        \int_{0}^\infty 2R^3e^{-\frac{R^2}{\sigma^2}} dR = 4\sigma^4.
    \end{gathered}
    \end{equation}
    Then, we define $\widetilde{\bfx}\in\mathbb{R}^{K_\ell}$ belongs to Gaussian distribution as $\widetilde{\bfx} \sim \mathcal{N}(\bfzero, \sigma^2\bfI)$.  Therefore, we have 
    \begin{equation}
    \begin{split}
        I_4
		\le&~32\sigma^2 \cdot \int_0^{\infty} R^2 \frac{1}{\sqrt{2\pi\sigma^2}}e^{-\frac{R^2}{\sigma^2}} dR\int_{0}^{\pi}\int_{0}^{2\pi} \sin\sigma_1 
		\cdot|\sigma^{\prime}\big(\theta_{\ell, j_2}^\top\bar{\bfx}\big)-\sigma^{\prime}\big({\theta^{\star \top}_{\ell,j_2}}\bar{\bfx}\big)|d\sigma_1d\sigma_2\\
		= &~32\sigma^2\cdot\mathbb{E}_{z_1, z_2, z_3}\big|\sigma^{\prime}\big(\theta_{\ell, j_2}^\top\widetilde{\bfx}\big)-\sigma^{\prime}\big({\theta^{\star \top}_{\ell,j_2}}\widetilde{\bfx}\big)|\\
		\eqsim &~\mathbb{E}_{\widetilde{\bfx}}\big|\sigma^{\prime}\big(\theta_{\ell, j_2}^\top\widetilde{\bfx}\big)-\sigma^{\prime}\big({\theta^{\star T}_{\ell,j_2}}\widetilde{\bfx}\big)|,
    \end{split}
    \end{equation}
    where $\widetilde{\bfx}$ belongs to Gaussian distribution.

    Therefore,  the inequality bound  over a sub-Gaussian distribution is bounded by the one over a Gaussian distribution. In the following contexts, we provide the upper bound of $\mathbb{E}_{\widetilde{\bfx}}\big|\sigma^{\prime}\big(\theta_{\ell, j_2}^\top\widetilde{\bfx}\big)-\sigma^{\prime}\big({\theta^{\star T}_{\ell,j_2}}\widetilde{\bfx}\big)|$.
    
	Define a set $\mathcal{A}_1=\{\bfx| ({\theta^{\star \top}_{\ell,j_2}}\widetilde{\bfx})(\theta_{\ell, j_2}^\top\widetilde{\bfx})<0 \}$. If $\widetilde{\bfx}\in\mathcal{A}_1$, then ${\theta^{\star \top}_{\ell,j_2}}\widetilde{\bfx}$ and $\theta_{\ell, j_2}^\top\widetilde{\bfx}$ have different signs, which means the value of $\sigma^{\prime}(\theta_{\ell, j_2}^\top\widetilde{\bfx})$ and $\sigma^{\prime}({\theta^{\star \top}_{\ell,j_2}}\widetilde{\bfx})$ are different. This is equivalent to say that 
	\begin{equation}\label{eqn:I_1_sub1}
	|\sigma^{\prime}(\theta_{\ell, j_2}^\top\widetilde{\bfx})-\sigma^{\prime}({\theta^{\star \top}_{\ell,j_2}}\widetilde{\bfx})|=
	\begin{cases}
	&1, \text{ if $\widetilde{\bfx}\in\mathcal{A}_1$}\\
	&0, \text{ if $\widetilde{\bfx}\in\mathcal{A}_1^c$}
	\end{cases}.
	\end{equation}
	Moreover, if $\widetilde{\bfx}\in\mathcal{A}_1$, then we have 
	\begin{equation}
	\begin{split}
	|{\theta^{\star T}_{\ell,j_2}}\widetilde{\bfx}|
	\le&|{\theta^{\star T}_{\ell,j_2}}\widetilde{\bfx}-\theta_{\ell, j_2}^\top\widetilde{\bfx}|
	\le\|{\theta^{\star }_{\ell,j_2}}-{\theta_{\ell, j_2}}\|_2\cdot\|\widetilde{\bfx}\|_2.
	\end{split}
	\end{equation}
	Let us define a set $\mathcal{A}_2$ such that
	\begin{equation}
	\begin{split}
	\mathcal{A}_2
	=&\Big\{\widetilde{\bfx}\Big|\frac{|{\theta^{\star T}_{\ell,j_2}}\widetilde{\bfx}|}{\|\theta_{\ell, j_2}^*\|_2\|\widetilde{\bfx}\|_2}\le\frac{\|{\theta_{\ell, j_2}^*}-{\theta_{\ell, j_2}}\|_2}{\|\theta_{\ell, j_2}^*\|_2}   \Big\}\\
	=&\Big\{\theta_{\widetilde{\bfx},\theta_{\ell, j_2}^*}\Big||\cos\theta_{\widetilde{\bfx},\theta^{\star}_{\ell, j_2}}|\le\frac{\|{\theta^{\star}_{\ell, j_2}}-{\theta_{\ell, j_2}}\|_2}{\|\theta^{\star}_{\ell, j_2}\|_2}   \Big\}.
	\end{split}
	\end{equation}	
	Hence, we have that
	\begin{equation}\label{eqn:I_1_sub2}
	\begin{split}
	\mathbb{E}_{\widetilde{\bfx}}
	|\sigma^{\prime}(\theta_{\ell, j_2}^\top\widetilde{\bfx})-\sigma^{\prime}({\theta^{\star T}_{\ell, j_2}}\widetilde{\bfx})|^2
	=& \mathbb{E}_{\widetilde{\bfx}}
	|\sigma^{\prime}(\theta_{\ell, j_2}^\top\widetilde{\bfx})-\sigma^{\prime}({\theta^{\star T}_{\ell,j_2}}\widetilde{\bfx})|\\
	=& \text{Prob}(\widetilde{\bfx}\in\mathcal{A}_1)\\
	\le& \text{Prob}(\widetilde{\bfx}\in\mathcal{A}_2).
	\end{split}
	\end{equation}
	Since $\widetilde{\bfx}\sim \mathcal{N}({\bf0},\|\bfa\|_2^2 \bfI)$, $\theta_{\widetilde{\bfx},\theta^{\star}_{\ell, j_2}}$ belongs to the uniform distribution on $[-\pi, \pi]$, we have
	\begin{equation}\label{eqn:I_1_sub3}
	\begin{split}
	\text{Prob}(\widetilde{\bfx}\in\mathcal{A}_2)
	=\frac{\pi- \arccos\frac{\|{\theta^{\star}_{\ell, j_2}}-{\theta_{\ell, j_2}}\|_2}{\|\theta^{\star}_{\ell, j_2}\|_2} }{\pi}
	\le&\frac{1}{\pi}\tan(\pi- \arccos\frac{\|{\theta^{\star}_{\ell, j_2}}-{\theta_{\ell, j_2}}\|_2}{\|\theta^{\star}_{\ell, j_2}\|_2})\\
	=&\frac{1}{\pi}\cot(\arccos\frac{\|{\theta^{\star}_{\ell, j_2}}-{\theta_{\ell, j_2}}\|_2}{\|\theta^{\star}_{\ell, j_2}\|_2})\\
	\le&\frac{2}{\pi}\frac{\|{\theta^{\star}_{\ell, j_2}}-{\theta_{\ell, j_2}}\|_2}{\|\theta^{\star}_{\ell, j_2}\|_2}\\
    \le & \|\theta^\star_\ell -\theta_\ell \|_2
	\end{split}
	\end{equation}

	Hence,  \eqref{eqn: ean111} and \eqref{eqn:I_1_sub3} suggest that
	\begin{equation}\label{eqn:I_1_bound}
    \begin{gathered}
        I_4\lesssim{\|\theta_{i} -\theta^{\star}_{i}\|_2} \cdot\|\bfa\|_2^2,\\
        \text{and}\qquad  \|\bfP_1\|_2\le \|\theta-\theta^\star\|_2 +I_4\lesssim {\|\theta -\theta^{\star}\|_2},
    \end{gathered}	
	\end{equation}
	The same bound that is shown in \eqref{eqn:I_1_bound} holds for $\bfP_2$ as well.  
	
	Therefore, we have 
	\begin{equation}
	    \begin{split}
	    \| \nabla^2_\ell f(\theta^\star) - \nabla^2_\ell f(\theta)\|_2
	    =&\max_{\|\boldsymbol{\alpha}\|_2\le 1} \Big|\boldsymbol{\alpha}^\top\Big(\nabla^2_\ell f(\theta^\star) - \nabla^2_\ell f(\theta)\Big)\boldsymbol{\alpha}\Big|\\
	    % \le&\sum_{j_1=1}^K\sum_{j_2=1}^K\Bigg|\boldsymbol{\alpha}_{j_1}^\top\bigg(\frac{\partial^2 f}{\partial {\theta}_{\ell, j_1}\partial {\theta}_{\ell, j_2}}(\theta^\star;p)
	    % -
	    % \frac{\partial^2 f}{\partial {\theta}_{\ell, j_1}\partial {\theta}_{\ell, j_2}}(\theta;p)\bigg)\boldsymbol{\alpha}_{j_2}\Bigg|\\
	    \le &\frac{1}{K^2}\sum_{j_1=1}^K\sum_{j_2=1}^K \|\bfP_1+\bfP_2\|_2 \cdot\|\boldsymbol{\alpha}_{j_1}\|_2 \cdot \|\boldsymbol{\alpha}_{j_2}\|_2\\
	    \lesssim& \frac{1}{K^2}\cdot \sum_{j_1=1}^K\sum_{j_2=1}^K\|\theta -\theta^{\star}\|_2\cdot \|\boldsymbol{\alpha}_{j_1}\|_2 \|\boldsymbol{\alpha}_{j_2}\|_2\\
        \lesssim& \frac{1}{K^2}\cdot \sum_{j_1=1}^K\sum_{j_2=1}^K\|\theta -\theta^{\star}\|_2\cdot \Big(\frac{\|\boldsymbol{\alpha}_{j_1}\|_2^2 +\|\boldsymbol{\alpha}_{j_2}\|_2^2}{2}\Big)\\
	    \lesssim & \frac{1}{K}\cdot \|\theta^\star-\theta\|_2,
	    \end{split}
	\end{equation}
	where $\boldsymbol{\alpha}\in \mathbb{R}^{Kd}$ and  $\boldsymbol{\alpha}_{j}\in \mathbb{R}^{K_\ell}$ with $\boldsymbol{\alpha} =[\boldsymbol{\alpha}_{1}^\top, \boldsymbol{\alpha}_{2}^\top, \cdots, \boldsymbol{\alpha}_{K}^\top]^\top$.
\end{proof}

\subsection{Proof of Lemma \ref{Lemma: Zhong}}\label{sec: proof_of_lemma_zhong}
We aim to prove that  $\int_{\mathcal{R}} \Big(\sum_{j=1}^K \bfal^\top\bfh \sigma^{\prime}(\theta_{\ell,j}^{\top}\bfh) \Big)^2 p_{H}(\bfh)\cdot  d \bfh$ is strictly greater than zero for any $\bfal$. Therefore, the $\rho_1$ in \eqref{Lemma: second_order_derivative} is strictly greater than zero. 
The proof is inspired by Theorem 3.1 in \cite{DZPS19}. It is obviously that $(\sum_{j=1}^K \bfal^\top\bfh \sigma^{\prime}(\theta_{\ell,j}^{\top}\bfh) )^2$ is greater or equal to zero. Given $(\sum_{j=1}^K \bfal^\top\bfh \sigma^{\prime}(\theta_{\ell,j}^{\top}\bfh))^2$ is continuous, we only need to show that $\alpha$ such that $\sum_{j=1}^K \bfal^\top\bfh \sigma^{\prime}(\theta_{\ell,j}^{\top}\bfh) \neq 0$ for any $\alpha$, namely, $\{\bfh \sigma^{\prime}(\theta_{\ell,j}^{\top}\bfh)\}_{j=1}^K$ are linear independent.                                           
\begin{proof}[Proof of Lemma \ref{Lemma: Zhong}]
    Let $\mathcal{H}$ be a Hilbert space on $\mathbb{R}^{K_\ell}$, and the inner product of $\mathcal{H}$ is defined as
    \begin{equation}
        \langle f , g \rangle = \int_{\mathcal{R}}f(\bfh)^\top g(\bfh) f_H(\bfh) \cdot d\bfh, \quad \forall f , g \in \mathcal{H},
    \end{equation} 
    where the Lebesgue measure of $\mathcal{R}$ over $\mathbb{R}^{K_\ell}$ is non-zero. 
    Instead of directly proving $\int_{\mathcal{R}} \Big(\sum_{k=1}^K \bfal^\top\bfh \sigma^{\prime}(\theta_k^{\top}\bfh) \Big)^2 f_H(\bfh) \cdot d\bfh>0$ for any $\bfal$, we note that it is sufficient to prove that $\{\bfh\sigma^\prime(\theta_k^\top\bfh)\}_{k\in[K]}$ are linear independent over the Hilbert space $\mathcal{H}$. Namely, if $\{\bfh\sigma^\prime(\theta_k^\top\bfh)\}_{k\in[K]}$ are linear independent, we have 
    \begin{equation}
        \bfal^\top\bfh \sigma^{\prime}(\theta_k^{\top}\bfh) \neq 0 \quad\textit{almost everywhere}.
    \end{equation}
    Therefore, we can know that  $\int_{\mathcal{R}} \Big(\sum_{j=1}^K \bfal^\top\bfh \sigma^{\prime}(\theta_{\ell,j}^{\top}\bfh) \Big)^2 p_{H}(\bfh)\cdot  d \bfh$ is strictly greater than zero. 

    Next, we provide the whole proof for that $\{x\sigma^\prime(\theta_k^\top\bfh)\}_{k\in[K]}$ are linear independent over the Hilbert space $\mathcal{H}$. 

    We define a group of functions $\{\psi_j(\bfh)\}_{j=1}^K$, where $\psi_j(\bfh)= \bfh \sigma^\prime (\theta_j^\top\bfh)$.  From the assumption in Lemma \ref{Lemma: Zhong}, we can justify that $\mathbb{E}_{\bfh\sim\mathcal{D}} |\psi_j(\bfh)|^2\le \mathbb{E}_{\bfh\sim\mathcal{D}} |\bfh|^2< \infty$.

    Let $\mathcal{X}_i=\{\bfh\mid \theta_i^\top\bfh =0 \}$ for any $i\in [K]$.
    For any fixed $k$, we can justify that $\mathcal{X}_k$ cannot be covered by other sets $\{\mathcal{X}_k\}_{j\neq k}$ as long as $\theta_k$ does not parallel to any other weights $\theta_j$ with $j\neq k$. Namely, $\mathcal{X}_k \not\subset \cup_{j\neq k}\mathcal{X}_j$. The idea of proving the claim above is that the intersection of $\mathcal{X}_j$ and $\mathcal{X}_k$ is only a hyperplane in $\mathcal{X}_k$. The union of finite many hyperplanes is not even a measurable space and thus cannot cover the original space. Formally, we provide the formal proof for this claim as follows.

    Let $\lambda$ be the Lebesgue measure on $\mathcal{X}_k$, then $\lambda(\mathcal{X}_k)>0$. When $\theta_j$ does not  parallel to $\theta_k$, $\mathcal{X}_k \cap \mathcal{X}_j$ is only a hyperplane in $\mathcal{X}_k$ for $j\neq k$. Hence, we have $\lambda(\mathcal{X}_j\cap \mathcal{X}_k)=0$. Next, we have 
    \begin{equation}
        \lambda\big(\mathcal{X}_k \cap (\cup_{j\neq k} \mathcal{X}_k)\big) 
        \le \sum_{j\neq k} \lambda(\mathcal{X}_k \cap \mathcal{X}_j) =0.
    \end{equation}
    Therefore, we have 
    \begin{equation}
        \lambda\big(\mathcal{X}_k /(\cup_{j\neq k} \mathcal{X}_k)\big)
        =\lambda(\mathcal{X}_k) - \lambda\big(\mathcal{X}_k \cap(\cup_{j\neq k} \mathcal{X}_k)\big) = \lambda(\mathcal{X}_k)>0.
    \end{equation}
    Therefore, we have $\mathcal{X}_k /(\cup_{j\neq k} \mathcal{X}_j)$ is not empty, which means that $\mathcal{X}_k \not\subset \cup_{j\neq k} \mathcal{X}_j$.

    Next, Since $\mathcal{X}_k /(\cup_{j\neq k} \mathcal{X}_j)$ is not an empty set, there exists a point $\bfz_k\in \mathcal{X}_k /(\cup_{j\neq k} \mathcal{X}_j)$ and $r_0>0$ such that 
    \begin{equation}
        \mathcal{B}(\bfz_k, r) \cap \mathcal{D}_j = \emptyset \quad \textit{with} \quad \forall r\le r_0 \textit{~and~} j\neq k,
    \end{equation}
    where $\mathcal{B}(\bfz_k, r)$ stands for a ball centered at $\bfz_k$ with a radius of $r$. Then, we divide $\mathcal{B}(\bfz_k,r)$ into two disjoint subsets such that 
    \begin{equation}
    \begin{gathered}
        \mathcal{B}_r^+ = \mathcal{B}(\bfz_k , r) \cap \{\bfh \mid \theta_k^\top \bfh >0 \},\\
        \mathcal{B}_r^- = \mathcal{B}(\bfz_k , r) \cap \{\bfh \mid \theta_k^\top \bfh <0 \}.
    \end{gathered}
    \end{equation}
    Because $\bfz_k$ is a boundary point of $\{\bfh|\theta_k^\top\bfh=0\}$, both $\mathcal{B}_r^+$ and
    $\mathcal{B}_r^-$ are non-empty.

    Note that $\psi_j(\bfh)$ is continuous at any point except for the ones in $\mathcal{X}_j$.
    Then, for any $j\neq k$, we know that $\sigma_j(\theta_k^\top\bfh)$ is continuous at point $\bfz_k$ since $\bfz_k\not\in \mathcal{X}_j$. Hence, it is easy to verify that 
    \begin{equation}\label{eqn: jneqk}
        \lim_{r\rightarrow 0_+} \frac{1}{\lambda(\mathcal{B}_r^+)}{\int_{\mathcal{B}_r^+} \psi_k(\bfh)d\bfh}
        =\lim_{r\rightarrow 0_-} \frac{1}{\lambda(\mathcal{B}_r^-)}{\int_{\mathcal{B}_r^+} \psi_k(\bfh)d\bfh} = \psi_k(\bfz_k).
    \end{equation}
    While for $\psi_k$, we know that  $\psi_k(\bfh) \equiv 0$ for $\bfh\in\mathcal{B}_r^-$, (ii) $\psi_k(\bfh) = \bfh$ for $\bfh\in\mathcal{B}_r^+$.
    Hence, it is easy to verify that 
    \begin{equation}\label{eqn: jeqk}
        \begin{gathered}
            \lim_{r\rightarrow 0_+} \frac{1}{\lambda(\mathcal{B}_r^+)}{\int_{\mathcal{B}_r^+} \psi_k(\bfh)d\bfh} = \bfz_k
        \\
        \lim_{r\rightarrow 0_-} \frac{1}{\lambda(\mathcal{B}_r^-)}{\int_{\mathcal{B}_r^+} \psi_k(\bfh)d\bfh} = 0.
        \end{gathered}
    \end{equation}
    Now let us proof that $\{\psi_j\}_{j=1}^K$ are linear independent by contradiction.  Suppose $\{\psi_j\}_{j=1}^K$ are linear dependent, we have 
    \begin{equation}\label{eqn: 5_2}
        \sum_{j=1}^K \alpha_j\psi_j(\bfh) \equiv 0, \quad \forall \bfh.  
    \end{equation}
    Then, we have 
    \begin{equation}
        \begin{gathered}
            \lim_{r\rightarrow 0_+} \frac{1}{\lambda(\mathcal{B}_r^+)}{\int_{\mathcal{B}_r^+} \sum_{j=1}^K\alpha_j\psi_j(\bfh)d\bfh}=0\\
            \lim_{r\rightarrow 0_+} \frac{1}{\lambda(\mathcal{B}_r^-)}{\int_{\mathcal{B}_r^+} \sum_{j=1}^K\alpha_j\psi_j(\bfh)d\bfh}=0
        \end{gathered}
    \end{equation}
    Then, we have 
    \begin{equation}\label{eqn: 5_1}
    \begin{split}
       0 =& 
        \lim_{r\rightarrow 0_+} \frac{1}{\lambda(\mathcal{B}_r^+)}{\int_{\mathcal{B}_r^+} \sum_{j=1}^K\alpha_j\psi_j(\bfh)d\bfh} -
            \lim_{r\rightarrow 0_+} \frac{1}{\lambda(\mathcal{B}_r^-)}{\int_{\mathcal{B}_r^+} \sum_{j=1}^K\alpha_j\psi_j(\bfh)d\bfh}\\
        = & \alpha_k \bfz_k
    \end{split}
    \end{equation}
    where the last equality comes from \eqref{eqn: jneqk} and \eqref{eqn: jeqk}. 
    
    Note that $\bfz_k$ cannot be $\bfzero$ because $\bfz_k \not\in \mathcal{X}_j$. Therefore, we have $\alpha_k =0$. Similarly to \eqref{eqn: 5_1}, we can obtain that $\alpha_j =0 $ by define $\bfz_j$ following the definition of $\bfz_k$ for any $j\in[K]$. Then, we know that \eqref{eqn: 5_2} holds if and only if $\bfal =\bfzero$, which contradicts the assumption that $\{\psi_j\}_{j=1}^K$ are linear dependent. 

    In conclusion, we know that $\{\psi_j\}_{j=1}^K$ are linear independent, and $\int_{\mathcal{R}} \Big(\sum_{j=1}^K \bfal^\top\bfh \sigma^{\prime}(\theta_{\ell,j}^{\top}\bfh) \Big)^2 p_{H}(\bfh)\cdot  d \bfh$ is strictly greater than zero. 
 \end{proof}

\subsection{Proof of Lemma \ref{Lemma: h_bound}}
\begin{proof}[Proof of Lemma \ref{Lemma: h_bound}]
    From the definition of \eqref{eqn: defi_h}, we have 
    \begin{equation}\label{eqn: 10_1}
        \begin{split}
            &\|\bfh^{(\ell)}(\theta) -\bfh^{(\ell)}(\theta^\star)\|_2\\
            =& \|\sigma\big( \theta_{\ell-1}^\top \bfh^{({\ell-1})}(\theta) \big) - \sigma\big( \theta_{\ell-1}^{\star\top} \bfh^{({\ell-1})}(\theta^\star) \big)\|_2\\
            = & \|\sigma\big( \theta_{\ell-1}^\top \bfh^{({\ell-1})}(\theta) \big) - \sigma\big( \theta_{\ell-1}^{\star\top} \bfh^{({\ell-1})}(\theta) \big)
            + \sigma\big( \theta_{\ell-1}^{\star\top} \bfh^{({\ell-1})}(\theta) \big) - \sigma\big( \theta_{\ell-1}^{\star\top} \bfh^{({\ell-1})}(\theta^\star) \big)\|_2\\
            \le& \|\sigma\big( \theta_{\ell-1}^\top \bfh^{({\ell-1})}(\theta) \big) - \sigma\big( \theta_{\ell-1}^{\star\top} \bfh^{({\ell-1})}(\theta) \big)\|_2
            + \|\sigma\big( \theta_{\ell-1}^{\star\top} \bfh^{({\ell-1})}(\theta) \big) - \sigma\big( \theta_{\ell-1}^{\star\top} \bfh^{({\ell-1})}(\theta^\star) \big)\|_2\\
            \le &\|\theta_{\ell-1} - \theta_{\ell-1}^\star\|_2\cdot \|\bfh^{({\ell-1})}(\theta)\|_2 
            +  \|\bfh^{(\ell-1)}(\theta)- \bfh^{(\ell-1)}(\theta^\star)\|_2.
        \end{split}
    \end{equation}
    With the assumption in the Lemma \ref{Lemma: h_bound} such that $\theta$ is close enough to $\theta^\star$, we have 
    \begin{equation}
        \|\theta_{i}\|_2 \le \|\theta_{i}^\star\|_2 + \|\theta_{i}-\theta_i^\star\|_2 \lesssim 1.
    \end{equation}
    Therefore, we have 
    \begin{equation}
        \|\bfh^{(i)}(\theta)\|_2 \le \|\theta_i\|_2\cdots \|\theta_1\|_2 \cdot \|\bfx\|_2 \lesssim \|\bfx\|_2.
    \end{equation}
    
    Then, we have 
    \begin{equation}
        \begin{split}
            &\|\bfh^{(\ell)}(\theta) -\bfh^{(\ell)}(\theta^\star)\|_2\\
        \le &\|\theta_{\ell-1} - \theta_{\ell-1}^\star\|_2\cdot \|\bfx\|_2 
            +  \|\bfh^{(\ell-1)}(\theta)- \bfh^{(\ell-1)}(\theta^\star)\|_2\\
        \le & \sum_{i=1}^{\ell-1}\|\theta_{i} - \theta_{i}^\star\|_2\cdot \|\bfx\|_2  
            +  \|\bfh^{(1)}(\theta)- \bfh^{(1)}(\theta^\star)\|_2\\
        = &    \sum_{i=1}^{\ell-1}\|\theta_{i} - \theta_{i}^\star\|_2\cdot \|\bfx\|_2  
            +  \|\bfx - \bfx\|_2\\
            = &    \sum_{i=1}^{\ell-1}\|\theta_{i} - \theta_{i}^\star\|_2\cdot \|\bfh^{({i-1})}(\theta)\|_2\\
            \le & \|\theta -\theta^\star \|_2 \cdot \|\bfx\|_2,
        \end{split}
    \end{equation}
    which completes the proof.
\end{proof}

\end{document}